\documentclass[11pt]{article}
\usepackage[T1]{fontenc}
\usepackage{times}
\usepackage[letterpaper,textwidth=6in,textheight=9in]{geometry}
\usepackage[authoryear,round]{natbib}
\usepackage{microtype}
\newlength{\figurewidth}
\usepackage{hyperref}
\hypersetup{hypertexnames=false,hidelinks,pdftitle={Jacobian Rank Collapse in Decision-Focused Learning},pdfauthor={Aojie Yuan, Haiyue Zhang, Zijian Su},pdfsubject={Predictor geometry and decision-focused learning}}
\usepackage{url}
\usepackage{amsmath,amssymb,amsfonts,amsthm}
\usepackage{booktabs}
\newtheorem{proposition}{Proposition}
\newtheorem{theorem}{Theorem}
\newtheorem{corollary}{Corollary}
\newtheorem{lemma}{Lemma}
\newtheorem{definition}{Definition}
\usepackage{graphicx}

\newsavebox{\publicationtablebox}
\newcommand{\publicationtable}[1]{%
  \sbox{\publicationtablebox}{#1}%
  \ifdim\wd\publicationtablebox>\linewidth
    \resizebox{\linewidth}{!}{\usebox{\publicationtablebox}}%
  \else
    \usebox{\publicationtablebox}%
  \fi}

\usepackage{algorithm}
\usepackage{algorithmic}
\usepackage{multirow}
\usepackage{xcolor}
\usepackage{subcaption}
\usepackage{enumitem}

\usepackage{fancyhdr}
\usepackage{placeins}
\definecolor{publicationink}{HTML}{243746}
\definecolor{publicationaccent}{HTML}{6542A6}
\definecolor{publicationmuted}{HTML}{64717A}
\renewcommand{\headrulewidth}{0.35pt}
\fancypagestyle{plain}{\fancyhf{}\fancyfoot[C]{\small\thepage}\renewcommand{\headrulewidth}{0pt}}
\makeatletter
\renewcommand{\headrule}{\hbox to\headwidth{\color{publicationaccent}\leaders\hrule height \headrulewidth\hfill}}
\renewcommand\section{\@startsection{section}{1}{\z@}{-2.6ex plus -.7ex minus -.2ex}{1.3ex plus .2ex}{\normalfont\large\sffamily\bfseries\color{publicationaccent}}}
\renewcommand\subsection{\@startsection{subsection}{2}{\z@}{-2ex plus -.5ex minus -.2ex}{.9ex plus .2ex}{\normalfont\normalsize\sffamily\bfseries}}
\renewcommand\@maketitle{\newpage\null\vskip -36pt
 \institutionlogos\par\vskip 12pt
 {\color{publicationaccent}\hrule height .45pt}\vskip 12pt
 {\raggedright\fontsize{18}{22}\selectfont\color{publicationink}\mbox{\@title}\par}
 \vskip 8pt
 {\centering\normalsize\@author\par}
 \vskip 5pt {\centering\small\authoraffiliations\par}
 \vskip 4pt {\centering\fontsize{8.5}{10}\selectfont\authoremails\par}
 \vskip 9pt}
\makeatother

\newcommand{\bw}{\mathbf{w}}
\newcommand{\br}{\mathbf{r}}
\newcommand{\bx}{\mathbf{x}}
\newcommand{\bSigma}{\hat{\Sigma}}
\newcommand{\bsigma}{\boldsymbol{\sigma}}
\newcommand{\cS}{\mathcal{S}}
\newcommand{\cL}{\mathcal{L}}
\newcommand{\btheta}{\boldsymbol{\theta}}

\title{Jacobian Rank Collapse in Decision-Focused Learning}

\author{Aojie Yuan\textsuperscript{1}\quad Haiyue Zhang\textsuperscript{1}\quad Zijian Su\textsuperscript{2}}
\newcommand{\institutionlogos}{%
\noindent\makebox[\linewidth]{%
\raisebox{-.5\height}{\includegraphics[width=155pt]{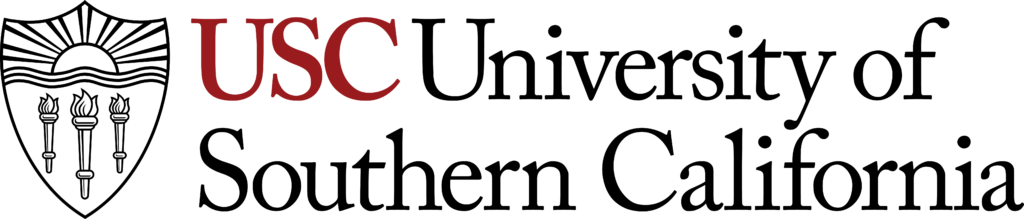}}\hfill
\raisebox{-.5\height}{\includegraphics[width=195pt]{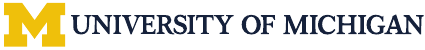}}}}
\newcommand{\authoraffiliations}{%
\textsuperscript{1}University of Southern California\quad
\textsuperscript{2}University of Michigan}
\newcommand{\authoremails}{%
\href{mailto:aojieyua@usc.edu}{\texttt{aojieyua@usc.edu}}\quad
\href{mailto:haiyuez@usc.edu}{\texttt{haiyuez@usc.edu}}\quad
\href{mailto:simoon@umich.edu}{\texttt{simoon@umich.edu}}}

\date{}

\newsavebox{\abstractframe}
\renewenvironment{abstract}{%
 \par\setlength{\fboxsep}{9pt}\setlength{\fboxrule}{.45pt}%
 \begin{lrbox}{\abstractframe}\begin{minipage}{\dimexpr\linewidth-2\fboxsep-2\fboxrule\relax}%
 \small\setlength{\parindent}{0pt}\setlength{\parskip}{5pt}%
 {\centering\normalsize\bfseries\color{publicationaccent}Abstract\par}\vspace{3pt}%
}{\end{minipage}\end{lrbox}\noindent\fcolorbox{publicationaccent}{white}{\usebox{\abstractframe}}\par\vspace{9pt}}

\usepackage{alphalph}
\begin{document}
\raggedbottom
\maketitle

\begin{abstract}
Decision-focused learning (DFL) trains predictors through downstream objectives, but a different loss need not provide an independent parameter-update direction. We characterize this restriction through the predictor Jacobian, using sparse index tracking to distinguish the covariance entries read by the optimizer from the parameter directions available to learning. Rank-one Jacobians make nonzero per-example gradients collinear; a conditional spectral bound describes near-collinearity. A batch-subspace characterization and counterexamples show why these local statements imply neither common minimizers nor collinear batch updates.

Experiments examine when geometry translates into decision quality. Across 38 one-parameter equity configurations, DFL gains over MSE remain below 1.8\%; a 385-parameter conditional predictor also has pointwise rank one. In validation-tuned shortest-path and knapsack experiments, full-capacity SPO+ reduces mean regret by 11.6\% and 10.6\%, respectively; only knapsack survives correction across eight comparisons. The capacity contrast persists on fresh datasets across batch orders and training budgets. Holding expressivity fixed, invertible coordinate scaling lowers spectral effective rank and ordinary SGD gains; compensating for the scaling restores the original trajectories. Financial forward-target controls separate forecast accuracy from decision quality; a matched neural comparison finds no aggregate DFL advantage in the tested architecture. These findings distinguish local rank restrictions, coordinate-dependent optimization and predictive accuracy. Predictor geometry helps explain available learning directions, while held-out decision quality remains the test of practical benefit.
\end{abstract}
\begin{figure}[!h]
\centering
\includegraphics[width=\figurewidth]{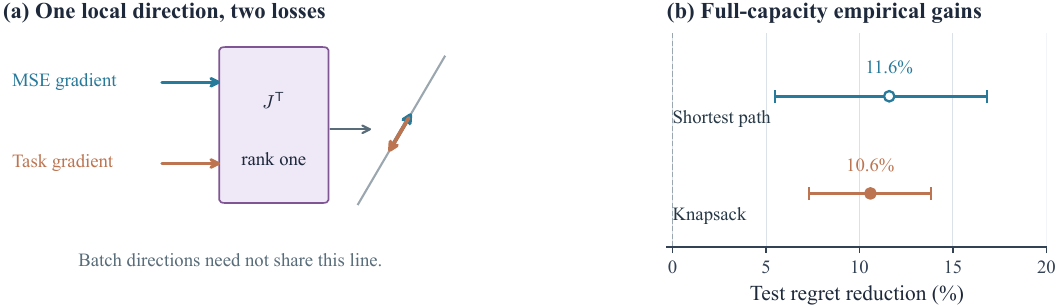}
\caption{\textbf{Local geometry restricts directions; task gains require evidence.} Left: rank-one parameter Jacobians map nonzero loss gradients onto one line, but batch directions can differ across examples. Right: full-capacity regret reductions in the controlled extension, with pointwise 95\% paired-dataset bootstrap intervals over ten datasets. Only knapsack passes Holm correction across eight comparisons. Filled markers denote Holm-adjusted significance.}
\label{fig:hero}
\end{figure}
\clearpage

\section{Introduction}

When should a practitioner pay for decision-focused learning (DFL) rather than fit a predictor and optimize its outputs? Differentiable optimization and decision-aware surrogates offer several ways to train end to end \citep{elmachtoub2022smart,donti2017task,wilder2019melding,mandi2024decision}. Their benefit depends on both the downstream objective and the predictor: a decision-relevant error direction is useful only if the model can express an update along it. We study this second restriction.

At a fixed input, the central object is the predictor Jacobian $\mathbf{J}=\partial\mathrm{vec}(\Sigma)/\partial\theta$. Both MSE and task gradients pass through $\mathbf{J}^\top$ before updating the parameters. At rank one, their nonzero images lie on one line (Proposition~\ref{prop:dim1}); a spectral bound controls angular separation when both the singular-value gap and the gradients' leading components permit it (Theorem~\ref{thm:spectral}). Rank collapse therefore excludes a new \emph{independent direction}, but does not exclude different signs, stationary points or final decisions. We distinguish this exact local statement from the empirical question of whether DFL improves test performance.

\paragraph{Why finance, and what should transfer?}
Sparse index tracking separates three restrictions: selecting assets determines which covariance entries the optimizer reads; the predictor maps output gradients into parameter directions; aggregation combines directions across examples. These operations need not impose the same geometry.

Across 38 one-parameter equity configurations, DFL gains over MSE remain below $1.8\%$. Yet a 385-parameter conditional predictor also has pointwise rank one: parameter count alone misses the bottleneck. The financial baseline reconstructs trailing covariance rather than forecasting future covariance (Section~\ref{sec:methodology}). Figure~\ref{fig:hero} connects the mechanism to the controlled evidence; Figure~\ref{fig:pipeline} shows the training paths.

We test transfer in shortest-path and knapsack tasks using shared initializers, separate validation and exact decision oracles. A capacity sweep first measures the performance contrast; an invertible reparameterization then holds expressivity fixed to examine coordinate effects. Finally, forward-covariance baselines test whether the financial findings depend on a reconstruction target. This sequence separates three questions: which directions are available, how optimization uses them, and whether the resulting predictions improve decisions.

\paragraph{Contributions.}
\begin{enumerate}[leftmargin=*,itemsep=3pt]
\item \textbf{A geometric account from output support to batch updates.} A support-energy inequality, conditional spectral bound and batch-subspace characterization separate what a decision loss observes from the directions a predictor can follow. Counterexamples rule out inferring common minimizers or collinear batch updates from pointwise rank one.
\item \textbf{Controlled tests beyond the financial setting.} Two synthetic tasks use matched initial predictors, independently selected hyperparameters and saved models for direct verification. A fresh-data follow-up varies minibatch order and training budget while allowing either loss to retain the unchanged baseline. An invertible reparameterization then holds expressivity fixed: spectral rank and ordinary SGD gains fall together, while compensation restores the original updates. This distinguishes a coordinate-dependent optimization effect from lost expressivity.
\item \textbf{Financial evidence with explicit scope.} Equity comparisons separate parameter count, local geometry and observed decision quality. A chronological forward-target control compares future-covariance MSE, task validation, Ledoit--Wolf and EWMA; better covariance forecasts need not yield better tracking. A matched neural forward-target comparison reports a small aggregate difference and frequent validation selection of the unchanged initializer. Dynamic-selection experiments and negative results delimit these claims. Regularization coefficients and benefit thresholds remain validation choices rather than consequences of rank collapse.
\end{enumerate}

\section{Related Work}

\paragraph{Sparse Index Tracking.}
Sparse portfolio construction has a rich history \citep{kolm2014years}: evolutionary heuristics \citep{beasley2003evolutionary}, mixed-integer programming \citep{canakgoz2009mixed}, $\ell_1$-penalized formulations \citep{benidis2018sparse,brodie2009sparse}, and cardinality-constrained optimization \citep{xu2016sparse}. These approaches address sparse portfolio construction; our focus is on training covariance predictors through the resulting decisions.

\paragraph{Decision-Focused Learning.}
DFL optimizes downstream objectives using differentiable QPs \citep{donti2017task,amos2017optnet,agrawal2019differentiable}, SPO+ \citep{elmachtoub2022smart}, combinatorial surrogates \citep{wilder2019melding,vlastelica2020differentiation}, or implicit differentiation \citep{blondel2022efficient,paulus2024lpgd}. PG losses have asymptotic decision-quality guarantees under misspecification \citep{huang2024decision}; PEAR characterizes regret gradients through active-constraint tangent spaces and local curvature \citep{lee2026pear}. Other work studies regret decomposition \citep{aldridge2026regret}, online DFL \citep{capitaine2026online}, and prediction inflation in portfolios \citep{wang2026decision}. Our analysis concerns a different stage: the predictor Jacobian that maps these output-space signals into trainable parameter directions.

\paragraph{Portfolio Optimization and Covariance Estimation.}
End-to-end portfolio construction \citep{butler2023integrating,zhang2020deep,kim2025covariance} and high-dimensional covariance estimation \citep{ledoit2004well,fan2013large,friedman2008sparse,engle2002dynamic} are both active areas. Concurrent work by \citet{jeon2026sparse} applies DFL to sparse tangent portfolios and reports gains in larger universes---complementary to our study of local gradient geometry and low-capacity predictors.

\section{Problem Formulation}
\label{sec:formulation}

Let $\br_t \in \mathbb{R}^N$ denote asset returns, $\bw_{\text{idx}}$ the index weights, and $\Sigma = \text{Cov}(\br_t)$. Tracking error minimization reduces to $\min_\bw (\bw - \bw_{\text{idx}})^\top \Sigma (\bw - \bw_{\text{idx}})$. The joint selection-and-weighting problem is NP-hard \citep{beasley2003evolutionary}, so we decompose into two stages.

\paragraph{Stage 1: Subset selection.} Choose $\cS \subseteq \{1,\dots,N\}$ with $|\cS| = K$ via market-cap ranking (default) or tracking-score: $s_i = w_i^{\text{idx}} \cdot (\bSigma_{i,:} \bw_{\text{idx}})^2 / \bSigma_{ii}$. Under tracking-score selection, $\cS$ depends on $\bSigma$, so off-block estimation errors propagate to stock selection.

\paragraph{Stage 2: Weight optimization (QP).} Given $\cS$, solve:
\begin{equation}
  \min_{\bw_\cS} \; \bw_\cS^\top \bSigma_{\cS\cS} \bw_\cS - 2\bw_\cS^\top \bSigma_{\cS,:}\bw_{\text{idx}}
  \;\; \text{s.t.} \; \mathbf{1}^\top \bw_\cS {=} 1,\, \bw_\cS {\geq} 0.
  \label{eq:qp}
\end{equation}
Selection $\cS$ is detached from the graph (Appendix~\ref{app:design}). The QP reads at most $2NK - K^2$ of $N^2$ entries; this asymmetry drives our gradient alignment analysis.

\section{Methodology}
\label{sec:methodology}

\begin{figure}[t]
\centering
\includegraphics[width=\figurewidth]{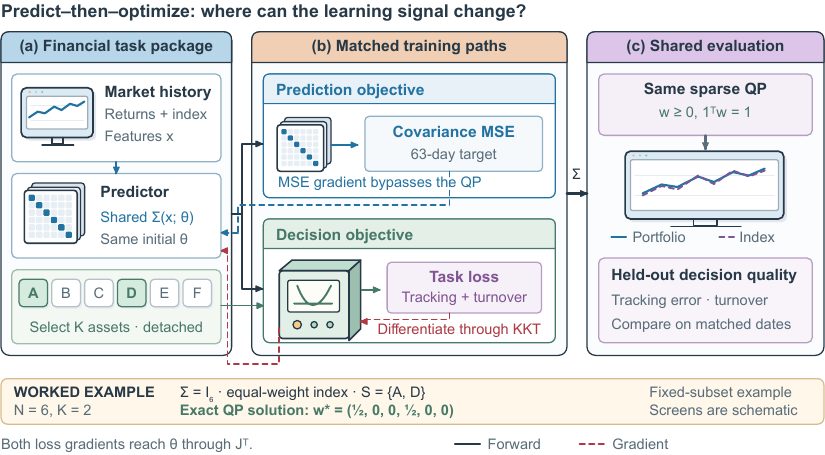}
\caption{\textbf{Two objectives, one predictor.} MSE directly supervises covariance prediction; task loss differentiates through the QP. Both reach $\theta$ through $\mathbf{J}^{\top}$, with selection detached. The toy example illustrates sparse weighting for an equal-weight index with identity covariance.}
\label{fig:pipeline}
\end{figure}

\subsection{Two-Stage Baseline and DFL}

The standard two-stage approach trains under Frobenius loss $\cL_{\text{MSE}} = N^{-2}\|\bSigma(\bx_t; \theta) - \Sigma_{\text{realized},t}\|_F^2$, where $\Sigma_{\text{realized},t}$ is the sample covariance of the preceding 63 trading days. In the inspected training implementation this is a trailing reconstruction target, not a future covariance label. For shrinkage families it is also the input sample covariance, so this baseline should not be interpreted as an optimally specified supervised forecasting model. In DFL, we embed the QP as a differentiable layer:
\begin{equation}
  \cL_{\text{task}} = \frac{252}{H} \sum_{s=0}^{H-1} \bigl(\bw^*(\bSigma_t)^\top \br_{t+s} - r_{t+s}^{\text{idx}}\bigr)^2 + \frac{\gamma}{2} \|\bw^*_t - \bw^*_{t-1}\|_1,
  \label{eq:task}
\end{equation}
where $t$ denotes the first return in the forward training window, $\bw^*(\bSigma_t)$ is the QP solution, $H = 21$ trading days, and $\gamma$ controls turnover. Gradients flow through the KKT conditions \citep{amos2017optnet} into $\theta$. Selection $\cS$ is detached from the graph.

\subsection{Block-Diagonal DFL (BD-DFL)}

Under detached selection, the task gradient is restricted to the current support $\mathcal T$. Errors outside that support remain unpenalized by the task objective and may matter after selection changes. BD-DFL adds relative MSE regularization:
\begin{equation}
  \cL_{\text{BD-DFL}} = \cL_{\text{task}} + \beta \cdot \|\bSigma(\theta) - \Sigma_{\text{realized}}\|_F^2 / (\|\Sigma_{\text{realized}}\|_F^2+10^{-12}),
  \label{eq:bddfl_method}
\end{equation}
where $\beta = 0$ recovers pure DFL and large $\beta$ emphasizes covariance fit. A selection-robustness perspective motivates this penalty: entries outside today's support may enter tomorrow's QP. Classical DRO provides related regularization principles \citep{mohajerin2018data,blanchet2019quantifying}, but does not establish the swap-dependent coefficient claimed for this particular model. We therefore choose $\beta$ empirically and state the limitation in Appendix~\ref{app:dro_proof}.

\begin{proposition}[Uniform control of selected inputs]
\label{thm:dro}
For every error matrix $E$ and subset $\cS$, $\|E_{\cS\cS}\|_{\mathrm{op}}\leq\|E\|_F$ and $\|E_{\cS,:}w_{\mathrm{idx}}\|_2\leq\|E\|_F\|w_{\mathrm{idx}}\|_2$. Thus a Frobenius penalty controls the two perturbed QP inputs uniformly over subsets. Converting this into a decision or regret bound requires solution-stability assumptions.
\end{proposition}

\subsection{Covariance Models}

Five architectures of increasing capacity, each trained under MSE and DFL:
\textbf{Shrinkage}~(1p): $\bSigma = (1{-}\alpha)S + \alpha \mu I$ \citep{ledoit2004well};
\textbf{Factor}~($O(NK_f)$p): $\bSigma = BFB^\top + D$;
\textbf{Neural}~($O(Nr)$p): MLP $\to LL^\top + D$;
\textbf{Structured}~(12p): learnable per-sector targets;
\textbf{Conditional}~(385p): regime-adaptive $\alpha_t = \sigma(\text{MLP}(\mathbf{z}_t))$.
Training: Adam (lr $10^{-3}$, weight decay $10^{-4}$, clip 1.0), early stopping on validation TE, cvxpylayers \citep{agrawal2019differentiable} for QP differentiation. Details in Appendix~\ref{app:training}.

\section{Local Geometry and Batch Limits}
\label{sec:geometry}

Throughout this section, $g$ denotes an output-space loss gradient, $J$ a pointwise predictor Jacobian, and $J_{\mathrm{stack}}$ the Jacobian stacked over examples. We reserve $r_{\mathrm{eff}}$ for spectral entropy rank and $d_{\mathrm{proxy}}=d h$ for the archived heterogeneity heuristic. Input sensitivity uses a different Jacobian, $J_x=\partial f_\theta(x)/\partial x$, whose spectral rank is denoted $r_{\mathrm{in}}$. Neither $r_{\mathrm{in}}$ nor $d_{\mathrm{proxy}}$ is interchangeable with the spectral rank of $J$ or $J_{\mathrm{stack}}$.

\subsection{Gradient Alignment Theory}

\begin{figure*}[t]
\centering
\includegraphics[width=\figurewidth]{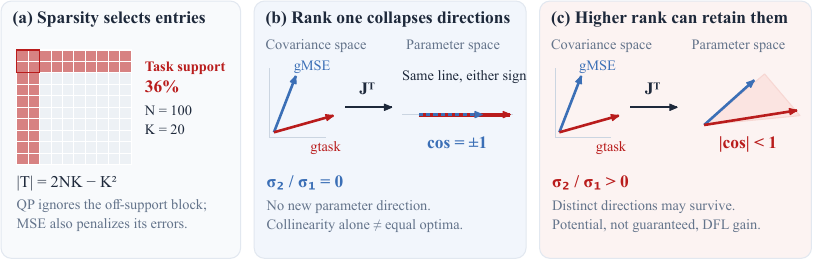}
\caption{\textbf{Where gradient directions are lost.} (a)~The task-relevant support has $2NK-K^2$ entries (36\% here). (b)~A rank-one predictor Jacobian maps nonzero gradients onto one line, with either sign. (c)~With multiple active singular directions, angular separation can survive; improved task performance is possible, not guaranteed.}
\label{fig:mechanism}
\end{figure*}

The task gradient decomposes as:
\begin{equation}
  \nabla_\theta \ell = \underbrace{\frac{\partial \ell}{\partial \bw^*}}_{\text{task}} \cdot \underbrace{\frac{\partial \bw^*}{\partial \bSigma}}_{\text{QP Jacobian}} \cdot \underbrace{\frac{\partial \bSigma}{\partial \theta}}_{\text{model}}.
  \label{eq:chain}
\end{equation}
The QP Jacobian $\partial \bw^* / \partial \bSigma$, derived from KKT differentiation \citep{amos2017optnet}, acts as a \emph{support-restricted linear map}: it zeroes out entries $\bSigma_{ij}$ where both $i \notin \cS$ and $j \notin \cS$, since the QP reads only $\bSigma_{\cS\cS}$ and $\bSigma_{\cS,:}\bw_{\text{idx}}$. MSE can place gradient energy on all entries. A small support limits alignment only when little MSE energy concentrates on that support.

\begin{proposition}[Gradient Alignment Bound]
\label{prop:alignment}
Let $\cS \subset \{1,\ldots,N\}$, $|\cS| = K$. Define the task-relevant support $\mathcal{T} = \{(i,j) : i \in \cS \text{ or } j \in \cS\}$, with $|\mathcal{T}| = 2NK - K^2$. For nonzero gradients, suppose the normalized MSE energy obeys $\mathbb{E}[\|\mathbf{g}_{\mathrm{MSE}}|_{\mathcal T}\|^2/\|\mathbf{g}_{\mathrm{MSE}}\|^2]\leq |\mathcal T|/N^2$. Then:
\begin{equation}
\mathbb{E}|\cos(\mathbf{g}_{\mathrm{MSE}},\, \mathbf{g}_{\mathrm{task}})| \leq \sqrt{\tfrac{2K}{N} - \tfrac{K^2}{N^2}} \approx \sqrt{\tfrac{2K}{N}} \;\; (K \ll N).
\label{eq:alignment}
\end{equation}
\end{proposition}

This support-only inequality is sharp for unrestricted ambient vectors (Appendix~\ref{app:lower_proof}); attaining it with gradients of the tracking QP is a separate question.

\begin{proposition}[Decision Invariance]
\label{prop:invariance}
The QP~\eqref{eq:qp} depends on~$\bSigma$ only through $\bSigma_{\cS\cS}$ and $\bSigma_{\cS,:}\bw_{\mathrm{idx}}$; off-block errors are irrelevant to both the solution and regret (Proposition~\ref{prop:regret}; proofs in Appendix~\ref{app:invariance_proof}).
\end{proposition}

This argument concerns the support of the output gradient. Its concentration within the selected block and its image under the predictor Jacobian are separate quantities; neither follows from cardinality alone.

\subsection{Predictor Jacobian and Gradient Alignment}
\label{sec:eff_dim}

The support bound concerns covariance-space geometry. We now characterize its image in parameter space; experimental capacity proxies are considered separately in Section~\ref{sec:proxy_evidence}.

\begin{proposition}[Jacobian Rank Collapse]
\label{prop:dim1}
Let $\mathbf J\in\mathbb R^{N^2\times d}$ be the Jacobian of a differentiable covariance predictor. For two nonzero parameter gradients,
\begin{equation}
\cos(\nabla_\theta\ell_{\mathrm{MSE}},\nabla_\theta\ell_{\mathrm{task}})
=\frac{\mathbf g_{\mathrm{MSE}}^\top\mathbf J\mathbf J^\top\mathbf g_{\mathrm{task}}}
{\|\mathbf J^\top\mathbf g_{\mathrm{MSE}}\|\,\|\mathbf J^\top\mathbf g_{\mathrm{task}}\|}.
\end{equation}
If $\mathrm{rank}(\mathbf J)=1$, this cosine belongs to $\{-1,+1\}$, regardless of the covariance-space angle.
\end{proposition}

\emph{Proof sketch.} Write $\mathbf{J}=\sigma\mathbf{u}\mathbf{v}^\top$. Then $\mathbf{J}^\top\mathbf{g}=\sigma(\mathbf{u}^\top\mathbf{g})\mathbf{v}$, so both nonzero images are scalar multiples of $\mathbf{v}$. Their signs and zeros can differ. Equality of minimizers requires further assumptions (Appendix~\ref{app:dim1_proof}). \qed

The following theorem gives the exact form of parameter-space alignment and bounds how quickly it collapses to $\pm 1$ as the Jacobian approaches rank one.

\begin{theorem}[Spectral Gap Controls Gradient Alignment]
\label{thm:spectral}
Let $\mathbf{J} = U S V^\top \in \mathbb{R}^{N^2 \times d}$ be the Jacobian of $\mathrm{vec}(\Sigma(\btheta))$, with $r=\mathrm{rank}(\mathbf J)\geq1$, thin factors $U\in\mathbb R^{N^2\times r}$ and $V\in\mathbb R^{d\times r}$ having orthonormal columns, and $S=\mathrm{diag}(\sigma_1,\ldots,\sigma_r)$. Assume both parameter gradients and both leading components $\mathbf u_1^\top\mathbf g$ are nonzero, and set $\sigma_2=0$ when $r=1$. Then:
\begin{enumerate}[leftmargin=*]
\item \emph{(Exact form.)} Parameter-space alignment is the alignment of the $\sigma$-weighted coordinates of the two gradients along the singular directions:
\begin{equation}
\cos(\nabla_{\btheta} \ell_{\mathrm{MSE}}, \nabla_{\btheta} \ell_{\mathrm{task}}) = \cos\!\bigl(S U^\top \mathbf{g}_{\mathrm{MSE}},\; S U^\top \mathbf{g}_{\mathrm{task}}\bigr).
\label{eq:spectral_align}
\end{equation}
\item \emph{(Near-rank-one collapse.)} For each gradient let $\rho = \min\{1, (\sigma_2/\sigma_1)/|\cos(\mathbf{g}, \mathbf{u}_1)|\}$, and let $\Theta = \arcsin \rho_{\mathrm{MSE}} + \arcsin \rho_{\mathrm{task}}$. If $\Theta < \pi/2$, then
\begin{equation}
|\cos(\nabla_{\btheta} \ell_{\mathrm{MSE}}, \nabla_{\btheta} \ell_{\mathrm{task}})| \;\geq\; \cos\Theta,
\qquad
\mathrm{sign}\,\cos(\cdot,\cdot) = \mathrm{sign}\bigl[(\mathbf{u}_1^\top \mathbf{g}_{\mathrm{MSE}})(\mathbf{u}_1^\top \mathbf{g}_{\mathrm{task}})\bigr].
\label{eq:spectral_bound}
\end{equation}
\item \emph{(Endpoints.)} $\sigma_2{=}0$ gives $|\cos|{=}1$ (Proposition~\ref{prop:dim1}); $\mathbf J\mathbf J^\top=cI$ with $c>0$ preserves output-space alignment. The expected support bound then transfers only under Proposition~\ref{prop:alignment}'s energy assumption.
\end{enumerate}
\end{theorem}

\emph{Proof sketch.} $\nabla_{\btheta}\ell = V(SU^\top\mathbf{g})$ with $V^\top V=I$ gives~\eqref{eq:spectral_align}; each $SU^\top\mathbf{g}$ lies within angle $\arcsin\rho$ of $\pm\mathbf{e}_1$, and the angle is a metric on the sphere (Appendix~\ref{app:dim1_proof}).~$\square$

\paragraph{Rank, angles, and a proxy are different quantities.}
The spectral effective rank is $r_{\mathrm{eff}}=\exp(-\sum_i p_i\log p_i)$, where $p_i=\sigma_i^2/\sum_j\sigma_j^2$. For a zero Jacobian we use the convention $r_{\mathrm{eff}}=0$. Concentration of this distribution at one forces $\sigma_2/\sigma_1\to0$, but does not ensure that a particular task gradient has a nonzero leading component. The empirical quantity $d_{\mathrm{proxy}}=d h$ is a separate heterogeneity-based proxy; no equality between it and $r_{\mathrm{eff}}$ is assumed. Corollary~\ref{cor:positive} requires the two projected gradients to be non-collinear, not merely rank greater than one.

\subsection{From a Single Example to a Batch}
\label{sec:batch}
Let $f_i(\theta)$ be the prediction for example $i$ and $J_i$ its Jacobian. A batch loss has gradient $G=\sum_i J_i^\top g_i$. The relevant parameter subspace is therefore shared across examples, rather than determined by the rank of any one $J_i$.

\begin{proposition}[Batch gradient subspace]
\label{prop:batch}
Let $J_{\mathrm{stack}}=[J_1^\top,\ldots,J_n^\top]^\top\ne0$. Every batch gradient lies in $\mathcal V=\operatorname{range}(J_{\mathrm{stack}}^\top)$, whose dimension is $\operatorname{rank}(J_{\mathrm{stack}})$. All possible nonzero batch gradients are collinear if and only if this rank is one. In particular, rank one for each $J_i$ is insufficient unless their nonzero row spaces share a common line.
\end{proposition}
\begin{proof}
Stacking the output gradients gives $G=J_{\mathrm{stack}}^\top[g_1^\top,\ldots,g_n^\top]^\top$. Its possible values form exactly $\mathcal V$. A nonzero vector space contains only collinear pairs precisely when its dimension is one. Here ``possible'' ranges over differentiable output losses at the fixed parameter value; it does not assert that a particular pair of losses realizes every direction.
\end{proof}

\paragraph{Two counterexamples.}
A scalar predictor $f(\theta)=\theta$ has rank one, yet losses $(f-1)^2$ and $(f-2)^2$ have different minimizers. For a batch, take $f_1(\theta)=\theta_1$ and $f_2(\theta)=\theta_2$. At $\theta=0$, the two losses $\frac12[(f_1+1)^2+(f_2+1)^2]$ and $\frac12[(f_1+1)^2+(f_2-1)^2]$ have orthogonal gradients $(1,1)$ and $(1,-1)$ although each example Jacobian has rank one. Appendix~\ref{app:batch} gives the full construction and the covariance-model interpretation.

This distinction matters for conditional shrinkage: $\Sigma(x;\theta)$ depends on a scalar $\alpha(x;\theta)$ for each $x$, but $\nabla_\theta\alpha(x;\theta)$ can change direction across examples. A pointwise rank-one measurement does not establish a one-dimensional training problem. The results concern raw Euclidean gradients at a fixed parameterization; adaptive optimizer histories and changes in $\theta$ require additional analysis.

\paragraph{A measured architectural bottleneck.}
Figure~\ref{fig:rank_collapse} tests this distinction directly. The conditional predictor's 385 parameters produce a rank-one covariance Jacobian at each measured input. The structured model has several active directions, but the near-rank-one bound is vacuous for its sampled task gradients. Appendix~\ref{app:bound_check} reports the numerical scope; neither observation is a test-loss guarantee.

\begin{figure}[h]
\centering
\includegraphics[width=\figurewidth]{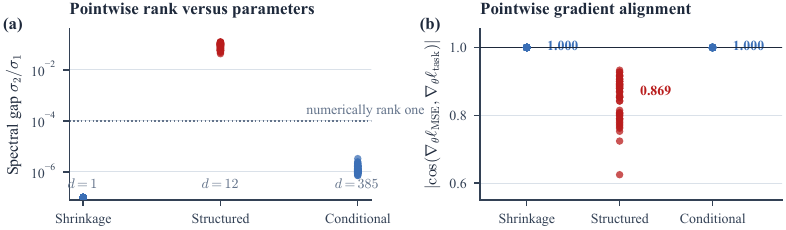}
\caption{\textbf{Pointwise Jacobian measurements on 54 inputs per model class.} A 385-parameter conditional predictor is numerically rank one because its output passes through a scalar shrinkage intensity. Its gradient direction can still vary across inputs. These measurements distinguish parameter count from local rank; they do not measure the stacked batch rank.}
\label{fig:rank_collapse}
\end{figure}

\section{Experimental Setup}

\subsection{Data}

Top 100 S\&P~500 stocks by market cap ($N{=}97$ after ${\geq}90\%$ coverage filter), daily returns 2006--2025. Rolling-window protocol: 3-year train, 6-month validation, 1-year test, 9 folds covering 2016--2025 in the primary neural evaluation. Features: 21-day volatility, 63-day momentum, sector, market-cap quintile. Scaling: S\&P~500 full ($N{=}478$, 5 folds), S\&P~1500 ($N{=}1258$, 8 folds), 20-year backtest ($N{=}100$, 17 folds). Cross-market: FTSE~100, Nikkei~225, Euro~Stoxx~50, ASX~200, Hang~Seng. Cross-domain: NOAA weather stations ($N{\in}\{100, 189, 415\}$), EPA PM2.5 monitors ($N{\in}\{124, 132\}$). Unless stated otherwise, financial paired comparisons use Wilcoxon signed-rank tests across folds: one-sided for directional hypotheses and two-sided for difference tests. Equivalence uses TOST with the stated margin. Correlation tests, synthetic dataset-level inference and descriptive forward-target controls are identified separately.

\subsection{Comparison Units and Evidence Provenance}
The unit of comparison is a matched fold or seed within one experiment, not an individual daily return. Chronological folds share training history and may be dependent. We report their variability descriptively; unadjusted signed-rank $p$-values do not by themselves establish generalization across markets or tasks. A non-significant difference is not evidence of equivalence; only comparisons with an explicit TOST margin receive an equivalence interpretation. Relative changes are computed from unrounded source values, so they need not equal ratios of rounded table entries.

The release distinguishes three records: corrected equity evaluations with bounded test windows; an archived diagnostic scoring record predating those corrections; and separate spatial and combinatorial experiments. The 33 archived diagnostic cases are not an independent validation of the corrected harness. Configurations can share folds and data, so the reported count of paired folds is not a count of independent datasets. Appendix~\ref{app:provenance} maps evidence types and reproducibility scope.

\subsection{Baselines}

We compare: \emph{Naive} (top-$K$ by index weight), \emph{K-only LW} ($K{\times}K$ covariance only), \emph{Ledoit-Wolf} (analytical shrinkage \citep{ledoit2004well}), \emph{POET} (factor + adaptive thresholding \citep{fan2013large}), \emph{GraphicalLasso} ($\ell_1$-penalized precision \citep{friedman2008sparse}), \emph{ValTuned} (grid search on validation TE), \emph{SPO+} \citep{elmachtoub2022smart}, \emph{LODL} (locally optimized decision losses \citep{shah2022decision}), and MSE-trained factor/neural models. Primary metric: annualized TE = $\text{std}(\bw^\top \br_t - r_t^{\text{idx}}) \times \sqrt{252}$.

\section{Results}

\subsection{Small Gains in One-Parameter Equity Models}

\noindent\textit{Main finding: DFL adds a ${\sim}3{,}300\times$ training-cost overhead with little measured benefit in the tested low-capacity settings.}

\smallskip\noindent At $N{=}100$, fitted one-parameter models achieve similar test performance across different training objectives. This is an empirical result, complementary to the local collinearity statement: the theorem restricts available directions, whereas test-loss equivalence must be measured.

\begin{table}[t]
\centering
\caption{\textbf{One-parameter comparison at $N{=}100$.} ValTuned, SPO+ and DFL meet the TOST criterion ($\delta=20$ bps) in all 12 pairwise comparisons (three pairs, four sparsity levels; max $|d|=0.66$). POET and GLasso are different model families, shown for reference. At $K=50$ ($K/N=0.52$), their ordering changes and the $\sqrt{2K/N}$ bound is vacuous; no equivalence with these baselines is claimed.}
\label{tab:exp3}

\fontsize{9}{10.8}\selectfont
\setlength{\tabcolsep}{5pt}
\renewcommand{\arraystretch}{1.08}
\publicationtable{%
\begin{tabular}{@{}ccccccc@{}}
\toprule
$K$ & LW & POET & GLasso & ValTuned & SPO+ & DFL \\
\midrule
5  & 0.0860 & \textbf{0.0765} & 0.0768 & 0.0773 & 0.0766 & 0.0766 \\
10 & 0.0511 & \textbf{0.0457} & 0.0461 & 0.0458 & 0.0458 & 0.0458 \\
20 & 0.0331 & \textbf{0.0267} & 0.0269 & 0.0269 & 0.0269 & 0.0268 \\
50 & 0.0096 & 0.0104 & 0.0124 & 0.0088 & \textbf{0.0088} & 0.0089 \\
\bottomrule
\end{tabular}}
\end{table}

Table~\ref{tab:exp3} reports equivalence under TOST with a 20-bps margin for all 12 pairwise comparisons among ValTuned, SPO+, and DFL. Each improves on Ledoit--Wolf by 7--19\%, so fitting the parameter matters even when these objectives yield similar test errors. POET and GLasso are separate model families; their agreement at smaller $K$ is empirical. LODL also matches DFL, indicating that explicit KKT differentiation is not necessary to obtain this performance. Higher-capacity models show different outcomes (Table~\ref{tab:n500}); the neural gain is significant at $K\leq30$ but absent at $K{=}50$ (Appendix~\ref{app:sparsity}).

\subsection{Architecture and Loss Choice at Larger Scale}

\noindent\textit{Main finding: Architecture changes can matter more than the choice of training loss.}

\smallskip\noindent At $N{=}478$, we compare changes in predictor architecture with changes in training loss (Table~\ref{tab:n500}).

\begin{table}[t]
\centering
\caption{\textbf{Higher-capacity comparison.} $N{=}478$, 5 folds, one evaluation window. $d$: Cohen's $d$ on per-fold relative change; $p$: one-sided Wilcoxon, for which $0.031$ is the floor at $n{=}5$ (i.e.\ 5/5 folds).}
\label{tab:n500}

\fontsize{9}{10.8}\selectfont
\setlength{\tabcolsep}{5pt}
\renewcommand{\arraystretch}{1.08}
\publicationtable{%
\begin{tabular}{@{}llcccc@{}}
\toprule
Model & $K$ & MSE TE & DFL TE & $d$ & $p$ \\
\midrule
GLasso & 20 & \multicolumn{2}{c}{0.0575} & --- & --- \\
POET & 20 & \multicolumn{2}{c}{0.0503} & --- & --- \\
Shrink.\ (1p) & 20 & 0.0517 & 0.0514 & $-1.07$ & $0.031^{*}$ \\
Cond.\ (385p) & 20 & \textbf{0.0454} & \textbf{0.0448} & $-0.83$ & 0.062 \\
Struct.\ (12p) & 20 & 0.0513 & 0.0508 & $-1.60$ & $0.031^{*}$ \\
Struct.\ (12p) & 50 & 0.0309 & 0.0305 & $-1.03$ & 0.062 \\
\bottomrule
\end{tabular}}
\end{table}

Table~\ref{tab:n500} separates architectural improvements from within-model loss changes. GLasso has 14.3\% higher TE than POET, while the MSE-trained conditional predictor has 9.8\% lower TE than POET. Switching from MSE to DFL yields smaller relative improvements: $0.68\%$, $1.07\%$ and $1.34\%$ for the 1-, 12- and 385-parameter models, respectively. These parameter counts do not order pointwise Jacobian rank: the conditional predictor remains rank one. The one-parameter model improves in all five folds, but its $0.68\%$ gain remains within the range observed at smaller scale. With five folds, $p=0.031$ is the smallest attainable one-sided Wilcoxon value; consistency of sign should therefore be read alongside effect size.

\subsection{Cross-Market Evidence for One-Parameter Models}
\label{sec:crossmarket}

\noindent Across six markets and the reported scale variants, one-parameter gains remain modest (Table~\ref{tab:crossmarket}). Across all 38 configurations, gains lie in $[-0.57\%,+1.76\%]$ over $K/N\in[0.008,0.43]$, with no detected monotone association with sparsity ($\rho=-0.16$, $p=0.35$). This observation is empirical; rank-one collinearity does not imply it. The structured model gives larger relative TE reductions on ASX ($2.95\%$), Hang Seng ($2.77\%$) and S\&P~1500 ($2.41\%$, all $p\leq0.008$), while gains on Euro Stoxx, Nikkei and FTSE are below $1\%$. Architecture alone does not predict the size of the improvement (Appendix~\ref{app:ftse}).

\begin{table}[t]
\centering
\caption{\textbf{Cross-market tracking gains.} Positive $\Delta\%$ denotes lower TE than matched MSE. The listed comparisons total 81 paired folds; the broader 38-configuration analysis contains 316 paired folds, with shared data and histories. Nikkei, Euro Stoxx, ASX and Hang Seng use one common four-market configuration with eight folds per market. $^{*}p<0.05$, $^{**}p<0.01$; one-sided Wilcoxon for the structured-model comparison.}
\label{tab:crossmarket}

\fontsize{9}{10.8}\selectfont
\setlength{\tabcolsep}{5pt}
\renewcommand{\arraystretch}{1.08}
\publicationtable{%
\begin{tabular}{@{}lrrrrr@{}}
\toprule
Market & $N$ & $K$ & Shrink $\Delta\%$ & Struct $\Delta\%$ & $p_{\text{struct}}$ \\
\midrule
S\&P 500 (20y) & 100 & 20 & $+0.29$ & $+0.40$ & $0.020^{*}$ \\
FTSE 100 & 91 & 10 & $-0.09$ & $+0.37$ & $0.039^{*}$ \\
Nikkei 225 & 218 & 10 & $+0.17$ & $+0.87$ & $0.027^{*}$ \\
Euro Stoxx 50 & 47 & 10 & $+0.29$ & $+0.95$ & $0.008^{**}$ \\
ASX 200 & 159 & 10 & $+0.98$ & $\mathbf{+2.95}$ & $\mathbf{0.004^{**}}$ \\
Hang Seng & 62 & 10 & $+1.76$ & $+2.77$ & $0.008^{**}$ \\
S\&P 1500 & 1258 & 10 & $+0.32$ & $+2.41$ & $0.004^{**}$ \\
S\&P 1500 & 1258 & 20 & $+0.44$ & $+2.21$ & $0.004^{**}$ \\
S\&P 1500 & 1258 & 50 & $-0.13$ & $+0.98$ & $0.039^{*}$ \\
\bottomrule
\end{tabular}}
\end{table}

\subsection{BD-DFL under Dynamic Selection}
\label{sec:bddfl}

\label{sec:bddfl_capacity}
The preceding comparisons hold the selection protocol fixed. We next examine changing asset support, where errors outside the current support can affect subsequent decisions. The effect of regularization depends on both the model and selection rule. With dynamic tracking-score selection at $N{=}100$, neural DFL increases mean TE by $8.6\%$ at $K{=}10$; BD-DFL with $\beta=0.001$ reduces it by $8.2\%$ relative to MSE. At $N{=}478$, the corresponding reduction is $7.4\%$ ($p=0.063$, five folds). In the corrected $N{=}451$ run, pure DFL already improves by $13.0\%$ at $K{=}10$ and BD-DFL improves by $15.6\%$ ($p=0.0024$, 9/11 wins). At $K{=}20$, pure DFL increases TE by $4.5\%$, while BD-DFL reduces it by $5.5\%$ ($p=0.103$). These results motivate validation of regularization, not a universal benefit or coefficient (Appendix~\ref{app:bddfl_detail}).

\subsection{Limits of Transfer}
\label{sec:proxy_evidence}
\label{sec:shortest_path}
Historical experiments test the limits of transfer. A heterogeneity-based capacity proxy changes with the chosen partition and is not a Jacobian rank (Appendix~\ref{app:diagnostic}). NOAA and EPA show opposite performance changes with only one or two evaluation windows per configuration. The archived prediction record also predates evaluation-harness corrections; its scoring is not prospective validation of the corrected results.

The older shortest-path sweeps measure the \emph{input} Jacobian, whereas the propositions concern the \emph{parameter} Jacobian. Separate PyEPO probes use a sampled stacked parameter Jacobian, but their empirical benefit threshold does not transfer from equities (Appendices~\ref{app:shortest_path}--\ref{app:pyepo}). These distinct measurements motivate the matched, all-output parameter probes and independent validation used next.

\subsection{Validation-Tuned Transfer Beyond Finance}
\label{sec:controlled}
To avoid relying on hyperparameters transferred from finance, we tune shortest-path and knapsack models separately (Appendix~\ref{app:controlled}). Each task uses ten independently generated datasets, a shared train-only ridge initializer, and identical validation budgets for MSE and SPO+ at every capacity. The probes measure all output coordinates at 32 held-out examples; they distinguish pointwise from stacked spectral rank.

Full-capacity SPO+ reduces mean test regret by $11.59\%$ on shortest path and $10.59\%$ on knapsack. Across the eight task--capacity comparisons, Holm-adjusted two-sided Wilcoxon $p$-values are $0.0684$ and $0.0156$, respectively. At one, two and eight update directions, mean changes range from $-0.61\%$ to $+1.89\%$, with pointwise bootstrap intervals crossing zero. Figure~\ref{fig:controlled} shows the accompanying loss of batch-gradient alignment as capacity increases. These measurements support a capacity-dependent contrast in this protocol, not a universal gain threshold or a causal identification of rank alone.

A fresh-data sensitivity study crosses three minibatch orders with 20/80-epoch budgets and allows both losses to select the unchanged ridge predictor. Full-capacity mean gains remain $12.12/12.54\%$ for path and $13.67/13.76\%$ for knapsack; scalar gains stay below $0.6\%$. Inference averages training seeds within each independent dataset, rather than counting them as additional datasets. Figure~\ref{fig:robustness} shows the follow-up: only full-capacity knapsack passes Holm correction (Appendix~\ref{app:robustness}).

\begin{figure}[t]
\centering
\includegraphics[width=\figurewidth]{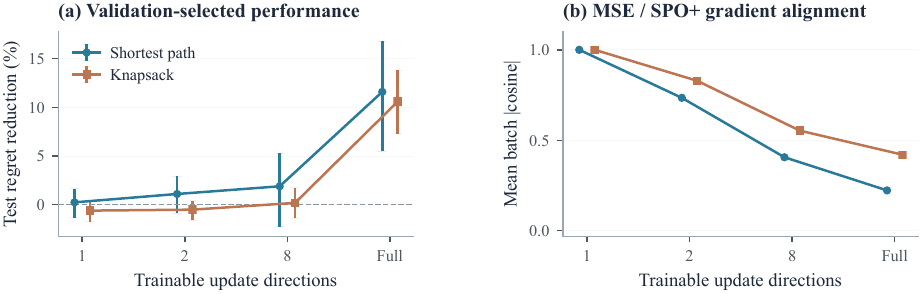}
\caption{\textbf{Validation-tuned cross-domain extension.} (a) Mean regret reduction with pointwise 95\% paired-dataset bootstrap intervals (ten datasets; 10,000 resamples). (b) Mean absolute MSE/SPO+ batch-gradient cosine at the selected MSE model, using 32 held-out examples. Each capacity is tuned separately; full capacity has 144/72 update directions for path/knapsack. Only full-capacity knapsack passes Holm correction across eight tests.}
\label{fig:controlled}
\end{figure}

\begin{figure}[t]
\centering
\includegraphics[width=\figurewidth]{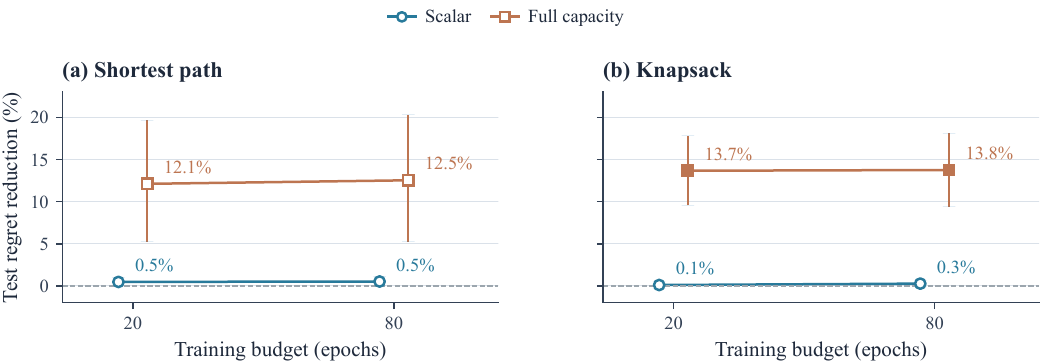}
\caption{\textbf{Capacity contrast on fresh data and longer training.} Mean test-regret reduction versus validation-selected MSE; whiskers are pointwise 95\% paired-dataset bootstrap intervals. Each task uses ten new datasets and three minibatch orders, averaged within dataset before inference. Both losses may select the unchanged initializer. Filled markers denote comparisons passing Holm correction across eight tests; only full-capacity knapsack passes. Scalar gains below 0.6\% are not evidence of equivalence.}
\label{fig:robustness}
\end{figure}

\subsection{Holding Expressivity Fixed}
\label{sec:fixed_class}
The capacity sweep changes both geometry and the function class. We therefore add an invertible coordinate transformation to a full affine predictor, $P=P_0+D_\varepsilon\Theta$, where $D_\varepsilon=\mathrm{diag}(1,\varepsilon,\ldots,\varepsilon)$. Every positive $\varepsilon$ represents the same predictors and preserves exact Jacobian rank. Shrinking $\varepsilon$ nevertheless lowers spectral effective rank and attenuates ordinary SGD updates in most output directions.

On ten fresh datasets per task, reducing $\varepsilon$ from $1$ to $0.01$ lowers mean SPO+ regret reduction from $11.79\%$ to $1.19\%$ for path and from $10.17\%$ to $0.52\%$ for knapsack. Compensating the parameter update by $D_\varepsilon^{-2}$ restores the identity-coordinate trajectories and their gains (Figure~\ref{fig:reparameterization}). Thus expressivity alone cannot explain this contrast. Spectral rank and conditioning still change together: this control identifies a removable coordinate effect, not a causal effect of rank independent of optimization. Appendix~\ref{app:reparameterization} gives the derivation and full protocol.

\begin{figure}[t]
\centering
\includegraphics[width=\figurewidth]{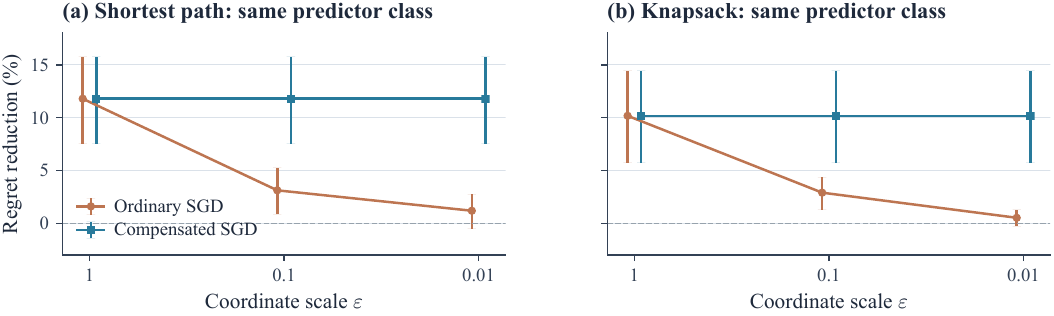}
\caption{\textbf{Same function class, different ordinary SGD behavior.} Invertible coordinate scaling reduces spectral effective rank while preserving exact rank and expressivity. Compensated SGD recovers the original predictor updates. Whiskers are pointwise 95\% paired-dataset bootstrap intervals over ten fresh datasets per task, not multiplicity-adjusted tests.}
\label{fig:reparameterization}
\end{figure}

\subsection{Forecast Accuracy Is a Separate Test}
\label{sec:forecast_control}
The fixed-class control addresses optimization coordinates; it leaves the prediction target unchanged. We next revisit the historical financial MSE target, which reconstructs trailing covariance. To test that choice, a separate chronological control fits the same scalar shrinkage family to future covariance, with forward labels purged at the training boundary. Across ten annual test folds (2016--2025), validation-selected future-target MSE reduces relative covariance error by $37.25\%$ but increases mean tracking error by $4.13\%$. Task-validation tuning reduces mean tracking error by $1.46\%$; Ledoit--Wolf reduces it by $0.69\%$. These descriptive results expose a forecast-to-decision gap rather than establish DFL superiority. The fixed, survivorship-biased universe and scalar predictor limit this comparison (Appendix~\ref{app:forecast_control}).

We then compare future-target MSE with DFL in the same 197-parameter residual covariance network, using matched initializations, forward windows and validation budgets. Across the same ten test years, with three initialization seeds averaged within year, mean annualized TE is $5.711\%$ for MSE and $5.715\%$ for DFL: DFL is $0.07\%$ worse in relative terms, despite improving in six years. Validation retains the unchanged initializer in 17/30 MSE and 12/30 DFL fits. Thus this neural control does not reproduce the larger historical DFL gains. It is neither an equivalence test nor evidence about all neural predictors; Appendix~\ref{app:neural_forward} reports the architecture, yearly results and training limits.

\FloatBarrier
\section{Conclusion}
Decision losses and predictors impose distinct restrictions: output support determines which errors a decision observes, the predictor Jacobian maps gradients into parameter directions, and stacking determines the batch subspace. Rank one restricts local directions without fixing training outcomes.

The experiments identify three distinct sources of variation. Capacity-dependent gains persist across tested batch orders and budgets, with stronger statistical evidence for knapsack. Invertible coordinate scaling changes ordinary SGD outcomes without changing expressivity, while better covariance forecasts can worsen tracking. A useful assessment of DFL must therefore examine geometry, optimization and held-out decision quality together. The matched neural forward-target control shows no aggregate DFL advantage under its finite budget. Broader architectures, longer training and point-in-time financial universes are needed to establish robustness; linking training-trajectory geometry to generalization remains open.


\subsection*{Ethics statement}
This work studies structural limitations of decision-focused learning in predict-then-optimize pipelines. Our experiments use publicly available financial, weather and air-quality data, together with synthetic optimization tasks. No human subjects are involved. Runtime comparisons describe computational cost in the tested settings; they do not measure energy savings. Financial backtests can be misinterpreted as prospective investment evidence. The reported tracking metrics do not establish profitability, deployment safety, or robustness to future market conditions.

\subsection*{AI use statement}
We used Claude (Anthropic) and Codex (OpenAI) to assist with analysis and plotting code, result auditing, literature checks, drafting and revising prose, and figure design. AI-assisted auditing identified evaluation-harness defects and unsupported theoretical steps, which prompted corrections and narrower claims in this revision. The accompanying audit records distinguish checked data, historical results, and remaining limitations. The authors remain responsible for the scientific content and the final submitted version.

\subsection*{Reproducibility statement}
\label{page:endmain}
The released bundle contains saved results, editable figures, numerical checks and a training-source snapshot. Three new combinatorial studies include 880 selected models, data arrays and rerunnable scripts. A scalar forecast-target study saves 70 annual portfolio evaluations; a matched neural study adds 120 candidate fits and 60 selected models. Historical real-data training has not been independently repeated in this release. Appendix~\ref{app:provenance} distinguishes current corrected equity results, archived diagnostics and new synthetic runs; Appendix~\ref{app:training} gives the equity training protocol.

The current equity comparisons use a repaired evaluation harness that confines holdings to each test window and uses actual trading dates. Earlier defects could extend validation holdings into test periods and changed some reported results. The correction history and remaining limits are documented in Appendix~\ref{app:provenance}. Archived diagnostic scores are not presented as independently certified preregistration or as a prospective test of the corrected harness.

\bibliographystyle{plainnat}
\bibliography{references}

\clearpage
\appendix
\section*{Appendix Guide}
The appendices separate mathematical arguments from experimental evidence. Appendix~\ref{app:geometry_group} gives proofs and counterexamples; Appendix~\ref{app:protocol_group} documents provenance, training and cost; Appendix~\ref{app:equity_group} collects equity comparisons; Appendix~\ref{app:diagnostic_group} examines the diagnostic and spatial transfer; Appendix~\ref{app:combinatorial_group} reports combinatorial experiments; Appendix~\ref{app:demonstration_group} separates expressivity, optimization coordinates and predictive targets; Appendix~\ref{app:limits_group} summarizes open questions.

\paragraph{Reading by claim.}
\begin{description}[leftmargin=0pt,style=nextline,itemsep=4pt]
\item[Local directions and batch limits.] Appendix~\ref{app:geometry_group}: proofs, conditional bounds and counterexamples. These are fixed-parameter statements, not guarantees about training trajectories or test performance.
\item[Financial utility and robustness.] Appendix~\ref{app:equity_group}: matched-fold tracking errors, risk, transaction costs and regularization. Rolling folds can share history; configurations are not independent datasets. Appendix~\ref{app:provenance} identifies corrected evaluations and remaining provenance limits.
\item[Controlled evidence outside finance.] Appendices~\ref{app:controlled}--\ref{app:robustness}: exact decision oracles, independently generated datasets, saved selected models and budget sensitivity. The statistical unit is a dataset; changing capacity also changes expressivity.
\item[Demonstration controls.] Appendix~\ref{app:demonstration_group}: fixed-function-class reparameterization and chronological scalar and neural future-target baselines. The neural comparison has no aggregate DFL gain in its tested architecture; this is a descriptive result, not an equivalence test.
\item[Exploration and failure boundaries.] Appendix~\ref{app:diagnostic_group} and the earlier combinatorial studies retain proxy failures, sparse spatial evaluations and input-Jacobian measurements. They are distinct from the new parameter-geometry tests; Appendix~\ref{app:limits_group} collects unresolved findings.
\end{description}

\FloatBarrier
\section{Geometry and Proofs}
\label{app:geometry_group}

\subsection{Proof of Proposition~\ref{prop:dim1} (Jacobian Rank Collapse)}
\label{app:dim1_proof}

\begin{proof}
At a fixed input and parameter value, write the nonzero rank-one Jacobian as $J=\sigma u v^\top$, with unit vectors $u,v$ and $\sigma>0$. By the chain rule, $\nabla_\theta\ell=J^\top g=\sigma(u^\top g)v$. Every nonzero gradient is thus a scalar multiple of $v$, with cosine equal to the sign of the product of its two coefficients. This proves the statement for any parameter dimension $d$, including conditional predictors with many parameters. If either coefficient vanishes, its parameter gradient is zero and cosine is undefined. Differentiability is local; for a QP it requires an appropriate regular solution map.
\end{proof}
\begin{equation}
\nabla_\theta\ell=J^\top g.
\label{eq:param_grad}
\end{equation}

\begin{corollary}[Vacuousness of the $\sqrt{K/N}$ bound for $d{=}1$]
\label{cor:vacuous}
Proposition~\ref{prop:alignment} bounds covariance-space alignment under its support-energy assumption. For a scalar predictor parameter, any two nonzero parameter gradients instead have cosine $\pm1$. Their signs and magnitudes can still depend on $K$ and $N$; the covariance-space angle does not survive as a separate parameter-space angle.
\end{corollary}

\begin{proof}
The optimizer updates $\theta$ using the parameter-space gradient $\nabla_\theta \ell$, which lies in $\mathbb{R}$ for a scalar parameter. The $N^2$-space misalignment between $\mathbf{g}_{\mathrm{MSE}}$ and $\mathbf{g}_{\mathrm{task}}$ is projected onto this subspace via~\eqref{eq:param_grad}. For any two vectors $\mathbf{u}, \mathbf{v}$ with angle $\phi$ in $\mathbb{R}^{N^2}$, their projections onto a one-dimensional subspace $\mathrm{span}(\mathbf{J})$ yield scalars $\mathbf{J}^\top \mathbf{u}$ and $\mathbf{J}^\top \mathbf{v}$ with cosine $\operatorname{sign}(\mathbf{J}^\top \mathbf{u} \cdot \mathbf{J}^\top \mathbf{v}) \in \{-1, +1\}$, collapsing $\phi$ to either $0$ or $\pi$. In particular, the covariance-space angle does not survive as an independent parameter-space angle for $d{=}1$.
\end{proof}

\paragraph{Concordance vs.\ discordance.}
Proposition~\ref{prop:dim1} establishes that $d{=}1$ gradients are collinear but does not determine the sign: $\cos = +1$ (concordant, MSE descent also reduces task loss) vs.\ $\cos = -1$ (discordant, MSE descent increases task loss). We formalize conditions for concordance and verify them experimentally.

\begin{lemma}[Concordance for unimodal losses]
\label{lem:concordance}
Let $\Sigma(\alpha) = (1{-}\alpha) S + \alpha \mu I$ be the linear shrinkage family with $\alpha \in [0,1]$. Suppose:
\begin{enumerate}[label=(\roman*)]
    \item $\ell_{\mathrm{MSE}}(\alpha)$ is strictly convex with minimizer $\alpha^*_{\mathrm{MSE}}$;
    \item $\ell_{\mathrm{task}}(\alpha)$ is unimodal (single minimizer $\alpha^*_{\mathrm{task}}$) with strictly negative/positive derivative to the left/right of its interior minimizer, and an $L$-Lipschitz derivative on the interval between the two minimizers;
    \item $|\alpha^*_{\mathrm{MSE}} - \alpha^*_{\mathrm{task}}| \leq \delta$ for some $\delta > 0$.
\end{enumerate}
Then $\cos(\nabla_\alpha \ell_{\mathrm{MSE}}, \nabla_\alpha \ell_{\mathrm{task}}) = +1$ for all $\alpha \notin [\min(\alpha^*_{\mathrm{MSE}}, \alpha^*_{\mathrm{task}}), \max(\alpha^*_{\mathrm{MSE}}, \alpha^*_{\mathrm{task}})]$, and the task-loss gap satisfies $\ell_{\mathrm{task}}(\alpha^*_{\mathrm{MSE}}) - \ell_{\mathrm{task}}(\alpha^*_{\mathrm{task}}) \leq L\delta^2/2$.
\end{lemma}

\begin{proof}
Condition (i) ensures $\nabla_\alpha \ell_{\mathrm{MSE}} < 0$ for $\alpha < \alpha^*_{\mathrm{MSE}}$ and $> 0$ for $\alpha > \alpha^*_{\mathrm{MSE}}$. Condition (ii) ensures $\nabla_\alpha \ell_{\mathrm{task}}$ has the same sign pattern around $\alpha^*_{\mathrm{task}}$. For $\alpha < \min(\alpha^*_{\mathrm{MSE}}, \alpha^*_{\mathrm{task}})$ or $\alpha > \max(\alpha^*_{\mathrm{MSE}}, \alpha^*_{\mathrm{task}})$, both gradients share the same sign, hence $\cos = +1$. Let $a=\alpha^*_{\mathrm{task}}$ and $b=\alpha^*_{\mathrm{MSE}}$. Since $\ell^{\prime}_{\mathrm{task}}(a)=0$ and its derivative is $L$-Lipschitz, integrating $|\ell^{\prime}_{\mathrm{task}}(t)|\leq L|t-a|$ between $a$ and $b$ bounds the loss difference by $L|b-a|^2/2\leq L\delta^2/2$.
\end{proof}

\emph{Condition verification.} We verify conditions (i)--(iii) across 20 train/test folds (4 markets: S\&P~500, Nikkei~225, Euro~Stoxx~50, Hang~Seng; $K{=}20$; $\alpha \in [0.01, 0.99]$). Condition (i): $\ell_{\mathrm{MSE}}(\alpha)$ is convex quadratic by construction \citep{ledoit2004well}. Condition (ii): $\ell_{\mathrm{task}}(\alpha)$ is unimodal in 20/20 folds (checked via monotonicity before and after the minimum, tolerating $\leq 2$ noise violations). Condition (iii): $|\alpha^*_{\mathrm{MSE}} - \alpha^*_{\mathrm{task}}|$ averages 0.154 (median 0.040), with 14/20 folds having gap $< 0.15$.

\emph{Scope.} The Jacobian is $\mathrm{vec}(\mu I-S)$ and has one column. This restricts directions, not minimizers: even two scalar convex losses $(\alpha-a)^2$ and $(\alpha-b)^2$ can prefer different parameters. The quadratic loss-gap bound above uses smoothness and nearby minimizers; the empirical unimodality check does not certify those assumptions globally.

\emph{Broader verification.} Across 38 $d{=}1$ configurations (6 equity markets, $K \in \{5, 10, 20, 50, 100\}$, $N$ from 47 to 1258, $K/N$ from $0.008$ to $0.43$), the DFL gain over MSE stays below $1.8\%$ with median $0.25\%$, consistent with the $O(\delta^2)$ gap from Lemma~\ref{lem:concordance}. It does not grow with $K/N$ ($\rho{=}{-}0.16$, $p{=}0.35$), an empirical observation compatible with, but not implied by, Corollary~\ref{cor:vacuous}: the $\Theta(\sqrt{K/N})$ mechanism is invisible through a rank-1 Jacobian, although signs, stationary points and test losses can still vary with sparsity.

\paragraph{Extension to $d > 1$.}
For a $d$-parameter model $\Sigma(\btheta)$ with $\btheta \in \mathbb{R}^d$, the Jacobian $\mathbf{J} \in \mathbb{R}^{N^2 \times d}$ maps $N^2$-space gradients to $d$-dimensional parameter-space gradients via $\nabla_{\btheta} \ell = \mathbf{J}^\top \mathbf{g}$. The parameter-space alignment becomes:
\begin{equation}
\cos(\nabla_{\btheta} \ell_{\mathrm{MSE}}, \nabla_{\btheta} \ell_{\mathrm{task}}) = \frac{\mathbf{g}_{\mathrm{MSE}}^\top \mathbf{J}\mathbf{J}^\top \mathbf{g}_{\mathrm{task}}}{\|\mathbf{J}^\top \mathbf{g}_{\mathrm{MSE}}\| \cdot \|\mathbf{J}^\top \mathbf{g}_{\mathrm{task}}\|}.
\label{eq:param_align_d}
\end{equation}
The matrix $\mathbf{J}\mathbf{J}^\top \in \mathbb{R}^{N^2 \times N^2}$ is a positive semidefinite form of rank at most $d$. When $d = 1$, it is rank-1 and collapses all angles to $\{0, \pi\}$ (Proposition~\ref{prop:dim1}). Because $\mathbf{J}\mathbf{J}^\top = U S^2 U^\top$, the relevant inner product weights each singular direction by $\sigma_i^2$; it is not the plain projection $P_J = UU^\top$ onto $\mathrm{col}(\mathbf{J})$, and the two agree only when all nonzero $\sigma_i$ are equal. In the extreme case of $N^2$ orthogonal columns of equal norm, $\mathbf{J}\mathbf{J}^\top = \sigma^2 I$ and the parameter-space alignment equals the $N^2$-space alignment $\cos(\mathbf{g}_{\mathrm{MSE}}, \mathbf{g}_{\mathrm{task}})$; the expected numerical bound requires the support-energy assumption, and full rank alone is not enough.

\paragraph{Proof of Theorem~\ref{thm:spectral}.}
\emph{Exact form.} With $\mathbf{J} = USV^\top$ (thin SVD, $V \in \mathbb{R}^{d\times r}$ with orthonormal columns), $\nabla_{\btheta}\ell = \mathbf{J}^\top\mathbf{g} = V\,\mathbf{a}$ with $\mathbf{a} = SU^\top\mathbf{g}$. Multiplication by $V$ preserves inner products and norms, so $\cos(\nabla_{\btheta}\ell_{\mathrm{MSE}}, \nabla_{\btheta}\ell_{\mathrm{task}}) = \cos(\mathbf{a}, \mathbf{b})$ with $\mathbf{b} = SU^\top\mathbf{g}_{\mathrm{task}}$.

\emph{Near-rank-one collapse.} Write $\mathbf{a} = (a_1, \mathbf{a}_\perp)$ with $a_1 = \sigma_1 \mathbf{u}_1^\top\mathbf{g}$ and $(\mathbf{a}_\perp)_i = \sigma_i \mathbf{u}_i^\top \mathbf{g}$ for $i \geq 2$. Then $\|\mathbf{a}_\perp\|^2 = \sum_{i\geq2}\sigma_i^2(\mathbf{u}_i^\top\mathbf{g})^2 \leq \sigma_2^2\|\mathbf{g}\|^2$, so
\[
\frac{\|\mathbf{a}_\perp\|}{|a_1|} \leq \frac{\sigma_2}{\sigma_1}\cdot\frac{\|\mathbf{g}\|}{|\mathbf{u}_1^\top\mathbf{g}|} = \frac{\sigma_2/\sigma_1}{|\cos(\mathbf{g},\mathbf{u}_1)|}.
\]
Let $\tilde{\mathbf{a}} = \mathrm{sign}(a_1)\,\mathbf{a}$, so $\tilde a_1 > 0$ and the angle $\theta_a$ between $\tilde{\mathbf{a}}$ and $\mathbf{e}_1$ lies in $[0,\pi/2)$ with $\sin\theta_a = \|\mathbf{a}_\perp\|/\|\mathbf{a}\| \leq \|\mathbf{a}_\perp\|/|a_1|$; hence $\theta_a \leq \arcsin\rho_{\mathrm{MSE}}$, and likewise $\theta_b \leq \arcsin\rho_{\mathrm{task}}$. The angle between unit vectors is a metric on the sphere, so the angle between $\tilde{\mathbf{a}}$ and $\tilde{\mathbf{b}}$ is at most $\theta_a + \theta_b \leq \Theta$. When $\Theta < \pi/2$ this gives $\cos(\tilde{\mathbf{a}},\tilde{\mathbf{b}}) \geq \cos\Theta > 0$, and since $\cos(\mathbf{a},\mathbf{b}) = \mathrm{sign}(a_1 b_1)\cos(\tilde{\mathbf{a}},\tilde{\mathbf{b}})$, both the bound on $|\cos|$ and the sign rule follow.

\emph{Endpoints.} If $\sigma_2 = 0$ then $\rho = 0$, $\Theta = 0$ and $|\cos| = 1$. If $\mathbf{J}$ has $N^2$ orthogonal columns of norm $\sigma$, then $S = \sigma I$ and $U$ is square orthogonal, so $\cos(\mathbf{a},\mathbf{b}) = \cos(\mathbf{g}_{\mathrm{MSE}}, \mathbf{g}_{\mathrm{task}})$.~$\square$

\emph{Numerical stress check.} The released script \texttt{verify\_geometry.py} samples 5,000 Jacobians with singular-value ratios from $10^{-6}$ to one. The exact alignment identity agrees to within $4.2\times10^{-14}$; the bound and sign rule have no violations in the 3,926 non-vacuous cases. The script also checks the batch, leading-component and support-sharpness constructions. These computations are checks of the formulas, not replacements for their proofs.

\paragraph{Connecting $d_{\mathrm{proxy}}$ to the Jacobian spectrum.}
We define the entropy-based spectral effective rank by $r_{\mathrm{eff}}=\exp(-\sum_i p_i\log p_i)$, with $p_i=\sigma_i^2/\sum_j\sigma_j^2$. The empirical proxy $d_{\mathrm{proxy}}=d\cdot h$ is distinct from this measured quantity. Neither a parameter count nor a heterogeneity coefficient alone determines the singular values or the gradients' leading components. The conditional-estimator and spectral-bound experiments (Appendices~\ref{app:bound_check} and~\ref{app:diagnostic}) should therefore be interpreted as diagnostic checks rather than a universal identity between $d\cdot h$ and rank.

\begin{corollary}[Non-collinear parameter gradients]
\label{cor:positive}
Suppose the projections of $\mathbf g_{\mathrm{MSE}}$ and $\mathbf g_{\mathrm{task}}$ onto $\mathrm{col}(\mathbf J)$ are non-collinear. Their parameter-space images are then non-collinear. Writing $c$ for their cosine, the task-gradient component orthogonal to the MSE gradient has norm
\[
\|\nabla_\theta\ell_{\mathrm{task}}\|\sqrt{1-c^2}>0.
\]
\end{corollary}
\begin{proof}
On $\mathrm{col}(\mathbf{J})$, the map $\mathbf{J}^\top$ is injective. It therefore preserves non-collinearity of the two projected vectors assumed in the statement. The stated norm follows by orthogonal decomposition.
\end{proof}

\subsection{Counterexamples and Batch Geometry}
\label{app:batch}
Proposition~\ref{prop:batch} is an exact statement about the attainable first-order subspace at a fixed parameter value. It does not require a QP. Positive averaging weights can be absorbed into the rows of $J_{\mathrm{stack}}$ without changing the span; zero-weight examples can be omitted.

\paragraph{Different minimizers with one parameter.}
On $(0,3)$, set $f(\theta)=\theta$, $L_1=(f-1)^2$ and $L_2=(f-2)^2$. Then $J=1$, while $\arg\min L_1=1$ and $\arg\min L_2=2$. Their nonzero derivatives have opposite signs for $1<\theta<2$. At $\theta=1$, only the first derivative vanishes. Thus collinearity does not imply concordance, common stationary points, or common minimizers. The example can use the positive scalar covariance $\Sigma(\theta)=\theta$ directly.

\paragraph{Orthogonal batch gradients from rank-one examples.}
For $\theta\in\mathbb R^2$, $J_1=(1,0)$ and $J_2=(0,1)$ each have rank one, while $J_{\mathrm{stack}}=I_2$. With losses as in Section~\ref{sec:batch}, differentiation at zero gives $(1,1)$ and $(1,-1)$, whose inner product is zero. Positive covariance outputs can be obtained by replacing $f_i$ with $2+\theta_i$ near zero and shifting the loss targets correspondingly. The example establishes a failure of the proposed inference; it does not purport to reproduce the financial experiment.

\paragraph{A spectral gap without a leading task component.}
Let $J=\operatorname{diag}(1,\varepsilon)$, $g_1=e_1$ and $g_2=e_2$, where $\varepsilon>0$. Their parameter gradients remain orthogonal for every $\varepsilon$, although the singular-value ratio tends to zero. This is why the leading-component condition in Theorem~\ref{thm:spectral} cannot be omitted. At $\varepsilon=0$, the second parameter gradient vanishes, so its cosine is undefined.

\subsection{Proof of Proposition~\ref{prop:alignment}}
\label{app:proof}

\begin{proof}
We prove the gradient alignment bound in three steps.

\textbf{Step 1: Sparsity of the task gradient.}
Under the hard-$K$ formulation, the QP solution $\bw^* \in \mathbb{R}^K$ depends on $\bSigma$ through two quantities: the selected submatrix $\bSigma_{\cS\cS} \in \mathbb{R}^{K \times K}$ and the cross-covariance vector $\hat{\bsigma}_{\text{idx},\cS} = \bSigma_{\cS,:}\bw_{\text{idx}} \in \mathbb{R}^K$. Entry $(i,j)$ of $\bSigma$ contributes to the cross-covariance whenever $i \in \cS$ (regardless of $j$, since $\bw_{\text{idx}}$ has support on all $N$ assets), and to the sub-matrix whenever both $i,j \in \cS$. For any pair $(i,j)$ with $i \notin \cS$ and $j \notin \cS$, the entry $\bSigma_{ij}$ does not appear in either quantity. By the chain rule through the KKT conditions \citep{amos2017optnet}, $(\mathbf{g}_{\mathrm{task}})_{ij} = 0$ for all $(i,j) \in \bar{\mathcal{T}}$.

The task-relevant support is $\mathcal{T} = \{(i,j) : i \in \cS \text{ or } j \in \cS\}$. Its cardinality is $|\mathcal{T}| = N^2 - (N-K)^2 = 2NK - K^2$.

\textbf{Step 2: Cosine similarity bound via Cauchy--Schwarz.}
Since $\mathbf{g}_{\mathrm{task}}$ is supported on $\mathcal{T}$, we can write:
\begin{align}
\langle \mathbf{g}_{\mathrm{MSE}},\, \mathbf{g}_{\mathrm{task}} \rangle
&= \sum_{(i,j) \in \mathcal{T}} (\mathbf{g}_{\mathrm{MSE}})_{ij}\, (\mathbf{g}_{\mathrm{task}})_{ij} \notag \\
&= \langle \mathbf{g}_{\mathrm{MSE}}|_{\mathcal{T}},\, \mathbf{g}_{\mathrm{task}} \rangle,
\end{align}
where $\mathbf{g}_{\mathrm{MSE}}|_{\mathcal{T}}$ denotes the restriction of $\mathbf{g}_{\mathrm{MSE}}$ to the support $\mathcal{T}$. By Cauchy--Schwarz:
\begin{equation}
|\langle \mathbf{g}_{\mathrm{MSE}},\, \mathbf{g}_{\mathrm{task}} \rangle|^2 \leq \|\mathbf{g}_{\mathrm{MSE}}|_{\mathcal{T}}\|^2 \cdot \|\mathbf{g}_{\mathrm{task}}\|^2.
\end{equation}
Dividing both sides by $\|\mathbf{g}_{\mathrm{MSE}}\|^2 \cdot \|\mathbf{g}_{\mathrm{task}}\|^2$:
\begin{equation}
\cos^2(\mathbf{g}_{\mathrm{MSE}},\, \mathbf{g}_{\mathrm{task}}) \leq \frac{\|\mathbf{g}_{\mathrm{MSE}}|_{\mathcal{T}}\|^2}{\|\mathbf{g}_{\mathrm{MSE}}\|^2}.
\end{equation}

\textbf{Step 3: Isotropy assumption.}
For a fixed support, a sufficient assumption is isotropy of the normalized energy: $\mathbb{E}[g_{ij}^2/\|\mathbf g\|^2]=1/N^2$ for every coordinate. For data-dependent selection the corresponding conditional assumption is needed. Under this condition,
\begin{equation}
\mathbb{E}\!\left[\frac{\|\mathbf{g}_{\mathrm{MSE}}|_{\mathcal{T}}\|^2}{\|\mathbf{g}_{\mathrm{MSE}}\|^2}\right] = \frac{|\mathcal{T}|}{N^2} = \frac{2NK - K^2}{N^2} = \frac{2K}{N} - \frac{K^2}{N^2}.
\end{equation}
Combining with Step 2 and applying Jensen's inequality:
\begin{equation}
\mathbb{E}\!\left[|\cos(\mathbf{g}_{\mathrm{MSE}},\, \mathbf{g}_{\mathrm{task}})|\right] \leq \sqrt{\frac{2K}{N} - \frac{K^2}{N^2}}.
\end{equation}
For $K \ll N$, the $K^2/N^2$ term is negligible, yielding $\mathbb E|\cos(\mathbf{g}_{\mathrm{MSE}},\, \mathbf{g}_{\mathrm{task}})| \lesssim \sqrt{2K/N}$ under the same assumption.
\end{proof}

\paragraph{Scope of the energy assumption.}
Equal unnormalized second moments do not suffice: expectation of a ratio is not the ratio of expectations. Normalized-energy isotropy is an explicit sufficient condition, not a property of every covariance estimator. Selection can concentrate error energy; then the deterministic support-energy ratio in Step 2 remains valid, while the numerical $\sqrt{2K/N}$ bound need not hold.

\subsection{Sharpness of the Support-Only Inequality}
\label{app:lower_proof}
\begin{proposition}[Ambient-vector sharpness]
\label{prop:lower}
For a support $\mathcal T$ of size $m$ in $\mathbb R^{N^2}$, there are nonzero vectors $\mathbf g,\mathbf h$ with $\mathbf h$ supported on $\mathcal T$ and uniform squared coordinates in $\mathbf g$, such that $\cos(\mathbf g,\mathbf h)=\sqrt{m/N^2}$.
\end{proposition}
\begin{proof}
Set every coordinate of $\mathbf g$ to one and let $\mathbf h$ be the indicator of $\mathcal T$. Their inner product is $m$, and their norms are $N$ and $\sqrt m$. Their cosine is therefore $\sqrt m/N$, attaining the support-energy bound.
\end{proof}
This proves sharpness of the linear-algebra inequality only. The earlier equicorrelation construction did not establish a nonzero tracking gradient with the asserted structure. In fact, adding $\epsilon\mathbf1\mathbf1^\top$ leaves the fully invested tracking objective unchanged because $\mathbf1^\top(\mathbf w-\mathbf w_{\rm idx})=0$. We do not claim a matching lower bound realized by the tracking QP.

\subsection{Proof of Proposition~\ref{prop:invariance}}
\label{app:invariance_proof}

\begin{proof}
Fix a selection set $\cS$ with $|\cS| = K$. For any portfolio $\bw$ with $w_i = 0$ for $i \notin \cS$, define $\mathbf{d} = \bw - \bw_{\mathrm{idx}}$. Then $d_i = w_i - w_{\mathrm{idx},i}$ for $i \in \cS$ and $d_i = -w_{\mathrm{idx},i}$ for $i \notin \cS$, where the latter is fixed. The QP objective decomposes as:
\begin{align}
\mathbf{d}^\top \bSigma\, \mathbf{d} &= \mathbf{d}_\cS^\top \bSigma_{\cS\cS}\, \mathbf{d}_\cS + 2\, \mathbf{d}_\cS^\top \bSigma_{\cS\bar\cS}\, \mathbf{d}_{\bar\cS} + \mathbf{d}_{\bar\cS}^\top \bSigma_{\bar\cS\bar\cS}\, \mathbf{d}_{\bar\cS}.
\end{align}
Since $\mathbf{d}_{\bar\cS} = -\bw_{\mathrm{idx},\bar\cS}$ is fixed, the third term is a constant that does not affect the argmin, and the second term involves only the $K$-vector $\bSigma_{\cS\bar\cS}\, \bw_{\mathrm{idx},\bar\cS}$. Thus, two estimates $\hat\Sigma_1, \hat\Sigma_2$ with $(\hat\Sigma_1)_{\cS\cS} = (\hat\Sigma_2)_{\cS\cS}$ and $(\hat\Sigma_1)_{\cS\bar\cS}\, \bw_{\mathrm{idx},\bar\cS} = (\hat\Sigma_2)_{\cS\bar\cS}\, \bw_{\mathrm{idx},\bar\cS}$ yield identical objectives up to a constant, hence $\bw^*(\hat\Sigma_1) = \bw^*(\hat\Sigma_2)$.

The sufficient condition $(\hat\Sigma_1)_\mathcal{T} = (\hat\Sigma_2)_\mathcal{T}$ (equality on the full task-relevant support) is stronger than needed---it implies both block-equality and cross-product equality---but matches the gradient support from Proposition~\ref{prop:alignment}.
\end{proof}

\subsection{Task Regret and Local Perturbation}
\label{app:regret_proof}
\begin{proposition}[Off-support invariance of regret]
\label{prop:regret}
Fix the selected set and the true covariance used to evaluate decisions. Changing only the estimated covariance entries in $\bar\cS\times\bar\cS$ leaves the QP solution and its evaluated task regret unchanged (with a common tie-breaking rule if necessary).
\end{proposition}
\begin{proof}
Such changes preserve $Q=\hat\Sigma_{\cS\cS}$ and $b=\hat\Sigma_{\cS,:}w_{\rm idx}$, so they preserve the entire QP objective and feasible set. They therefore preserve its decision and any fixed evaluation of that decision.
\end{proof}

A local quantitative bound additionally requires stability assumptions. Let $Q=\Sigma_{\cS\cS}\succ0$ denote the true block, with $\lambda=\lambda_{\min}(Q)>0$. Write $\hat Q=Q+\Delta_Q$, $\hat b=b+\Delta_b$. Suppose $\|\Delta_Q\|_{\rm op}\leq\lambda/2$ and both solutions have the same active set. With $\delta w=\hat w-w^*$, subtracting the two KKT stationarity equations and multiplying by $\delta w$ gives
\[
\delta w^\top Q\delta w
 =\delta w^\top(\Delta_b-\Delta_Q\hat w),
\]
because $\mathbf1^\top\delta w=0$. Hence
\[
\|\delta w\|\leq\frac{2}{\lambda}
 (\|\Delta_Q\|_{\rm op}\|w^*\|+\|\Delta_b\|).
\]
For the tracking-variance objective $V(w)=(w-w_{\rm idx})^\top\Sigma(w-w_{\rm idx})$, KKT stationarity on the common free face makes the linear term in $V(\hat w)-V(w^*)$ exactly zero. Consequently,
\[
0\leq V(\hat w)-V(w^*)\leq
 \frac{4\|\Sigma\|_{\rm op}}{\lambda^2}
 (\|\Delta_Q\|_{\rm op}\|w^*\|+\|\Delta_b\|)^2.
\]
This is a local bound on tracking variance, not a global bound on annualized tracking-error standard deviation across active-set changes. The two perturbations involve only the selected block and index cross-covariance; they do not depend on estimated off-support entries.

\subsection{Selection-Robustness Motivation and Its Limits}
\label{app:dro_proof}
For any covariance error matrix $E$ and selected subset $\cS'$, $\|E_{\cS'\cS'}\|_{\mathrm{op}}\leq\|E\|_F$ and $\|E_{\cS',:}\bw_{\mathrm{idx}}\|\leq\|E\|_F\|\bw_{\mathrm{idx}}\|$. Thus full-matrix error control also controls each selected block and index cross-covariance term. Combined with a local QP perturbation bound under a stable active set, this supplies a conservative reason to regularize covariance predictions when selection changes.

An earlier argument inferred a worst-case coefficient $O(r^2/N^2)$ from the fraction of entries exposed by $r$ swaps. That inference is not valid without a bound on error concentration: one newly selected row can contain nearly all error energy. Nor can the oracle objective for one subset be replaced by that for another without a separate bound. We therefore make no minimax-equivalence or universal coefficient claim. BD-DFL is the empirical regularized objective in Eq.~\eqref{eq:bddfl_method}; its reported gains do not require that derivation. A theorem linking swap radius to a sharp regularization weight remains open.

\FloatBarrier
\section{Protocols, Provenance, and Computational Cost}
\label{app:protocol_group}

\subsection{Evidence Provenance and Reproduction Scope}
\label{app:provenance}
The source bundle compiles the paper and includes its figures. A separate reproduction bundle contains saved result files, figure generators, checks, and a snapshot of the local training source. The code snapshot does not reconstruct every historical run. The new controlled extension comprises 480 candidate fits and 160 selected models; the sensitivity study adds 720 optimization trajectories and 480 selected models. The fixed-class study adds 720 candidate fits and 240 selected models, bringing the new synthetic total to 1,920 fits and 880 selected models. The scalar financial target control saves 70 annual portfolio evaluations. The matched neural forward-target study adds 120 candidate trajectories and 60 selected models, with identical per-loss tuning budgets. These new financial controls use separate protocols and do not rerun the historical neural experiments. Historical spatial training was not rerun.

\paragraph{Current equity comparisons.}
The primary neural files identify the corrected rebalance harness. Each contains both a common naive baseline and a trained model, so selection must use model, loss, cardinality and fold jointly. For $K=20$, there are exactly nine neural rows per loss and 2,096 test-return observations. Recomputed sample-standard-deviation TE matches the saved metric in every selected fold. Earlier risk aggregation mixed in naive rows; the present risk table excludes them. Saved configurations and evaluation windows should be checked before combining other equity runs.

\paragraph{Evaluation-harness corrections.}
An earlier implementation held each fold's final rebalance to the end of the dataset instead of its test window. It also used calendar month-end labels: weekend month ends could be skipped, and labels beyond a window could extend holdings outside it. In validation, such extensions could reach test dates. The corrected harness uses the last trading day within each window and terminates holdings at that window's end. Current equity comparisons use the saved corrected runs; the archived diagnostic record does not.

These changes affected interpretation as well as magnitude. The corrected $N=451$, $K=10$ comparison favors pure DFL, contrary to an earlier overfitting description. For the nine-fold neural comparison at $K=20$, the corrected mean TE reduction is $16.5\%$. Source files, test windows and model identities are retained so that these comparisons can be distinguished from older evaluations.

\paragraph{Archived predictions.}
The 33-case diagnostic record predates the corrected equity evaluation. It contains 31 committed predictions with 27 correct outcomes and two abstentions. Timestamps and calibration separation have not been independently certified in this release. Threshold sensitivity and measured-$h$ re-scoring are post-hoc analyses and cannot be used as independent prospective validation.

\paragraph{Spatial and synthetic evidence.}
The synthetic heterogeneity sweep contains 18 settings with five folds each. NOAA and EPA provide only one or two evaluation windows per configuration, limiting inference about cross-domain replication. The listed sensor experiment has ten configurations with five trials. Hypothetical facility-location and resource-allocation applications are not additional experiments.

\paragraph{Combinatorial evidence.}
The custom shortest-path and PyEPO experiments have their own seeds, generators and training protocols. The custom shortest-path sweep measures input Jacobians. PyEPO rank estimates use a sampled stacked parameter Jacobian, whereas the equity bound experiment uses pointwise parameter Jacobians. The controlled extension uses all-output parameter probes. Full-capacity hyperparameter selection can favor that regime; unstable and negative cases are reported rather than removed by the rank argument.

\paragraph{What the checks establish.}
The numerical checkers compare selected table entries and figures with saved files. They do not certify data-provider revisions, every historical training configuration, independent market sampling, or the correctness of all scientific interpretations. The geometry script tests the stated identities on generated matrices and explicit counterexamples. The manuscript's mathematical arguments and experiment-design limitations remain necessary for interpretation.

\paragraph{Training-target limitation.}
The inspected equity trainer uses a trailing covariance reconstruction target for its MSE branch, while the task branch uses a forward return window. Consequently, these experiments compare specific implemented objectives rather than an optimal future-covariance forecasting baseline. For scalar shrinkage, reconstructing its own sample-covariance input is especially restrictive. Separate validation-tuned shrinkage and statistical baselines partly broaden the comparison, but a matched future-target forecasting experiment is still needed.

\subsection{Training Details}
\label{app:training}

\begin{algorithm}[htbp]
\caption{Accumulated-gradient training for sparse index tracking}
\label{alg:dfl}
\begin{algorithmic}[1]
\REQUIRE Predictor $f_\theta$, training dates $\mathcal D$, index weights, $K$, accumulation size $B=4$
\FOR{each epoch}
  \STATE Zero accumulated parameter gradients; set $c=0$ and $w_{\mathrm{prev}}$ unset
  \FOR{each valid date $t\in\mathcal D$}
    \STATE Form trailing features and covariance target using data before $t$
    \STATE Predict $\hat\Sigma_t=f_\theta(x_t)$; select a detached subset $\mathcal S_t$
    \STATE Solve the differentiable QP for $w_t^*$; record and skip solver failures
    \STATE Evaluate Eq.~\eqref{eq:task}; add Eq.~\eqref{eq:bddfl_method}'s penalty when $\beta>0$
    \STATE Accumulate $\nabla_\theta(\mathcal L_t/\min(B,|\mathcal D|))$; set $c\leftarrow c+1$
    \STATE Set $w_{\mathrm{prev}}\leftarrow\operatorname{stopgrad}(w_t^*)$
    \IF{$c$ is a multiple of $B$}
      \STATE Clip accumulated gradients; take an Adam step; zero gradients
    \ENDIF
  \ENDFOR
  \STATE Flush a nonempty partial accumulation with the same scaling
  \STATE Evaluate validation TE, update early stopping, and advance the scheduler
\ENDFOR
\end{algorithmic}
\end{algorithm}

The MSE-only branch uses its reconstruction loss without solving the QP during the backward pass. The pseudocode states the single-step task-training path and the accumulation convention in the inspected source. Partial accumulations retain the same denominator; they are not reweighted to an exact partial-batch mean. Previous holdings are detached, so the basic path does not differentiate through the full sequence of past decisions. Experiment-specific runners override some settings; saved configurations take precedence over the defaults below.

The default equity trainer uses Adam (lr $= 10^{-3}$, weight decay $10^{-4}$), cosine annealing, gradient clipping at norm 1.0, and gradient accumulation over 4 rebalancing dates. Early stopping with patience 10 monitors validation tracking error. Training uses weekly rebalancing; evaluation uses monthly rebalancing. The QP is solved via cvxpylayers \citep{agrawal2019differentiable} for backward passes and SCS for test-time inference. DFL training takes ${\sim}10$ min/fold on Apple M1 ($N{=}100$, 50 epochs).

\subsection{Computational Cost Comparison}
\label{app:cost}

Table~\ref{tab:cost} compares measured wall-clock cost and tracking error in one $N=100$, $K=20$ setting. The reported signed-rank comparisons do not detect a difference from POET; they are not equivalence tests. Timing and accuracy must be assessed jointly for the intended implementation, hardware and tolerance.

\begin{table}[h]
\centering
\caption{Computational cost per fold ($N{=}100$, $K{=}20$). Tracking errors are
regenerated from the result files by \texttt{scripts/emit\_table\_cost.py}; an
earlier version of this table carried pre-correction values (0.0285--0.0292) that
disagreed with Table~\ref{tab:exp3} for the same quantities. Timings are wall-clock for one fold, measured by
\texttt{scripts/measure\_training\_cost.py} and omitted where not measured. The
POET figure is its \emph{full} per-fold cost---the $3{\times}3$ validation grid over
(factors, threshold) plus a refit and QP solve at every rebalance date, 76 fits
in all---not a single estimate. Measured on one machine (Apple MPS); both absolute times and ratios depend on implementation and hardware.}
\label{tab:cost}

\fontsize{9}{10.8}\selectfont
\setlength{\tabcolsep}{5pt}
\renewcommand{\arraystretch}{1.08}
\publicationtable{%
\begin{tabular}{@{}lccl@{}}
\toprule
Method & Params & Time & TE \\
\midrule
POET + QP     & 0 (tuned) & 0.23\,s  & \textbf{0.0267} \\
GLasso + QP   & 0 (CV)    & ---      & 0.0269 \\
ValTuned      & 1         & ---      & 0.0269 \\
SPO+          & 1         & ---      & 0.0269 \\
LODL          & 12        & ---      & 0.0270 \\
Shrink-DFL    & 1         & 13\,min  & 0.0268 \\
Struct-DFL    & 12        & 3\,min   & 0.0267 \\
\bottomrule
\end{tabular}}
\end{table}

POET's reported per-fold cost includes the validation grid and refits. The saved timing record gives a roughly $3{,}329\times$ ratio for one-parameter DFL and $804\times$ for structured DFL; the displayed times are rounded. These measurements motivate a cost comparison, but do not imply that spending additional compute can never help a one-parameter predictor or that the ratios transfer to other implementations.

\FloatBarrier
\section{Equity Comparisons and Robustness}
\label{app:equity_group}
\begin{figure}[t]
\centering
\includegraphics[width=\figurewidth]{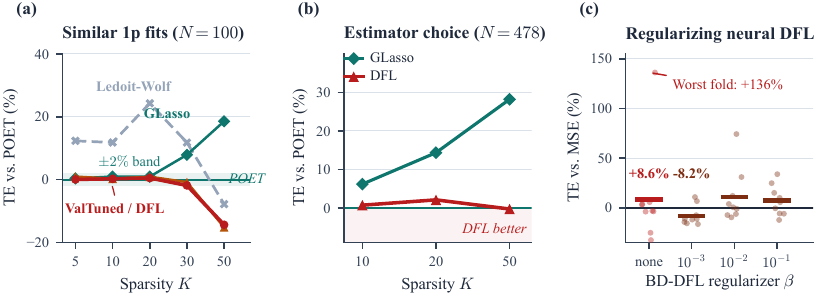}
\caption{\textbf{Loss choice, model capacity, and selection stability.} (a)--(b) Tracking error relative to POET; (c) relative to MSE (lower is better). At $N{=}100$, the three fitted one-parameter estimators improve on Ledoit--Wolf but meet the empirical equivalence criterion in Table~\ref{tab:exp3}. At $N{=}478$, estimator choice changes the ordering. Under dynamic selection, BD-DFL reduces the extreme errors of unregularized neural DFL; dots show individual folds.}
\label{fig:financial_overview}
\end{figure}

\subsection{DFL Gains across Sparsity Levels}
\label{app:sparsity}

We sweep the cardinality budget $K \in \{5, 10, 20, 30, 50\}$ using the neural model under hard-$K$ selection.

\begin{table}[h]
\centering
\caption{Neural model: DFL vs.\ MSE across sparsity levels. 95\% CI from 10{,}000-sample bootstrap; $p$-values from Wilcoxon signed-rank.}
\label{tab:exp1}

\fontsize{9}{10.8}\selectfont
\setlength{\tabcolsep}{5pt}
\renewcommand{\arraystretch}{1.08}
\publicationtable{%
\begin{tabular}{@{}ccccccc@{}}
\toprule
$K$ & MSE TE & DFL TE & Gain & 95\% CI & $p$-value & Wins \\
\midrule
5  & 0.0859 & 0.0784 & $-8.8\%$  & $[-13.9, -3.6]$ & 0.014 & 8/9 \\
10 & 0.0510 & 0.0462 & $-9.4\%$  & $[-13.1, -5.1]$ & 0.004 & 8/9 \\
20 & 0.0331 & 0.0276 & $-16.5\%$  & $[-26.1, -7.1]$ & 0.002 & 9/9 \\
30 & 0.0208 & 0.0193 & $-7.0\%$  & $[-15.1, -0.7]$ & 0.004 & 8/9 \\
50 & 0.0096 & 0.0096 & $-0.2\%$  & $[-0.5, +0.1]$ & 0.213 & 5/9 \\
\bottomrule
\end{tabular}}
\end{table}

Table~\ref{tab:exp1} and Figure~\ref{fig:main}(a) show neural DFL reductions of 7.0--16.5\% at the tested $K\leq30$, with one-sided Wilcoxon $p\leq0.014$. At $K=50$, the estimated reduction is 0.2\% ($p=0.213$, 5/9 wins). The largest reduction occurs at $K=20$ rather than the smallest portfolio. These finite comparisons support a dependence on the experimental setting, not a universal monotone law in $K/N$.

\paragraph{Tracking-risk summaries.}
For the same nine neural-model folds, worst-fold TE decreases from 0.0564 to 0.0360 and the standard deviation of fold TE decreases from 0.0098 to 0.0045. Recomputed daily-tail loss and arithmetic tracking drawdown decrease by 14.3\% and 15.5\%, respectively (Appendix~\ref{app:risk}). These are descriptive summaries of the evaluated folds, not prospective risk guarantees.

\begin{figure}[h]
\centering
\includegraphics[width=\figurewidth]{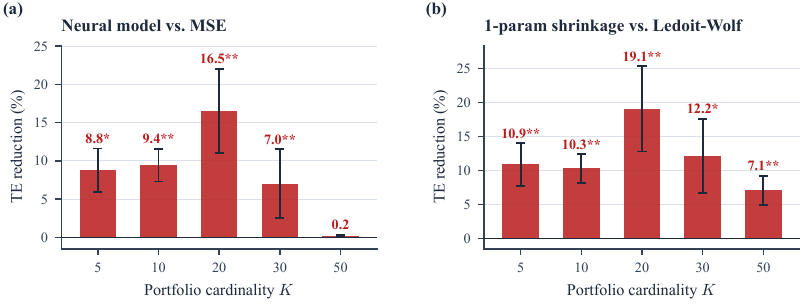}
\caption{\textbf{Tracking-error reduction over nine paired folds.} Bars show $100(1-\overline{\mathrm{TE}}_{\rm DFL}/\overline{\mathrm{TE}}_{\rm base})$; whiskers are $\pm1$ paired delta-method SE of this ratio, a descriptive summary across folds. Unadjusted one-sided Wilcoxon: $^{*}p<0.05$, $^{**}p<0.01$, $^{***}p<0.001$. (a) Neural DFL vs. MSE: reductions of 7.0--16.5\% at $K\leq30$, and 0.2\% at $K=50$. (b) Shrinkage DFL vs. Ledoit--Wolf: reductions of 7.1--19.1\%; comparisons with other fitted one-parameter estimators are in Table~\ref{tab:exp3}.}
\label{fig:main}
\end{figure}

Figure~\ref{fig:cumulative} shows the corrected test-return series. Each 63-day window lies within one fold; windows crossing train/test boundaries are excluded.

\begin{figure}[h]
\centering
\includegraphics[width=\figurewidth]{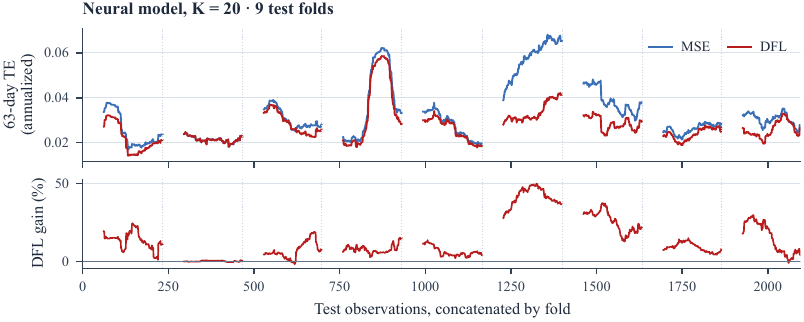}
\caption{\textbf{Corrected rolling tracking error for the neural model at $K{=}20$: 2,096 observations across nine test folds.} Top: annualized 63-day standard deviation of tracking returns. Bottom: relative TE reduction from DFL. Dotted lines mark fold boundaries; gaps exclude each fold's first 62 observations.}
\label{fig:cumulative}
\end{figure}

\subsection{Per-Fold Results}
\label{app:folds}

Table~\ref{tab:folds_exp1} reports per-fold tracking errors for Experiment~1 (neural model, $K{=}20$), demonstrating consistency across diverse market regimes.

\begin{table}[h]
\centering
\caption{Per-fold TE for neural model at $K{=}20$ (Experiment~1). DFL wins 9/9 folds, with the largest gains concentrated in the high-volatility regimes of 2021--2023 ($-40.0\%$ and $-26.6\%$) and the smallest in the calmest year, 2017--2018 ($-0.3\%$).}
\label{tab:folds_exp1}

\fontsize{9}{10.8}\selectfont
\setlength{\tabcolsep}{5pt}
\renewcommand{\arraystretch}{1.08}
\publicationtable{%
\begin{tabular}{@{}clccc@{}}
\toprule
Fold & Test period & MSE TE & DFL TE & Gain \\
\midrule
0 & 2016--2017 & 0.0260 & 0.0223 & $-14.3\%$ \\
1 & 2017--2018 & 0.0224 & 0.0223 & $-0.3\%$ \\
2 & 2018--2019 & 0.0306 & 0.0287 & $-6.2\%$ \\
3 & 2019--2020 & 0.0389 & 0.0360 & $-7.5\%$ \\
4 & 2020--2021 & 0.0279 & 0.0260 & $-6.7\%$ \\
5 & 2021--2022 & 0.0564 & 0.0339 & $-40.0\%$ \\
6 & 2022--2023 & 0.0394 & 0.0289 & $-26.6\%$ \\
7 & 2023--2024 & 0.0264 & 0.0242 & $-8.4\%$ \\
8 & 2024--2025 & 0.0299 & 0.0263 & $-11.8\%$ \\
\bottomrule
\end{tabular}}
\end{table}

\subsection{Sparsity vs.\ Model Capacity Interaction}
\label{app:interaction}

Our central hypothesis predicted that DFL gains depend on both sparsity ($K$) and model misspecification ($K_f$). To test this, we run a full $K \times K_f$ interaction grid: factor model with $K_f \in \{3, 5, 10, 20\}$ and $K \in \{5, 10, 20, 50\}$, yielding 16 cells $\times$ 2 losses $\times$ 9 folds = 288 experimental runs.

\begin{table}[h]
\centering
\caption{DFL gain (\%) over MSE: full $K \times K_f$ interaction grid (factor model, nine paired folds per cell). Unadjusted one-sided Wilcoxon: $^{*}p<0.05$, $^{**}p<0.01$, $^{***}p<0.001$, matching Figure~\ref{fig:grid}.}
\label{tab:exp2}

\fontsize{9}{10.8}\selectfont
\setlength{\tabcolsep}{5pt}
\renewcommand{\arraystretch}{1.08}
\publicationtable{%
\begin{tabular}{@{}c|cccc@{}}
\toprule
$K_f \backslash K$ & 5 & 10 & 20 & 50 \\
\midrule
3  & $-8.8\%^{**}$ & $-10.2\%^{**}$ & $-15.9\%^{**}$ & $-0.4\%$ \\
5  & $-8.6\%^{**}$ & $-9.8\%^{**}$ & $-14.9\%^{**}$ & $-2.1\%$ \\
10 & $-9.2\%^{**}$ & $-10.2\%^{**}$ & $-15.1\%^{**}$ & $-2.5\%$ \\
20 & $-8.6\%^{*}$ & $-9.7\%^{**}$ & $-13.9\%^{**}$ & $-1.1\%$ \\
\bottomrule
\end{tabular}}
\end{table}

Table~\ref{tab:exp2} and Figure~\ref{fig:grid} show that relative TE changes vary more across portfolio cardinalities $K$ than across factor counts $K_f$ in this experiment. At $K=20$, DFL reduces TE by $13.9$--$15.9\%$ across factor counts; at $K=50$, reductions are $0.4$--$2.5\%$. Within each fixed-$K$ column, the spread across factor counts is at most $2.1$ percentage points.

This pattern is consistent with a role for decision sparsity in the difference between the two objectives. The grid alone does not isolate sparsity from predictor capacity or establish a causal mechanism. Similar relative gains across factor counts also do not imply similar absolute MSE baselines. We therefore interpret the grid as an empirical comparison alongside the conditional geometric analysis in the main text.

\begin{figure}[h]
\centering
\includegraphics[width=.74\figurewidth]{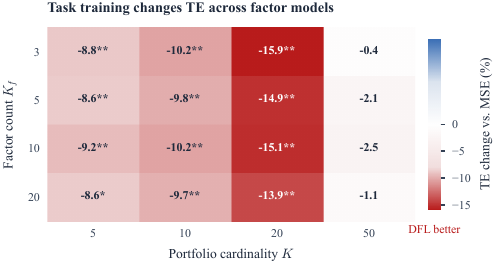}
\caption{\textbf{Task training across factor counts and portfolio cardinalities (nine paired folds per cell).} Values are relative TE changes, $100(\overline{\mathrm{TE}}_{\rm DFL}/\overline{\mathrm{TE}}_{\rm MSE}-1)$; negative values indicate improvement. In this grid, changes vary more across $K$ than across $K_f$, and improvements are smaller at $K=50$. Unadjusted one-sided Wilcoxon: $^{*}p<0.05$, $^{**}p<0.01$, $^{***}p<0.001$. The heatmap does not identify a causal effect of either axis.}
\label{fig:grid}
\end{figure}

\subsection{DFL with Richer Model Classes}
\label{app:richer}

We introduce two richer shrinkage models: (i) \textbf{Conditional shrinkage} ($\sim$385 params): $\alpha_t = \text{MLP}(\mathbf{z}_t)$ with market-level features, and (ii) \textbf{Structured shrinkage} (12 params): $\hat\Sigma = (1-\alpha)S + \alpha T$ with learnable per-sector variance targets. For baseline comparison, we use MSE training and \emph{random search} (200 trials for conditional, 500 for structured).

\begin{table}[h]
\centering
\caption{Richer model classes: MSE vs.\ DFL, 9 folds, one-sided Wilcoxon ($^{*}p{<}0.05$, $^{**}p{<}0.01$). The random-search rows of an earlier version are withheld pending a re-run: those runs stored no return series, so they cannot be placed on the corrected evaluation window that every row here uses.}
\label{tab:exp6}

\fontsize{9}{10.8}\selectfont
\setlength{\tabcolsep}{5pt}
\renewcommand{\arraystretch}{1.08}
\publicationtable{%
\begin{tabular}{@{}llcccc@{}}
\toprule
$K$ & Method & TE & TO & $p$ & Wins \\
\midrule
\multirow{2}{*}{10}
& Cond.\ Shrinkage + MSE          & 0.0462 & 0.034 & --- & --- \\
& \textbf{Cond.\ Shrinkage + DFL} & \textbf{0.0459} & 0.040 & $0.010^{**}$ & 8/9 \\
\cmidrule{2-6}
& Struct.\ Shrinkage + MSE         & 0.0459 & 0.041 & --- & --- \\
& \textbf{Struct.\ Shrinkage + DFL} & \textbf{0.0453} & 0.043 & $0.014^{*}$ & 8/9 \\
\midrule
\multirow{2}{*}{20}
& Cond.\ Shrinkage + MSE          & 0.0273 & 0.036 & --- & --- \\
& Cond.\ Shrinkage + DFL          & 0.0270 & 0.040 & $0.020^{*}$ & 7/9 \\
\cmidrule{2-6}
& Struct.\ Shrinkage + MSE         & 0.0270 & 0.045 & --- & --- \\
& \textbf{Struct.\ Shrinkage + DFL} & \textbf{0.0267} & 0.046 & $0.010^{**}$ & 8/9 \\
\bottomrule
\end{tabular}}
\end{table}

Table~\ref{tab:exp6} shows that \textbf{DFL helps both richer model classes, and helps the smaller one more}. The 12-parameter structured model gains $1.2\%$ at $K{=}10$ and $1.6\%$ at $K{=}20$ ($p = 0.002$, 9/9 folds), while the 385-parameter conditional model gains only $0.7\%$ and $1.0\%$. Parameter count alone does not predict the relative gain. The conditional model also trains less stably (validation-to-training loss ratios range from $0.64$ to $3.42$ across folds). These results favor the structured model in this comparison, but do not isolate the reason: input-dependent batch geometry, optimization and model specification can all affect the ordering.

\begin{figure}[h]
\centering
\includegraphics[width=.80\figurewidth]{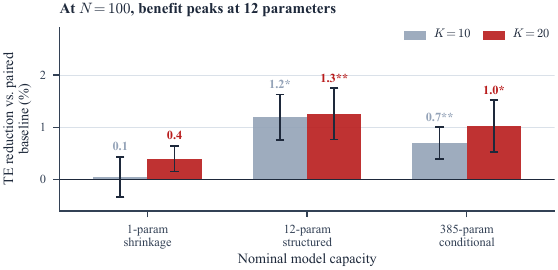}
\caption{\textbf{DFL gains for three model classes at $N=100$, using nine paired folds.} Baselines are validation-tuned shrinkage (one parameter) and MSE training (structured and conditional models). Bars show relative reduction in mean TE; whiskers are $\pm1$ paired delta-method SE. Unadjusted one-sided Wilcoxon: $^{*}p<0.05$, $^{**}p<0.01$, $^{***}p<0.001$. The 12-parameter model has the largest observed reduction here; the ordering differs at $N=478$ (Table~\ref{tab:n500}). Architecture changes alongside parameter count.}
\label{fig:model_complexity}
\end{figure}

\subsection{Cross-Market Dim-1 Validation}
\label{app:ftse}

\begin{table}[h]
\centering
\caption{Cross-market dim-1 validation: the 1-parameter shrinkage sweep across three sparsity levels, re-run under the repaired harness. DFL gain stays at or below $1.76\%$ across 5 markets, 3 sparsity levels ($K/N$ from $0.02$ to $0.43$) and 120 paired folds, and is unordered in $K/N$, an empirical observation separate from Proposition~\ref{prop:dim1} and Corollary~\ref{cor:vacuous}. Even where gains are statistically significant (e.g., Hang Seng $K{=}10$, $p{=}0.012$), their value depends on the application tolerance and computational budget.}
\label{tab:crossmarket_full}

\fontsize{9}{10.8}\selectfont
\setlength{\tabcolsep}{5pt}
\renewcommand{\arraystretch}{1.08}
\publicationtable{%
\begin{tabular}{@{}lrrrrrr@{}}
\toprule
Market & $N$ & $K$ & MSE TE & DFL TE & Shrink $\Delta\%$ & $p$ \\
\midrule
\multirow{3}{*}{FTSE 100} & 91 & 5 & 0.0896 & 0.0896 & $-0.02$ & 0.770 \\
 & 91 & 10 & 0.0516 & 0.0516 & $+0.01$ & 0.473 \\
 & 91 & 20 & 0.0261 & 0.0260 & $+0.17$ & 0.039 \\
\multirow{3}{*}{Nikkei 225} & 218 & 5 & 0.1106 & 0.1101 & $+0.46$ & 0.125 \\
 & 218 & 10 & 0.0823 & 0.0821 & $+0.17$ & 0.039 \\
 & 218 & 20 & 0.0511 & 0.0511 & $+0.02$ & 0.578 \\
\multirow{3}{*}{Euro Stoxx 50} & 47 & 5 & 0.0731 & 0.0731 & $+0.06$ & 0.074 \\
 & 47 & 10 & 0.0331 & 0.0330 & $+0.29$ & 0.055 \\
 & 47 & 20 & 0.0140 & 0.0140 & $+0.08$ & 0.191 \\
\multirow{3}{*}{ASX 200} & 159 & 5 & 0.0778 & 0.0769 & $+1.21$ & 0.004 \\
 & 159 & 10 & 0.0441 & 0.0437 & $+0.98$ & 0.004 \\
 & 159 & 20 & 0.0274 & 0.0273 & $+0.38$ & 0.004 \\
\multirow{3}{*}{Hang Seng} & 62 & 5 & 0.0803 & 0.0793 & $+1.28$ & 0.027 \\
 & 62 & 10 & 0.0441 & 0.0434 & $+1.76$ & 0.012 \\
 & 62 & 20 & 0.0223 & 0.0223 & $+0.42$ & 0.055 \\
\bottomrule
\end{tabular}}
\end{table}

The expanded cross-market results reveal a consistent pattern: shrinkage (1p) DFL gains stay at or below $1.76\%$ across all 15 market-$K$ combinations (120 paired folds), at $K/N$ from $0.02$ to $0.43$; only three clear $1\%$. Even where Wilcoxon $p < 0.05$ (ASX $K{=}10$: $p{=}0.004$; Hang Seng $K{=}10$: $p{=}0.012$), the absolute gains ($+0.98\%$, $+1.76\%$) are small in relative terms; the $1.76\%$ case exceeds the diagnostic's 1\% cutoff. The largest gain (Hang Seng $+1.76\%$) is not the largest $K/N$ in the sweep: Euro Stoxx at $K{=}20$ reaches $K/N{=}0.43$ and gains $+0.08\%$. Within this table the gain is unordered in $K/N$, which is not implied by Corollary~\ref{cor:vacuous}. The runtime comparison is confined to Appendix~\ref{app:cost}; no market-wide cost-benefit threshold is established.

\paragraph{Shared-run comparison with the structured model.}
Table~\ref{tab:crossmarket_4mkt} reports a later, self-contained run over four of these markets that trains \emph{both} the 1-parameter and the 12-parameter model on the current data pull, at $K \in \{10, 20\}$, 8 folds each. It supplies the structured-model column of Table~\ref{tab:crossmarket}, and the contrast within each row is the point: on the same folds and the same data, the two model classes have different relative changes on the same folds; the ordering and magnitude vary across markets and $K$.

We flag one thing an earlier draft claimed and this one does not. When both tables ran on different data pulls, their shrinkage columns constituted an independent replication, and we reported it as such. Re-running everything under the repaired harness put both on the current pull, the same splitter and the same per-fold seeds, and their overlapping cells ($K \in \{10,20\}$, four markets) are now identical to eight decimal places. That is a determinism check, not independent evidence, and we no longer describe it as the latter. Table~\ref{tab:crossmarket_full}'s remaining independent contribution is its $K{=}5$ column and its fifth market.

\begin{table}[h]
\centering
\caption{Four-market run training both model classes on one data pull (8 folds each). The shrinkage columns share a configuration with Table~\ref{tab:crossmarket_full} and match it exactly where they overlap; the structured columns supply Table~\ref{tab:crossmarket}. TE is the MSE-trained baseline; $\Delta\%$ is the DFL gain over it. $^{*}p{<}0.05$, $^{**}p{<}0.01$, one-sided Wilcoxon.}
\label{tab:crossmarket_4mkt}

\fontsize{9}{10.8}\selectfont
\setlength{\tabcolsep}{5pt}
\renewcommand{\arraystretch}{1.08}
\publicationtable{%
\begin{tabular}{@{}lrrrrrrrr@{}}
\toprule
& & & \multicolumn{3}{c}{Shrinkage (1p)} & \multicolumn{3}{c}{Structured (12p)} \\
\cmidrule(lr){4-6} \cmidrule(lr){7-9}
Market & $N$ & $K$ & TE & $\Delta\%$ & $p$ & TE & $\Delta\%$ & $p$ \\
\midrule
\multirow{2}{*}{Nikkei 225} & 218 & 10 & 0.0823 & $+0.17$ & $0.039^{*}$ & 0.0826 & $+0.87$ & $0.027^{*}$ \\
 & & 20 & 0.0511 & $+0.02$ & 0.578 & 0.0511 & $+0.20$ & 0.055 \\
\midrule
\multirow{2}{*}{Euro Stoxx 50} & 47 & 10 & 0.0331 & $+0.29$ & 0.055 & 0.0329 & $\mathbf{+0.95}$ & $0.008^{**}$ \\
 & & 20 & 0.0140 & $+0.08$ & 0.191 & 0.0141 & $+0.15$ & 0.230 \\
\midrule
\multirow{2}{*}{ASX 200} & 159 & 10 & 0.0441 & $+0.98$ & $0.004^{**}$ & 0.0436 & $\mathbf{+2.95}$ & $0.004^{**}$ \\
 & & 20 & 0.0274 & $+0.38$ & $0.004^{**}$ & 0.0273 & $\mathbf{+1.27}$ & $0.004^{**}$ \\
\midrule
\multirow{2}{*}{Hang Seng} & 62 & 10 & 0.0441 & $+1.76$ & $0.012^{*}$ & 0.0433 & $\mathbf{+2.77}$ & $0.008^{**}$ \\
 & & 20 & 0.0223 & $+0.42$ & 0.055 & 0.0223 & $\mathbf{+0.57}$ & $0.004^{**}$ \\
\bottomrule
\end{tabular}}
\end{table}

\subsection{Cross-Benchmark Validation (Russell 2000 / Tech Sector)}
\label{app:crossbench}

To test whether DFL's $K/N$ theory generalizes beyond the S\&P~500, we evaluate the 1-parameter shrinkage model on two additional benchmarks: Russell~2000 (proxied by S\&P~600, $N{=}100$, tracked against IWM) and S\&P~500 Information Technology sector ($N{=}48$, tracked against XLK). Table~\ref{tab:alt} reports results.

\begin{table}[h]
\centering
\caption{Cross-benchmark validation (1-parameter shrinkage). Every gain is under $0.3\%$ across a small-cap index and a single sector, including at $K/N{=}0.42$; these rows extend the observed small-gain pattern to additional tested universes.}
\label{tab:alt}

\fontsize{9}{10.8}\selectfont
\setlength{\tabcolsep}{5pt}
\renewcommand{\arraystretch}{1.08}
\publicationtable{%
\begin{tabular}{@{}llrccc@{}}
\toprule
Benchmark & $K$ & $K/N$ & MSE TE & DFL TE & Gain \\
\midrule
S\&P~500 ($N{=}100$) & 20 & 0.206 & 0.0269 & 0.0268 & $-0.3\%$ \\
Russell~2000 ($N{=}100$) & 10 & 0.100 & 0.0797 & 0.0797 & $-0.0\%$ \\
Russell~2000 ($N{=}100$) & 20 & 0.200 & 0.0492 & 0.0491 & $-0.2\%$ \\
Tech Sector ($N{=}48$) & 5  & 0.104 & 0.0874 & 0.0873 & $-0.1\%$ \\
Tech Sector ($N{=}48$) & 10 & 0.208 & 0.0524 & 0.0522 & $-0.3\%$ \\
Tech Sector ($N{=}48$) & 20 & 0.417 & 0.0222 & 0.0221 & $-0.2\%$ \\
\bottomrule
\end{tabular}}
\end{table}

Across three indices spanning large-cap, small-cap and sector portfolios, the 1-parameter shrinkage model shows near-zero DFL advantage---every gain under $0.3\%$, including the Tech Sector at $K/N{=}0.42$---consistent with the low-capacity pattern observed here; collinearity alone does not imply redundancy.

\subsection{BD-DFL Detailed Results}
\label{app:bddfl_detail}

\begin{table}[h]
\centering
\caption{Dynamic selection at $N{=}100$ (nine paired folds). Entries below the baseline are relative changes in mean TE; negative is better. Best $\beta$ denotes the best tested setting, an exploratory comparison.}
\label{tab:bddfl_dynamic_audit}

\fontsize{9}{10.8}\selectfont
\setlength{\tabcolsep}{5pt}
\renewcommand{\arraystretch}{1.08}
\publicationtable{%
\begin{tabular}{@{}lrrrr@{}}
\toprule
 & \multicolumn{2}{c}{$K{=}10$} & \multicolumn{2}{c}{$K{=}20$} \\
\cmidrule(lr){2-3}\cmidrule(lr){4-5}
Method & Shrink & Neural & Shrink & Neural \\
\midrule
MSE baseline & 0.0629 & 0.0512 & 0.0377 & 0.0327 \\
Pure DFL & $-5.4\%$ & $+8.6\%$ & $-0.9\%$ & $+6.7\%$ \\
BD-DFL (best tested) & $-10.3\%$ & $-8.2\%$ & $-7.0\%$ & $-11.6\%$ \\
\bottomrule\end{tabular}}\end{table}

\begin{table}[h]
\centering
\caption{Dynamic selection at $N{=}478$ (five paired folds). MSE TE / pure-DFL change / best-tested BD-DFL change. Best-tested comparisons are exploratory.}
\label{tab:bddfl_n478_audit}

\fontsize{9}{10.8}\selectfont
\setlength{\tabcolsep}{5pt}
\renewcommand{\arraystretch}{1.08}
\publicationtable{%
\begin{tabular}{@{}lrr@{}}\toprule
Model & $K{=}10$ & $K{=}20$ \\ \midrule
Shrinkage & --- & 0.0586 / $-0.8\%$ / $-8.3\%$ \\
Structured & --- & 0.0555 / $-2.4\%$ / $-5.5\%$ \\
Neural & 0.0698 / $+0.7\%$ / $-7.4\%$ & 0.0527 / $-0.1\%$ / $+0.5\%$ \\
\bottomrule\end{tabular}}\end{table}

\begin{table}[h]
\centering
\caption{Corrected neural BD-DFL results at $N{=}451$ (11 paired folds, dynamic selection). Changes use the ratio of fold means; $p$ is one-sided Wilcoxon against MSE.}
\label{tab:bddfl_n451}

\fontsize{9}{10.8}\selectfont
\setlength{\tabcolsep}{5pt}
\renewcommand{\arraystretch}{1.08}
\publicationtable{%
\begin{tabular}{@{}llcccc@{}}\toprule
$K$ & Method & Mean TE & vs.\ MSE & Wins & $p$ \\ \midrule
10 & MSE baseline & 0.0734 & --- & --- & --- \\
 & Pure DFL & 0.0638 & $-13.0\%$ & 8/11 & 0.0093 \\
 & BD-DFL ($\beta=0.001$) & 0.0620 & $-15.6\%$ & 9/11 & 0.0024 \\
20 & MSE baseline & 0.0516 & --- & --- & --- \\
 & Pure DFL & 0.0539 & $+4.5\%$ & 5/11 & 0.6499 \\
 & BD-DFL ($\beta=0.001$) & 0.0487 & $-5.5\%$ & 8/11 & 0.1030 \\
\bottomrule\end{tabular}}\end{table}

The corrected $N=451$ results differ from the earlier harness: unregularized DFL improves at $K=10$, while regularization adds a smaller further reduction. At $K=20$ its direction is favorable but not significant at 5\%. The choice $\beta=0.001$ is empirical; proximity to $r^2/N^2$ would not establish a minimax law.

\begin{figure}[h]
\centering
\includegraphics[width=\figurewidth]{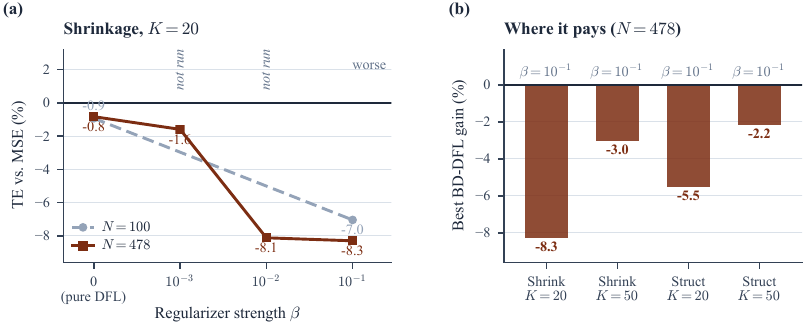}
\caption{\textbf{Regularization under dynamic selection.} (a) Shrinkage at $K=20$ across the tested coefficients; dashed segments bridge untested settings. (b) Best-tested coefficient for each model and sparsity at $N=478$. These exploratory comparisons do not establish that regularization gains grow monotonically with universe size.}
\label{fig:scaling_bddfl}
\end{figure}

\subsection{Turnover Penalty Ablation}
\label{app:turnover}

DFL methods produce higher turnover than MSE. A natural concern is that DFL's TE advantage is an artifact of excessive trading. We ablate the turnover penalty $\gamma \in \{0, 0.01, 0.1, 1.0\}$ for the neural DFL model at $K = 20$.

\begin{table}[h]
\centering
\caption{Turnover-TE Pareto frontier (Neural, $K = 20$).}
\label{tab:exp4}

\fontsize{9}{10.8}\selectfont
\setlength{\tabcolsep}{5pt}
\renewcommand{\arraystretch}{1.08}
\publicationtable{%
\begin{tabular}{@{}lcccc@{}}
\toprule
Method & $\gamma$ & TE & Turnover & vs MSE \\
\midrule
MSE baseline  & ---  & 0.0332 & 0.003 & --- \\
DFL            & 0    & 0.0275 & 0.030 & $-17.1\%$ \\
DFL            & 0.01 & 0.0290 & 0.005 & $-12.8\%$ \\
DFL            & 0.1  & 0.0313 & 0.000 & $-5.7\%$ \\
DFL            & 1.0  & 0.0332 & 0.000 & $-0.1\%$ \\
\bottomrule
\end{tabular}}
\end{table}

Table~\ref{tab:exp4} and Figure~\ref{fig:pareto} show the trade-off at the tested coefficients. At $\gamma=0.01$, mean TE improves by 12.8\%, while turnover remains higher than MSE (0.005 vs.\ 0.003). Larger penalties reduce turnover and erode the observed TE advantage. These measurements do not establish dominance at every turnover level or a universal optimal penalty.

\begin{figure}[h]
\centering
\includegraphics[width=.70\figurewidth]{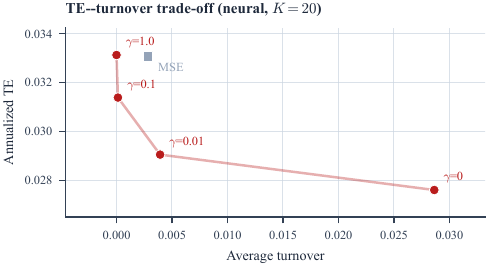}
\caption{\textbf{Tracking-error--turnover trade-off for the tested neural configurations at $K=20$.} Points are fold means; connectors join the sampled regularization coefficients and do not certify a continuous Pareto frontier. The MSE reference is shown separately. At $\gamma=0.01$, mean TE is lower than MSE while turnover is higher (0.005 vs. 0.003; Table~\ref{tab:exp4}); the preferred trade-off depends on trading costs.}
\label{fig:pareto}
\end{figure}

\subsection{Transaction Cost-Aware Optimization}
\label{app:txcost}

We augment the QP with an explicit transaction cost penalty: $\min_\bw \bw^\top \bSigma \bw - 2\bw^\top \bsigma_{\text{idx}} + c \|\bw - \bw_{\text{prev}}\|_1$, where $\bw_{\text{prev}}$ is the previous portfolio and $c$ is the cost coefficient.

\begin{table}[h]
\centering
\caption{Transaction cost experiment. Net cost = TE + $c \times$ TO. All rows use the TC-aware QP. Val-tuned grid search is omitted: no stored run for it survives, and comparing a window-corrected DFL row against an uncorrected baseline would be invalid.}
\label{tab:exp5}

\fontsize{9}{10.8}\selectfont
\setlength{\tabcolsep}{5pt}
\renewcommand{\arraystretch}{1.08}
\publicationtable{%
\begin{tabular}{@{}llcccc@{}}
\toprule
$K$ & Method & $c$ (bps) & TE & TO & Net Cost \\
\midrule
\multirow{4}{*}{10}
& Neural MSE + TC & 10  & 0.0510 & 0.003 & 0.0510 \\
& Neural DFL + TC & 10  & 0.0462 & 0.041 & \textbf{0.0462} \\
\cmidrule{2-6}
& Neural MSE + TC & 100 & 0.0510 & 0.001 & 0.0511 \\
& Neural DFL + TC & 100 & 0.0462 & 0.041 & \textbf{0.0466} \\
\midrule
\multirow{4}{*}{20}
& Neural MSE + TC & 10  & 0.0331 & 0.002 & 0.0331 \\
& Neural DFL + TC & 10  & 0.0277 & 0.026 & \textbf{0.0278} \\
\cmidrule{2-6}
& Neural MSE + TC & 100 & 0.0331 & 0.000 & 0.0331 \\
& Neural DFL + TC & 100 & 0.0277 & 0.025 & \textbf{0.0280} \\
\bottomrule
\end{tabular}}
\end{table}

Table~\ref{tab:exp5} shows that DFL's tracking-error advantage survives transaction costs. DFL trades roughly $15\times$ as much as the MSE baseline (turnover 0.030--0.046 vs.\ 0.002--0.003), which is the natural objection to any task-trained policy. But the cost of that turnover is small next to the TE it buys: even at $c = 100$ bps it adds only 4~bps of net cost at $K{=}10$ and 3~bps at $K{=}20$, against a TE reduction of 49 and 57~bps respectively. DFL therefore wins on net cost at both cost levels and both sparsity settings, and the margin narrows only slightly as $c$ rises.

\subsection{Design Choices and Robustness}
\label{app:design}

\paragraph{Hard-$K$ vs.\ $\ell_1$ relaxation.}
Since $\bw \geq 0$ and $\mathbf{1}^\top \bw = 1$ imply $\|\bw\|_1 = 1$, an $\ell_1$ penalty cannot induce sparsity in long-only portfolios. With elastic-net $\lambda\|\bw\|_2^2$ at $\lambda = 0.01$, the portfolio retains 97 of 100 stocks (TE $= 0.0065$), making fair comparison with hard-$K$ ($K{=}20$, TE $= 0.0332$) impossible. All experiments therefore use hard-$K$.

\paragraph{Fixed vs.\ dynamic stock selection.}
Dynamic selection (re-selecting stocks at each rebalance by tracking score) combined with DFL \emph{hurts} performance: TE $= 0.0346$ (vs.\ 0.0287 for fixed) with turnover $6\times$ higher. This comparison changes the selection rule as well as the resulting holdings. It suggests a stability cost for the tested dynamic rule, but does not isolate gradient feedback or establish that detachment causes the difference. Differentiation remains conditional on the selected subset in the stated protocol.

\paragraph{Market regime analysis.}
Table~\ref{tab:regime} reports regime-conditional tracking errors.

\begin{table}[h]
\centering
\caption{Regime analysis: Neural DFL vs.\ MSE at $K = 20$.}
\label{tab:regime}

\fontsize{9}{10.8}\selectfont
\setlength{\tabcolsep}{5pt}
\renewcommand{\arraystretch}{1.08}
\publicationtable{%
\begin{tabular}{@{}lcccc@{}}
\toprule
Regime & MSE TE & DFL TE & Gain & $p$-value \\
\midrule
High-volatility & 0.0343 & 0.0300 & $-12.7\%$ & 0.057 \\
Low-volatility  & 0.0308 & 0.0276 & $-10.5\%$ & 0.008 \\
\bottomrule
\end{tabular}}
\end{table}

\paragraph{Covariance window sensitivity.}
We ablate the lookback window for realized covariance estimation across $\{21, 42, 63, 126\}$ trading days (neural model, $K{=}20$, 10 rolling folds with 2-year training window). DFL's gain is stable across all windows: $+12.8\%$ for the 21-, 42- and 63-day windows and $+13.9\%$ for the 126-day window (all $p \leq 0.002$). The MSE baseline sits at TE $=0.0326$ regardless of window length---the neural model learns its own covariance representation---while DFL reaches $0.0281$--$0.0284$. The improvement persists across these archived lookback settings; these runs use a different fold protocol from the primary comparison.

\paragraph{Rebalancing frequency sensitivity.}
We ablate the test-time rebalancing frequency at weekly, biweekly and monthly intervals (neural model, $K{=}20$, 9 rolling folds; training uses weekly rebalancing in all cases). The tested rebalancing schedules give similar relative gains: $+14.9\%$ weekly (MSE $0.0325 \to$ DFL $0.0277$), $+14.4\%$ biweekly ($0.0325 \to 0.0278$) and $+14.7\%$ monthly ($0.0331 \to 0.0282$), all at $p = 0.002$ (Wilcoxon). The three gains agree to within half a percentage point, an empirical comparison within these three schedules. Turnover per rebalance increases as the interval lengthens (DFL: 0.013 weekly, 0.019 biweekly, 0.024 monthly), reflecting that each rebalance absorbs a larger drift when intervals are longer. The improvement is observed under all three tested schedules.

\paragraph{Position size constraints.}
Our main experiments impose $w_i \geq 0$ and $\sum w_i = 1$ but no upper bound. To test robustness under realistic portfolio constraints, we add $w_i \leq w_{\max}$ to the QP and re-evaluate neural DFL vs.\ MSE at $K{=}20$. With $w_{\max} = 10\%$ , DFL gain is $+12.7\%$ ($p = 0.002$, 9/9 folds); with $w_{\max} = 20\%$ it is $+14.7\%$ ($p = 0.002$, 9/9), matching the unconstrained case ($+14.7\%$) to the decimal. The tighter $w_{\max}{=}10\%$ constraint raises \emph{both} MSE and DFL tracking errors because it forces the QP away from the optimal solution, but the \emph{relative} DFL advantage is preserved. The relative improvement persists under both tested position caps.

\subsection{Negative Results: Multi-Step Lookahead and Curriculum Learning}
\label{app:negative}

We evaluated two training extensions that did not improve performance:

\paragraph{Multi-step lookahead.} Chaining 2--3 consecutive QP solves per training step (backpropagating through the trajectory $\bw^*_{t_1} \to \bw^*_{t_2} \to \bw^*_{t_3}$) did not reduce tracking error beyond single-step DFL. We tested lookahead depths of 2 and 3 with both shrinkage and conditional shrinkage models at $K \in \{10, 20\}$. In all cases, multi-step TE was within 0.1\% of the single-step baseline. This finite ablation does not identify why lookahead failed to help; dependence across dates and optimizer settings remain possible factors.

\paragraph{Curriculum learning (MSE $\to$ task).} We trained with a curriculum that begins with pure MSE loss, linearly transitions to task loss over 5--10 epochs, then continues with pure task loss. The intuition is that MSE provides a smoother loss landscape for initial optimization. However, the curriculum produced identical final TE to direct task-loss training. The task loss is not proved convex in the shrinkage parameter, and the ablation does not identify a basin-of-attraction mechanism.

\subsection{Value of the QP Structure}
\label{app:qpvalue}

An alternative to predict-then-optimize is to directly predict portfolio weights from features, bypassing the QP entirely. We test this with an MLP (same architecture as the neural model) that outputs $K$-sparse softmax weights, trained end-to-end on tracking error. At $K{=}20$, direct weight prediction achieves TE $= 0.0589$, more than double the DFL pipeline's $0.0276$; at $K{=}10$ the gap is $0.0806$ vs.\ $0.0462$. Both a normalized softmax and the QP can enforce budget and non-negativity constraints. The observed comparison favors the tested QP pipeline, but also changes the parameterization and optimization problem. It therefore does not isolate which aspect of the architecture causes the difference.

\subsection{Risk Metrics and Domain-Specific Analysis}
\label{app:risk}

Tracking error summarizes average dispersion but can conceal tail losses and variation across folds. We therefore examine additional tracking-risk measures on the same matched neural-model folds; these retrospective measures do not establish prospective investment performance.

\paragraph{Comprehensive risk metrics.}
Table~\ref{tab:risk_metrics} reports risk metrics at $K{=}20$ computed from daily portfolio returns. The reported reductions describe variability of tracking returns, not expected monetary savings.

\begin{table}[h]
\centering
\caption{Tracking-risk metrics at $K=20$ from exactly nine neural-model folds. Standard deviations use the sample convention; daily tail loss and arithmetic tracking drawdown are computed within folds. Negative relative change denotes a reduction. The common naive baseline is excluded.}
\label{tab:risk_metrics}

\fontsize{9}{10.8}\selectfont
\setlength{\tabcolsep}{5pt}
\renewcommand{\arraystretch}{1.08}
\publicationtable{%
\begin{tabular}{@{}lccc@{}}
\toprule
Metric & MSE & DFL & Improvement \\
\midrule
Annualized TE & 0.0331 & 0.0276 & $-16.5\%$ \\
21-day block TE (ann.) & 0.0296 & 0.0258 & $-12.7\%$ \\
Worst 21-day TE & 0.0538 & 0.0451 & $-16.1\%$ \\
Max tracking drawdown & 0.0362 & 0.0306 & $-15.5\%$ \\
CVaR$_{5\%}$ (daily) & 0.0044 & 0.0038 & $-14.3\%$ \\
Worst-fold TE & 0.0564 & 0.0360 & $-36.3\%$ \\
TE std across folds & 0.0098 & 0.0045 & $-54.0\%$ \\
\bottomrule
\end{tabular}}
\end{table}

\paragraph{Per-year regime breakdown.}
Each fold's test window is one year, so Table~\ref{tab:folds_exp1} is already a per-year breakdown; we do not repeat it here. Reading it as a regime series, DFL's gain peaks during volatile regime shifts (2021--2023: $26.6$--$40.0\%$) and is smallest in the calmest stretch (2017--2018: $0.3\%$, a small change), This descriptive variation does not identify covariance stability as its cause.

\paragraph{Factor exposure.}
A factor-neutrality claim requires the factor construction, aligned return series, regression specification and uncertainty estimates. The released results do not provide a complete auditable regression artifact, so we do not claim that the tracking improvement is factor-neutral.

\FloatBarrier
\section{Capacity Diagnostics and Spatial Transfer}
\label{app:diagnostic_group}

\subsection{Practitioner Diagnostic}
\label{app:diagnostic}
\begin{definition}[Heterogeneity-based capacity proxy]
For a structured model with $d$ trainable parameters and a specified partition into $R$ regions, let $m_r$ be the target index-weight mass in region $r$. Define $h=\operatorname{sd}(m_1,\ldots,m_R)/\operatorname{mean}(m_1,\ldots,m_R)$ and $d_{\mathrm{proxy}}=d h$ when the mean is positive. This proxy depends on the partition and is neither a Jacobian rank nor bounded by $d$. Archived equity predictions used $h=1$ by convention; measured heterogeneity and the resulting post-hoc re-scoring are reported separately in Appendix~\ref{app:diagnostic}.
\end{definition}
\begin{figure}[t]
\centering
\includegraphics[width=\figurewidth]{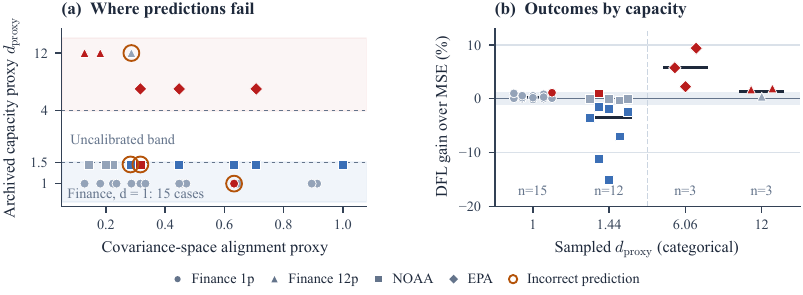}
\caption{\textbf{Archived blind-evaluation record (33 matched cases).} Colour shows DFL gain: red $>1\%$, grey within $\pm1\%$, blue $<-1\%$. (a)~Four incorrect predictions are ringed; overlapping finance points are counted. (b)~Equal spacing denotes the four sampled categories, not a continuous sweep; horizontal marks show means. This historical snapshot includes two abstentions.}
\label{fig:diagnostic}
\end{figure}

The archived diagnostic is a hypothesis-generating rule. A practical assessment should distinguish direct gradient measurements from proxy values:
\begin{enumerate}
\item Specify the input, parameterization, loss and aggregation unit. Measure output-space support energy and, when feasible, the pointwise and stacked Jacobians or the actual aggregate gradient angle.
\item Use $\sqrt{2K/N}$ only as an isotropic reference; it is not a measured cosine. The deterministic support-energy ratio is valid without that isotropy assumption.
\item Compare inexpensive baselines and task-trained models on held-out folds with paired settings. An observed small gain and a non-significant test do not establish equal minima.
\item Treat the proxy thresholds 1.5 and 4 as historical calibration choices. Neither the $d h$ proxy nor a measured spectral rank establishes a universal performance threshold. Recalibration requires separate validation data.
\item Under changing selection, tune covariance regularization on validation data and retain unstable runs in the report. The best tested coefficient on evaluation outcomes is an exploratory result, not a deployable selection rule.
\end{enumerate}

\paragraph{Selected standard estimators.} The released script probes six chosen covariance parameterizations on 252 days of returns for $N{=}97$ assets (\texttt{scripts/measure\_deployed\_deff.py}). The reported quantity is the entropy rank of a 24-row Gaussian projection of the Jacobian, rather than its full singular spectrum. It is distinct from the proxy $d\cdot h$; the sketch introduces additional approximation error. This sample is not a survey of deployment prevalence. Discrete hyperparameters contribute no differentiable Jacobian direction.

\begin{table}[h]
\centering
\caption{Spectral-rank estimates from 24 Gaussian Jacobian probes for six selected parameterizations. ``Low'' and ``higher'' describe this sketch, not a theorem about DFL performance. Zero denotes no trainable parameter. The sample does not establish deployment prevalence.}
\label{tab:deployed_deff}

\fontsize{9}{10.8}\selectfont
\setlength{\tabcolsep}{5pt}
\renewcommand{\arraystretch}{1.08}
\publicationtable{%
\begin{tabular}{@{}lrrl@{}}
\toprule
Estimator & $d$ & $r_{\mathrm{eff}}$ (sketch) & Sketch class \\
\midrule
Sample covariance & 0 & 0.00 & low \\
Ledoit-Wolf & 1 & 1.00 & low \\
RiskMetrics (EWMA) & 1 & 1.00 & low \\
Constant correlation & 1 & 1.00 & low \\
Two-target shrinkage & 2 & 1.29 & low \\
Structured shrinkage (12p) & 12 & 8.03 & higher \\
\bottomrule
\end{tabular}}
\end{table}

The three measured one-parameter families have rank-one Jacobians, so their nonzero parameter gradients satisfy Proposition~\ref{prop:dim1}. This geometric fact does not establish equal task optima or a universal cost advantage. The runtime comparison in Appendix~\ref{app:cost} concerns the tested implementation and configuration.

The archived scoring file contains 33 cases. The revised rule makes 31 committed predictions, of which 27 are correct, and abstains on two cases. Its historical 29/33 score counts both abstentions as correct and should not be interpreted as accuracy on 33 committed predictions. These archived outcomes predate the corrected equity harness.

\paragraph{Threshold sensitivity.} The diagnostic thresholds ($d_{\mathrm{proxy}} \leq 1.5$, $d_{\mathrm{proxy}} \geq 4$, alignment $< 0.4$) were examined post hoc: sweeping $d_{\text{low}} \in [1.0, 3.0]$, $d_{\text{high}} \in [3.0, 7.0]$, alignment $\in [0.25, 0.50]$ at 0.5/0.5/0.05 steps, using the archived convention that counts abstentions as correct, 73.3\% of all 270 parameter combinations achieve $\geq 85\%$ accuracy, and accuracy never falls below 78.8\% anywhere on the grid (median 87.9\%, max 97.0\%). Setting $d_{\text{low}} = 1.0$ yields 93.9\% but by classifying more cases as ``ambiguous'' (non-committal); our threshold of 1.5 makes the diagnostic more informative by issuing substantive predictions for low-heterogeneity domains ($d_{\mathrm{proxy}} = 1.4$ for weather).

\paragraph{Is $h$ doing any work on the equity corpus?}
A fair objection to $d_{\mathrm{proxy}} = d \cdot h$ is that our equity experiments
set $h{=}1$ by convention, so that on the corpus carrying most of the paper
$d_{\mathrm{proxy}}$ reduces to $d$ and the extra factor explains nothing. We test
this directly. Applying the same definition the synthetic sweep uses---$h$ is
the coefficient of variation of per-region index-weight mass, with regions
given by the structured model's sector partition---to the actual index weights
gives Table~\ref{tab:h_measured}. No equity universe is anywhere near $h{=}1$:
the measured values span $0.31$ to $0.71$, bracketing EPA ($0.61$) and sitting
well above NOAA ($0.14$).

\begin{table}[h]
\centering
\caption{Measured index-weight heterogeneity $h$ per universe (\texttt{scripts/measure\_heterogeneity.py}). The $h{=}1$ used in the main results is a conservative convention, not a measurement; every universe is in fact less heterogeneous than that.}
\label{tab:h_measured}

\fontsize{9}{10.8}\selectfont
\setlength{\tabcolsep}{5pt}
\renewcommand{\arraystretch}{1.08}
\publicationtable{%
\begin{tabular}{@{}lrrr@{}}
\toprule
Universe & $N$ & $R$ & measured $h$ \\
\midrule
S\&P 500 ($N{=}451$)   & 451 & 11 & 0.314 \\
S\&P 500 ($N{=}478$)   & 478 & 11 & 0.323 \\
S\&P 500 ($N{=}465$)   & 465 & 11 & 0.342 \\
S\&P 500 20y           & 100 & 11 & 0.378 \\
Nikkei 225             & 218 & 11 & 0.391 \\
FTSE 100               &  91 & 11 & 0.407 \\
S\&P 500 top-100       &  97 & 11 & 0.477 \\
ASX 200                & 162 & 11 & 0.505 \\
Euro Stoxx 50          &  47 &  9 & 0.636 \\
Hang Seng              &  62 & 11 & 0.713 \\
\midrule
\multicolumn{3}{@{}l}{NOAA weather (reference)} & 0.144 \\
\multicolumn{3}{@{}l}{EPA PM$_{2.5}$ (reference)} & 0.606 \\
\bottomrule
\end{tabular}}
\end{table}

Post-hoc re-scoring with measured heterogeneity gives 25/28 correct committed predictions (89.3\%), compared with 27/31 (87.1\%) under the archived convention. The committed case sets differ, so the percentages are not a paired estimate of improved diagnostic accuracy. The revised rule abstains more often and was examined after outcomes were available. Measured $h$ is reproducible from weights and a region partition, but these data do not establish that it explains the performance differences.

Three of the 38 one-parameter configurations clear a $1\%$ gain, and they come
from two markets: Hang Seng ($+1.76\%$ at $K{=}10$, $+1.28\%$ at $K{=}5$) and
ASX~200 ($+1.21\%$ at $K{=}5$). We have no measurement that explains them, and
we list what we ruled out rather than offer a mechanism we cannot support.
Measured $h$ does not: Euro~Stoxx carries the second-largest $h$ in the corpus
($0.636$, against Hang Seng's $0.713$ and ASX's $0.505$) and its three gains are
$+0.06\%$, $+0.29\%$ and $+0.08\%$. Sparsity does not: ASX at $K{=}5$ sits at
$K/N{=}0.031$, among the smallest ratios we test. Universe size does not:
Euro~Stoxx is the smallest universe at $N{=}47$. An earlier draft attributed
these gains to $h$ on the strength of Hang Seng alone; with the sweep re-run and
ASX also above $1\%$, that attribution no longer survives its own table. What
does survive is the conclusion that matters here: every one of these gains is
under $1.8\%$, far below what a $100$--$1000\times$ training overhead could
justify.

\subsection{Testing the Spectral-Gap Bound Directly}
\label{app:bound_check}

The saved experiment directly evaluates Theorem~\ref{thm:spectral} from the predictor Jacobian and the two output gradients. The computation is pointwise in the rebalance input.

For each (model, fold, rebalance date) we form $\mathbf{J} = \partial
\mathrm{vec}(\bSigma)/\partial\btheta$ by forward-mode differentiation
($d \ll N^2$, so this costs $O(d)$ passes), take its SVD, compute
$\mathbf{g}_{\mathrm{MSE}}$ analytically and $\mathbf{g}_{\mathrm{task}}$ through
the same differentiable QP the experiments train with, and compare the realised
$|\cos(\mathbf{J}^\top\mathbf{g}_{\mathrm{MSE}}, \mathbf{J}^\top\mathbf{g}_{\mathrm{task}})|$
against $\cos\Theta$. Across 162 points on the $N{=}97$ universe there are no recorded violations within numerical precision. The bound is informative at 108 rank-one or near-rank-one points and vacuous at all 54 structured-model points; vacuous cases do not provide a quantitative validation.

\paragraph{The rank-one endpoint.} For
the $d{=}1$ shrinkage model the bound evaluates to $1$ and the realised
alignment is $1.000000$ at all 54 points, with zero slack. This is the
theorem's endpoint and it is reproduced to the precision of the arithmetic.

\paragraph{Pointwise rank one with 385 parameters.}
The conditional model has median singular-value ratio approximately $1.2\times10^{-6}$ and absolute alignment rounding to one at all 54 sampled points. It maps each input to a scalar shrinkage intensity, so its covariance Jacobian is rank at most one for that input. Across inputs, the gradient of the intensity can rotate in parameter space. Consequently, this measurement does not identify the stacked rank or explain the ordering of test gains between the conditional and structured models. It establishes that parameter count is not pointwise rank; Proposition~\ref{prop:batch} describes the additional information needed for a training-level interpretation.

\paragraph{The near-rank-one bound is vacuous where it would be useful.} For the $d{=}12$ structured model the gap is genuinely open (median $0.089$) and the alignment is high but not unity (median $0.869$, minimum $0.625$)---the
regime the theorem's second part is meant to describe. There
$\Theta \geq \pi/2$ at every one of the 54 points, so the bound returns no
information. The cause is that $\rho_{\mathrm{task}}$ saturates:
$\mathbf{g}_{\mathrm{task}}$ sits almost orthogonal to the leading singular
direction $\mathbf{u}_1$ ($|\cos| \lesssim \sigma_2/\sigma_1$), while
$\mathbf{g}_{\mathrm{MSE}}$ does not. The bound assumes both gradients retain
non-trivial alignment with $\mathbf{u}_1$, and the task gradient does not.

We report this rather than omit it. The theorem is not contradicted---no point
violates it---but its quantitative content is confined to the rank-one endpoint,
and the intermediate regime is carried by the empirical diagnostic rather than
by the bound. Sharpening it for gradients that are near-orthogonal to
$\mathbf{u}_1$ is the obvious next piece of theory.

\subsection{Synthetic $d_{\mathrm{proxy}}$ Sweep}
\label{app:deff_sweep}

Table~\ref{tab:deff_sweep} and Figure~\ref{fig:deff_phase} report a controlled heterogeneity intervention; changing index weights also changes task difficulty. We fix $N{=}100$, $K{=}20$, 10 regions, 5 folds, and vary only the weight heterogeneity $h$ (via Dirichlet sampling of region weights), sweeping $d_{\mathrm{proxy}}$ from 0.5 to 20 across 18 settings.

\begin{figure}[h]
\centering
\includegraphics[width=\figurewidth]{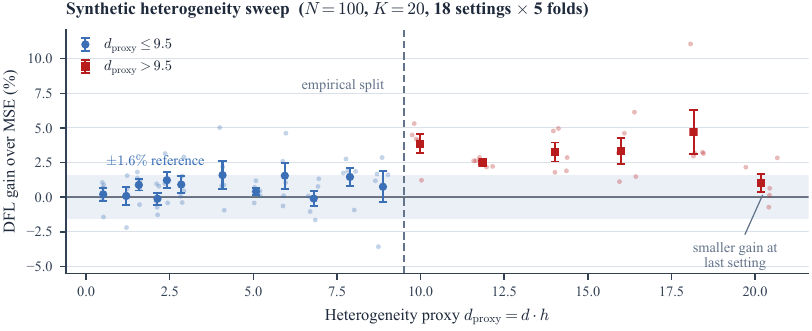}
\caption{\textbf{Synthetic heterogeneity sweep at $N=100$, $K=20$: 18 settings, five folds each.} Faint dots show every fold; larger markers show mean per-fold TE reduction and $\pm1$ SEM. The $\pm1.6\%$ band is a descriptive reference, not a confidence interval; individual folds can lie outside it. Mean gains are larger at several settings above the empirical split at $d_{\mathrm{proxy}}=9.5$, but fall to approximately 1.0\% at the final setting. The split is not a validated universal threshold.}
\label{fig:deff_phase}
\end{figure}

\begin{table}[h]
\centering
\caption{Synthetic heterogeneity sweep (18 settings, five folds each). Mean gains are larger at several high-heterogeneity settings, with an exception at the final setting. This exploratory pattern does not by itself establish a universal phase transition.}
\label{tab:deff_sweep}

\fontsize{9}{10.8}\selectfont
\setlength{\tabcolsep}{5pt}
\renewcommand{\arraystretch}{1.08}
\publicationtable{%
\begin{tabular}{@{}rrrrrr@{}}
\toprule
$h$ & $d_{\mathrm{proxy}}$ & TE$_\text{MSE}$ & TE$_\text{DFL}$ & Gain & $\sigma_\text{gain}$ \\
\midrule
0.05 & 0.51  & 4.296 & 4.286 & $+0.2\%$ & 0.9\% \\
0.12 & 1.20  & 4.251 & 4.246 & $+0.1\%$ & 1.2\% \\
0.16 & 1.58  & 4.145 & 4.108 & $+0.9\%$ & 0.8\% \\
0.21 & 2.13  & 4.169 & 4.175 & $-0.1\%$ & 0.9\% \\
0.24 & 2.41  & 3.805 & 3.758 & $+1.2\%$ & 1.1\% \\
0.28 & 2.84  & 3.962 & 3.925 & $+0.9\%$ & 1.2\% \\
0.41 & 4.08  & 3.766 & 3.704 & $+1.6\%$ & 2.0\% \\
0.51 & 5.08  & 3.457 & 3.442 & $+0.4\%$ & 0.5\% \\
0.59 & 5.94  & 3.348 & 3.295 & $+1.5\%$ & 1.9\% \\
0.68 & 6.80  & 3.474 & 3.477 & $-0.1\%$ & 1.1\% \\
0.79 & 7.89  & 2.410 & 2.375 & $+1.5\%$ & 1.2\% \\
0.89 & 8.88  & 3.309 & 3.282 & $+0.8\%$ & 2.2\% \\
\midrule
1.00 & 9.98  & 2.098 & 2.017 & $\mathbf{+3.9\%}$ & 1.4\% \\
1.19 & 11.85 & 1.975 & 1.926 & $\mathbf{+2.5\%}$ & 0.3\% \\
1.40 & 14.02 & 1.784 & 1.725 & $\mathbf{+3.3\%}$ & 1.4\% \\
1.60 & 15.98 & 1.297 & 1.254 & $\mathbf{+3.3\%}$ & 1.9\% \\
1.82 & 18.16 & 0.852 & 0.811 & $\mathbf{+4.7\%}$ & 3.2\% \\
2.02 & 20.17 & 1.101 & 1.090 & $+1.0\%$ & 1.3\% \\
\bottomrule
\end{tabular}}
\end{table}

The first twelve settings ($d_{\mathrm{proxy}}\leq9$) have mean gains of approximately $-0.1$--$1.6\%$. Five of the next six settings have means of 2.5--4.7\%, while the final setting returns to approximately 1.0\%. Baseline TE also changes with heterogeneity, so these measurements do not isolate effective dimension from task difficulty or establish a sharp universal transition.

\subsection{Generalization beyond Portfolio Optimization}
\label{app:generalization}

The deterministic support-energy inequality applies whenever the downstream loss depends on a restricted set of output coordinates. Its numerical isotropy bound requires an additional assumption and does not follow from selection alone. We examine one matching spatial QP and discuss other potential applications.

\textbf{Sensor placement (empirical validation).} Sensor placement is a canonical spatial selection problem \citep{krause2008near, joshi2009sensor}. We validate on a synthetic sensor network: $N$ sensors on a unit square with exponential spatial covariance, selecting $K$ sensors to track the field average (identical QP formulation). A structured shrinkage model with 10 parameters (9~spatial regions $+$ global $\alpha$) is trained under MSE and DFL. Table~\ref{tab:sensor} shows results across 10 configurations with 5 trials each.

\begin{table}[h]
\centering
\caption{Sensor placement: gradient alignment and DFL gain across $K/N$. The column $\sqrt{2K/N}$ is a reference under the energy assumption, not a universal bound on these measured cosines.}
\label{tab:sensor}

\fontsize{9}{10.8}\selectfont
\setlength{\tabcolsep}{5pt}
\renewcommand{\arraystretch}{1.08}
\publicationtable{%
\begin{tabular}{@{}rrrrrr@{}}
\toprule
$N$ & $K$ & $K/N$ & Reference & $|\cos|$ & DFL Gap \\
\midrule
200 & 10 & 0.05 & 0.316 & 0.297 & $-2.7\%$ \\
100 & 5  & 0.05 & 0.316 & 0.394 & $-0.9\%$ \\
50  & 5  & 0.10 & 0.447 & 0.295 & $-1.1\%$ \\
100 & 10 & 0.10 & 0.447 & 0.474 & $-1.5\%$ \\
50  & 10 & 0.20 & 0.632 & 0.349 & $-0.9\%$ \\
100 & 20 & 0.20 & 0.632 & 0.499 & $-0.8\%$ \\
200 & 40 & 0.20 & 0.632 & 0.617 & $-3.2\%$ \\
50  & 25 & 0.50 & 1.000 & 0.507 & $-0.4\%$ \\
100 & 50 & 0.50 & 1.000 & 0.606 & $-2.2\%$ \\
200 & 100 & 0.50 & 1.000 & 0.670 & $-0.3\%$ \\
\bottomrule
\end{tabular}}
\end{table}

Across the ten displayed configurations, median $|\cos|$ is approximately 0.346 at $K/N=0.05$ and 0.606 at $K/N=0.50$. All ten rows report a lower TE under DFL, but two measured cosines exceed the displayed isotropic reference. These observations do not establish the energy assumption or a universal sparsity law; the deterministic support-energy inequality remains the appropriate statement without that assumption.

\textbf{Facility location.} Selecting $K$ facilities from $N$ sites given predicted demand covariance. The allocation QP depends on demand correlations among selected sites, but MSE prediction does not distinguish selected from unselected.

\textbf{Sparse resource allocation.} Any optimization that distributes resources across $K \ll N$ options based on a predicted parameter matrix (e.g., sparse Markowitz with general objectives, sparse experimental design, or sparse regression with downstream optimization).

The key structural requirement is: (i)~the optimization depends on a $K$-dimensional sub-problem of an $N$-dimensional prediction, (ii)~the sub-problem selection is fixed or detached, and (iii)~the optimization admits efficient differentiation (e.g., KKT conditions). For each application, the actual support, differentiability, and stability assumptions must be checked separately. We do not establish a general grid-search sufficiency theorem or a universal regret guarantee for these applications.

\textbf{NOAA weather stations vs.\ EPA air quality (empirical validation).} Both use the same 10-parameter structured spatial model---the domains also differ in observations, targets, and index-weight heterogeneity $h$. NOAA stations ($N{\in}\{100, 189, 415\}$, temperature tracking): uniform coverage yields $h{=}0.14$, $d_{\mathrm{proxy}}{=}1.4$. EPA PM2.5 monitors ($N{\in}\{124, 132\}$, pollution index tracking): urban concentration yields $h{=}0.61$, $d_{\mathrm{proxy}}{=}6.1$. Table~\ref{tab:cross_domain} reports the exploratory contrast: The observed changes are negative for NOAA and positive for EPA; the small number of windows precludes a strong inferential claim.

\begin{table}[h]
\centering
\caption{Exploratory spatial comparisons with the $d_{\mathrm{proxy}}$ heuristic. Same 10-parameter structured model, different heterogeneity. NOAA ($d_{\mathrm{proxy}}{=}1.4$): DFL hurts. EPA ($d_{\mathrm{proxy}}{=}6.1$): DFL helps. \emph{The $n$ column is the point of caution}: these domains provide one evaluation window each (two for EPA at $N{=}124$), so each row has only one or two evaluation windows. We report them as the direction the diagnostic predicts, not as significance, and the equity results carry the statistical weight.}
\label{tab:cross_domain}

\fontsize{9}{10.8}\selectfont
\setlength{\tabcolsep}{5pt}
\renewcommand{\arraystretch}{1.08}
\publicationtable{%
\begin{tabular}{@{}llrrrrrr@{}}
\toprule
Domain & $d_{\mathrm{proxy}}$ & $N$ & $K$ & MSE & DFL & $\Delta\%$ & $n$ \\
\midrule
NOAA   & 1.4 & 100 & 20 & 20.03 & 20.76 & $-3.6\%$ & 1 \\
NOAA   & 1.4 & 189 & 20 & 14.87 & 15.14 & $-1.8\%$ & 1 \\
NOAA   & 1.4 & 415 & 20 & 14.41 & 14.76 & $-2.5\%$ & 1 \\
\midrule
EPA    & 6.1 & 132 & 10 & 33.94 & 31.98 & $\mathbf{+5.8\%}$ & 1 \\
EPA    & 6.1 & 132 & 50 & 24.39 & 22.09 & $\mathbf{+9.4\%}$ & 1 \\
EPA    & 6.1 & 124 & 50 & 13.83 & 13.27 & $\mathbf{+4.0\%}$ & 2 \\
\bottomrule
\end{tabular}}
\end{table}

\FloatBarrier
\section{Combinatorial Experiments}
\label{app:combinatorial_group}

\subsection{Shortest Path Experiment Details}
\label{app:shortest_path}

\begin{table}[h]
\centering
\caption{Archived shortest-path results on an $8\times8$ grid. $r_{\mathrm{in}}$ measures input sensitivity, not parameter-gradient rank. Regret is excess true path cost divided by optimal true path cost. $\Delta$ is MSE minus SPO+ regret in percentage points, averaged over five seeds.}
\label{tab:shortest_path}

\fontsize{9}{10.8}\selectfont
\setlength{\tabcolsep}{5pt}
\renewcommand{\arraystretch}{1.08}
\publicationtable{%
\begin{tabular}{@{}lrrrrr@{}}
\toprule
Model & Params & $r_{\mathrm{in}}$ & MSE Reg.\% & SPO+ Reg.\% & $\Delta$ \\
\midrule
$k{=}1$     &   149 & 1.0 & 53.8\% & 54.2\% & $-0.4$\% \\
$k{=}2$              &   234 & 2.0 & 58.4\% & 56.1\% & \textbf{$+2.3$\%} \\
$k{=}5$              &   489 & 4.8 & 63.8\% & 64.0\% & $-0.2$\% \\
$k{=}10$             &   914 & 9.1 & 72.5\% & 70.0\% & \textbf{$+2.6$\%} \\
$k{=}20$             &  1764 & 15.3 & 77.4\% & 74.2\% & \textbf{$+3.2$\%} \\
Full linear          &  1344 & 15.3 & 77.4\% & 73.3\% & \textbf{$+4.1$\%} \\
Full MLP             &   27K & 15.1 & 77.4\% & 75.9\% & $+1.5$\% \\
\bottomrule
\end{tabular}}
\end{table}

\begin{figure}[h]
\centering
\includegraphics[width=\figurewidth]{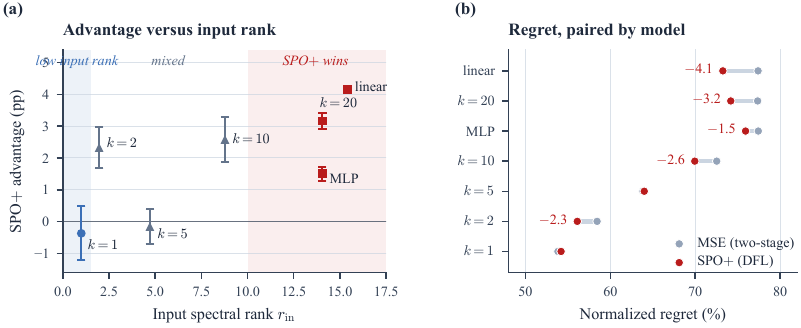}
\caption{\textbf{Shortest-path experiment on an $8\times8$ grid with five paired seeds.} (a) Mean SPO+ advantage in percentage points, $\Delta=\mathrm{regret}_{\rm MSE}-\mathrm{regret}_{\rm SPO+}$, with $\pm1$ SEM. Bottleneck width $k$ bounds input-Jacobian rank; full linear and MLP models also change architecture. The low-rank point has a small negative estimate, while intermediate ranks are mixed. (b) Mean normalized regret for the same paired runs; labels give the SPO+ change in percentage points. Shading describes sampled rank ranges, not a theorem about test regret.}
\label{fig:shortest_path}
\end{figure}

\paragraph{Setup.} We construct an $8 \times 8$ grid graph ($N{=}64$ vertices) with 8-connected edges. Each instance consists of features $x \in \mathbb{R}^{20}$ drawn i.i.d.\ from $\mathcal{N}(0, I)$, mapped to positive vertex costs via a fixed ground-truth network: $c = \text{softplus}(Wx + b + \varepsilon)$ where $W \in \mathbb{R}^{64 \times 20}$ and $\varepsilon \sim \mathcal{N}(0, 0.02^2 I)$. The optimal path $z^* = \arg\min_{z \in \mathcal{P}} c^\top z$ is solved by Dijkstra's algorithm (source: vertex 0, target: vertex 63; edge cost $=$ the entered vertex cost). We generate 8K training and 1.5K test instances.

\paragraph{Bottleneck architecture.} Each model maps $x \mapsto \hat{c}$ through a linear bottleneck of width $k$: $\hat{c} = \text{softplus}(\text{Dec}(\text{Enc}(x)))$ where $\text{Enc}: \mathbb{R}^{20} \to \mathbb{R}^k$ and $\text{Dec}: \mathbb{R}^k \to \mathbb{R}^{64}$. The bottleneck width $k$ controls the Jacobian rank: $\text{rank}(\partial \hat{c}/\partial x) \leq k$ by construction, bounding, but not fixing, the input spectral effective rank $r_{\mathrm{in}}$. We test $k \in \{1, 2, 5, 10, 20\}$ plus a full-rank linear ($k{=}20$, no bottleneck) and a 2-hidden-layer MLP (128 units, ${\sim}27$K parameters). All use the \texttt{softplus} activation on the output to ensure positive cost predictions.

\paragraph{Training.} Models are trained with Adam (lr $= 10^{-3}$ for MSE, $3{\times}10^{-4}$ for SPO+, weight decay $10^{-5}$), cosine annealing over 80 epochs, batch size 128, gradient clipping at norm 5.0. SPO+ loss uses the correct minimization formulation: $\ell_{\text{SPO+}}(\hat{c}, c, z^*) = (2\hat{c} - c)^\top (z^* - z_{\text{spo}})$ where $z_{\text{spo}} = \arg\min_{z \in \mathcal{P}} (2\hat{c} - c)^\top z$.

\paragraph{Spectral effective rank measurement.} For each trained model, we compute $r_{\mathrm{in}}$ of the input-space Jacobian $J_x = \partial \hat{c}/\partial x \in \mathbb{R}^{64 \times 20}$ (not the parameter-space Jacobian, for computational tractability). This bounds input sensitivity only. In particular, the trainable decoder bias supplies one parameter direction per output before numerical saturation; even a width-one bottleneck does not impose parameter-Jacobian rank one. We compute $J_x$ via autograd for 30 test inputs and report the mean $r_{\mathrm{in}}$.

\paragraph{Evaluation metrics.} \emph{Normalized regret}: $(c^\top z_{\text{pred}} - c^\top z^*) / c^\top z^*$, where $z_{\text{pred}}$ is the path obtained by running Dijkstra on the predicted costs $\hat{c}$. \emph{Accuracy}: fraction of test instances where $z_{\text{pred}} = z^*$ (exact path match). Results are averaged over 5 random seeds.

\paragraph{Prediction error and regret.}
Wider bottlenecks can give lower prediction error but higher decision regret in these saved results. The observation does not by itself identify overfitting or a gradient mechanism: architecture, optimization and path sensitivity change together. It should not be treated as evidence that a parameter-space no-go theorem extends to these input-rank measurements.

\paragraph{Scaling to $12{\times}12$ grid.} To test scale dependence, we replicate the bottleneck experiment on a $12{\times}12$ grid ($N{=}144$ vertices, 8-connected edges). We generate 5K training and 1K test instances with the same feature dimension ($n_{\text{features}}{=}20$) and training protocol (60 epochs, 5-fold cross-validation). Table~\ref{tab:sp12} and Figure~\ref{fig:sp_scale} report results.

\begin{table}[h]
\centering
\caption{Shortest path on $12{\times}12$ grid ($N{=}144$). The SPO+ advantage $\Delta$ increases monotonically with $r_{\mathrm{in}}$, reaching $+16.4\%$ for the full linear model---a $4\times$ amplification over the $8{\times}8$ grid (Table~\ref{tab:shortest_path}). Results averaged over 5 folds.}
\label{tab:sp12}

\fontsize{9}{10.8}\selectfont
\setlength{\tabcolsep}{5pt}
\renewcommand{\arraystretch}{1.08}
\publicationtable{%
\begin{tabular}{@{}lrrrrr@{}}
\toprule
Model & Params & $r_{\mathrm{in}}$ & MSE Reg.\% & SPO+ Reg.\% & $\Delta$ \\
\midrule
$k{=}1$     &   309 & 1.0 & 77.0\% & 74.6\% & $+2.4$\% \\
$k{=}2$              &   474 & 1.9 & 89.3\% & 83.7\% & \textbf{$+5.6$\%} \\
$k{=}5$              &   969 & 4.6 & 106.6\% & 99.5\% & \textbf{$+7.1$\%} \\
$k{=}10$             &  1794 & 8.7 & 118.5\% & 109.6\% & \textbf{$+8.9$\%} \\
$k{=}20$             &  3444 & 15.2 & 129.1\% & 118.1\% & \textbf{$+11.1$\%} \\
Full linear          &  3024 & 20.0 & 126.0\% & 109.6\% & \textbf{$+16.4$\%} \\
\bottomrule
\end{tabular}}
\end{table}

\begin{figure}[h]
\centering
\includegraphics[width=.75\figurewidth]{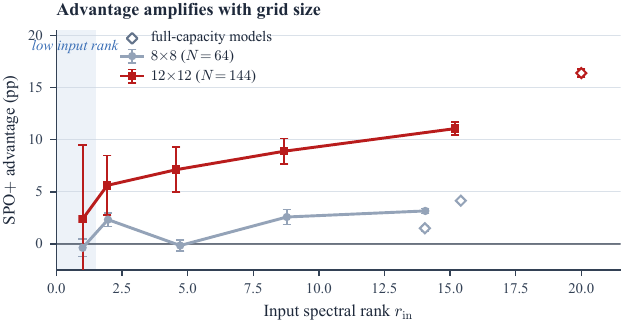}
\caption{\textbf{SPO+ advantage in percentage points at two grid sizes; positive values favor SPO+.} Whiskers show $\pm1$ SEM across five paired seeds. Filled markers join the controlled bottleneck settings; open diamonds denote full-capacity architectures. The larger grid has greater observed advantage at the sampled settings. A SEM bar crossing zero is not a hypothesis test; paired inferential results are reported separately in the corresponding table.}
\label{fig:sp_scale}
\end{figure}

The larger grid has greater measured SPO+ advantage at the sampled settings, but grid size changes the problem as well as the cost-space dimension; the experiment does not isolate the mechanism. The number of paired seeds remains five. In the $12\times12$ experiment, one-sided Wilcoxon gives $p=0.312$ at $k=1$ and $p=0.031$ at the tested $k\geq10$. With five pairs, 0.03125 is the smallest possible exact one-sided signed-rank p-value; these tests are unadjusted across settings.

\subsection{External Validation on PyEPO}
\label{app:pyepo}

We complement the custom tracking and path implementations with archived experiments using PyEPO benchmark generators and decision-loss implementations. This separates implementation provenance from the financial setting, while retaining the protocol limitations described below.

PyEPO \citep{tang2024pyepo} supplies the benchmark generator and optimization layers. The reported $r_{\mathrm{eff}}$ is measured on a sampled stacked Jacobian: the script takes up to 32 examples and the first eight output coordinates per example. It is an entropy-rank estimate for that probe matrix, not a full pointwise spectrum or the $d h$ proxy. We
use its shortest-path generator, its optimization model and its SPO+
implementation unmodified; the only thing we supply is the predictor.

\paragraph{A controlled parameterization sweep.} We hold the task, the data, the
optimizer and the training budget fixed and vary \emph{only} the number of
trainable parameters. The predictor is a fixed base map plus $d$ fixed
directions with trainable coefficients,
\begin{equation}
\mathbf{c}(\btheta) = W_0 \mathbf{x} + \textstyle\sum_{k=1}^{d} \theta_k (A_k \mathbf{x}),
\end{equation}
so $\partial \mathbf{c} / \partial \theta_k = A_k \mathbf{x}$ and the Jacobian has exactly
$d$ columns. $W_0$ is MSE-pretrained so that the restricted update starts from a useful predictor. This is a design choice for the comparison, not an implication of the one-parameter geometry.

Two variants we tried first do \emph{not} control $r_{\mathrm{eff}}$, and we
record them because both look reasonable. Constraining $W = UV^\top$ to low
matrix rank leaves $\partial \mathbf{c}/\partial(U,V)$ high-rank---measured
$r_{\mathrm{eff}}$ was $12.7$ at matrix rank one. Keeping a trainable bias adds
one free direction per output coordinate, putting $d{=}1$ at
$r_{\mathrm{eff}}{=}5.95$ on a 40-edge grid. Only the form above makes measured
$r_{\mathrm{eff}}$ track $d$.

\begin{table}[h]
\centering
\caption{PyEPO shortest path, $5{\times}5$ grid. Only the predictor's parameter
count varies; task, data, solver and budget are fixed. Regret is PyEPO's
normalised decision regret (lower is better); gain is the SPO+ reduction over
two-stage MSE training. $^{*}p{<}0.05$, one-sided Wilcoxon across seeds.}
\label{tab:pyepo}

\fontsize{9}{10.8}\selectfont
\setlength{\tabcolsep}{5pt}
\renewcommand{\arraystretch}{1.08}
\publicationtable{%
\begin{tabular}{@{}lrrrrrr@{}}
\toprule
$d$ & $r_{\mathrm{eff}}$ & MSE regret & SPO+ regret & Gain $\Delta\%$ & $p$ & wins \\
\midrule
1 & 1.00 & 0.0884 & 0.0887 & $-0.25$ & $0.754$ & 4/10 \\
2 & 1.94 & 0.0887 & 0.0890 & $-0.28$ & $0.722$ & 5/10 \\
4 & 3.73 & 0.0882 & 0.0874 & $+0.91$ & $0.042^{*}$ & 6/10 \\
8 & 7.02 & 0.0882 & 0.0876 & $+0.72$ & $0.246$ & 6/10 \\
16 & 12.70 & 0.0887 & 0.0870 & $+1.92$ & $0.065$ & 7/10 \\
32 & 20.39 & 0.0886 & 0.0845 & $+4.66$ & $0.002^{**}$ & 9/10 \\
full & 44.78 & 0.0886 & 0.0771 & $+12.99$ & $0.001^{**}$ & 10/10 \\
\bottomrule
\end{tabular}}
\end{table}

\paragraph{What it shows, and what it does not.} The gain rises with measured
$r_{\mathrm{eff}}$ across a $45\times$ range, from $-0.25\%$ at
$r_{\mathrm{eff}}{=}1.00$ to $+12.99\%$ at full rank, with Spearman
$\rho{=}{+}0.93$ ($p{=}0.0025$). At the lowest measured rank, the estimated gain is slightly negative, on a benchmark built by other authors and an SPO+ implementation we
did not write.

The $r_{\mathrm{eff}} \geq 4$ threshold does \emph{not} transfer. At
$r_{\mathrm{eff}}{=}7.02$ the gain is $+0.72\%$, below the $1\%$ the finance
calibration would predict, and the two rows either side of it ($+0.91\%$ at
$3.73$, $+1.92\%$ at $12.70$) win on only 6/10 and 7/10 seeds. What survives is
the ordering, not the cut point: $r_{\mathrm{eff}}$ ranks configurations by how
much DFL can buy, while the numeric thresholds are calibrated on equity data and
should be re-calibrated per domain. We report this because an earlier run at
five seeds did show both thresholds transferring, and the apparent agreement did
not survive ten.

\paragraph{The positive control.} The full-rank row recovers SPO+'s known
advantage over two-stage training ($+12.99\%$, $p{=}0.001$, 10/10 seeds) on the
library's own benchmark, through the same code path that produces the nulls at
low $r_{\mathrm{eff}}$. The positive full-capacity result shows that the implementation can improve this benchmark under at least one tested configuration; it does not validate every lower-capacity optimization setting.

\begin{figure}[h]
\centering
\includegraphics[width=\figurewidth]{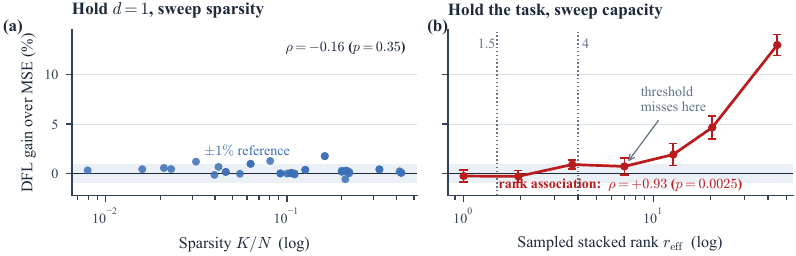}
\caption{\textbf{Two empirical axes, with a shared vertical scale.} (a) Thirty-eight one-parameter equity configurations across six markets and a $54\times$ range of $K/N$: gains remain below 1.8\%, with no detected monotone association ($\rho=-0.16$, $p=0.35$). The shaded $\pm1\%$ band is a reference, not an envelope containing all points. (b) PyEPO capacity sweep: relative reduction in mean regret, with $\pm1$ paired delta-method SE across ten seeds. Dotted thresholds were calibrated on equities; one transfer failure is marked. The two panels concern different tasks and support association rather than a causal comparison.}
\label{fig:two_axis}
\end{figure}

\paragraph{Does rank collapse kill every DFL method, or only the one we tested?}
Proposition~\ref{prop:dim1} is a statement about the Jacobian, not about a
particular decision loss: at $\mathrm{rank}(\mathbf{J}){=}1$ any loss reaching
$\btheta$ through $\mathbf{J}^\top$ produces a gradient collinear with the MSE
gradient. This statement constrains local directions for each loss; it does not predict that their trained solutions or performance coincide. We tested that with
four further PyEPO losses spanning unrelated mechanisms.

Method hyperparameters were selected using the full-capacity stage and then frozen for the reduced-capacity sweep (\texttt{scripts/tune\_pyepo\_methods.py}). This conditions the comparison on settings that work at full capacity; it does not give each reduced-capacity model its best validation-tuned configuration. The selection record does not establish a fully independent evaluation of the tuning procedure. We therefore report this as a transfer-of-hyperparameters experiment.

The black-box method did not improve the full-capacity baseline in the six tested configurations (best reported gain $-4.84\%$). It remains a negative result, rather than being removed from the scope of the conclusion. Other runs exceeding $1.5\times$ baseline regret are marked unstable. Such runs are relevant performance outcomes; rank-one collinearity does not rule them out or make them inadmissible counterevidence.

\begin{table}[h]
\centering
\caption{PyEPO decision losses at low and full measured capacity, with hyperparameters frozen from the full-capacity stage. Values are relative regret reductions over two-stage training. ``Unstable'' denotes regret exceeding $1.5\times$ baseline, a reported outcome rather than evidence of equivalence. Non-significant signed-rank comparisons do not establish equality. $^{*}p<0.05$, two-sided Wilcoxon over ten seeds.}
\label{tab:pyepo_methods}

\fontsize{9}{10.8}\selectfont
\setlength{\tabcolsep}{5pt}
\renewcommand{\arraystretch}{1.08}
\publicationtable{%
\begin{tabular}{@{}llrr@{}}
\toprule
Decision loss & Frozen hyper-parameters & $r_{\mathrm{eff}}{=}1.00$ & full rank ($44.8$) \\
\midrule
SPO+ & lr=0.01 & $-0.25$ & $+12.99^{*}$ \\
perturbedOpt & lr=0.01, $n{=}20$, $\sigma{=}0.5$ & $-2.32$ & $+4.98$ \\
negativeIdentity & lr=0.01 & \emph{unstable} & $+2.24$ \\
Perturbed Fenchel-Young & lr=0.01, $n{=}10$, $\sigma{=}0.5$ & $-0.57$ & $+15.80^{*}$ \\
\bottomrule
\end{tabular}}
\end{table}

At the lowest capacity, the finite results for SPO+, perturbed optimization and the perturbed Fenchel--Young loss are not significantly different from two-stage training under the reported tests. This does not establish equivalence. The negative-identity method is unstable at reduced capacity. Perturbed optimization has substantial negative gains at intermediate capacities, even though its full-capacity result is positive. These failures show that rank or a fixed hyperparameter configuration alone does not determine reliable optimization.

\paragraph{Multiple knapsack.}
The archived five-seed knapsack sweep does not show a detected monotone relationship between gain and measured capacity ($\rho=0.37$, $p=0.47$). At full capacity the estimated gain is $1.85\%$ ($p=0.41$, two of five wins), and several lower-capacity results favor MSE. Limited replication leaves substantial uncertainty, but the benchmark is still evidence of a setting where the proposed ordering did not appear. It should not be discarded solely because the full-capacity control failed to improve.

\paragraph{Relation to the main text.}
The equity sweep varies sparsity within one-parameter models, whereas this archived experiment varies an update subspace on a fixed generated task. The full-capacity endpoint also changes initialization, and hyperparameters are transferred across capacities. These experiments are complementary observations rather than a joint identification of the causal effect of rank. The following extension aligns initialization and validation budgets.

\subsection{Validation-Tuned Controlled Extension}
\label{app:controlled}
\paragraph{Purpose and relation to earlier experiments.}
This extension addresses two limitations of the archived PyEPO sweep: transferring hyperparameters from full capacity, and starting the full-capacity comparator differently from restricted models. It uses a new protocol and smaller exact-enumeration problems. It does not overwrite the earlier knapsack null result or establish that tuning caused the difference.

\paragraph{Tasks, splits and exact oracles.}
We use PyEPO's synthetic feature--cost generators \citep{tang2024pyepo} with five Gaussian features, degree four and multiplicative noise width 0.5. Ten seeds generate independent datasets per task, each split into 512 training, 256 validation and 512 test examples. The shortest-path task is a directed $4\times4$ grid with 24 edges and 20 feasible monotone paths. The two-dimensional knapsack has 12 items; each capacity is 35\% of the corresponding total item weight. All feasible subsets are enumerated. Knapsack values are negated to express both tasks as minimization. Exact enumeration is checked against dynamic programming for signed shortest-path costs and SciPy's MILP solver for knapsack instances. Costs are scaled using training data only.

\paragraph{Matched prediction families and selection.}
Let $\tilde x=(x^\top,1)^\top$. A ridge predictor $P_0$ is fitted only on training data, using penalty $10^{-3}$. Every loss and capacity starts from $P_0$ and uses
\[
\hat c(\theta;x)=\left(P_0+\sum_{k=1}^{d}\theta_k A_k\right)\tilde x,
\qquad\theta_0=0.
\]
The vectorized directions $A_k$ are a nested orthonormal basis, fixed within each dataset. We test $d=1,2,8$ and the complete basis ($d=144$ for path, $72$ for knapsack). Adam runs for 20 epochs with batch size 64. Each loss at each capacity independently selects among learning rates $0.003,0.01,0.03$ using validation decision regret at the final epoch. Initialization and minibatch order are paired. This gives 480 candidate fits and 160 selected models; test regret is evaluated for selected candidates only. The shared ridge baseline is evaluated separately. This initial protocol has one optimization seed per generated dataset. Appendix~\ref{app:robustness} adds a crossed batch-order analysis on fresh datasets.

\paragraph{Geometry, outcome and inference.}
We compute Jacobians for every output at the first 32 test examples after model selection. Spectral rank uses squared singular values, as in the main text; pointwise ranks are averaged, whereas stacked rank is computed from the full probe matrix. Gradients compare MSE with the SPO+ surrogate, not a derivative of discrete regret. The affine parameterization makes these Jacobians constant in $\theta$, while gradient alignment still depends on the selected predictor and data. Test regret is mean excess objective divided by mean absolute optimal objective. Reported gain is the reduction in mean regret across the ten paired datasets. Bootstrap intervals resample dataset pairs 10,000 times; they are pointwise, not simultaneous. Two-sided Wilcoxon tests receive Holm correction over all eight comparisons.

\begin{table}[t]
\centering
\caption{Controlled extension with independent validation at each capacity. Spectral ranks and batch alignment are measured at selected MSE predictors; means are over ten datasets. Gain is relative test-regret reduction in percent. $p_H$ is Holm-adjusted across eight comparisons. Full capacity differs by task.}
\label{tab:controlled}
\fontsize{9}{10.8}\selectfont
\setlength{\tabcolsep}{5pt}
\renewcommand{\arraystretch}{1.08}
\publicationtable{\begin{tabular}{@{}lrrrrrr@{}}
\toprule
Task & $d$ & Point $r_{\mathrm{eff}}$ & Stack $r_{\mathrm{eff}}$ & Batch $|\cos|$ & Gain & $p_H$ \\
\midrule
Path & 1 & 1.00 & 1.00 & 1.00 & $+0.24$ & 1.0000 \\
Path & 2 & 1.93 & 2.00 & 0.73 & $+1.10$ & 1.0000 \\
Path & 8 & 6.85 & 7.96 & 0.41 & $+1.89$ & 1.0000 \\
Path & 144 & 24.00 & 130.27 & 0.22 & $+11.59$ & 0.0684 \\
Knapsack & 1 & 1.00 & 1.00 & 1.00 & $-0.61$ & 1.0000 \\
Knapsack & 2 & 1.88 & 1.99 & 0.83 & $-0.51$ & 1.0000 \\
Knapsack & 8 & 5.82 & 7.90 & 0.55 & $+0.19$ & 1.0000 \\
Knapsack & 72 & 12.00 & 65.03 & 0.42 & $+10.59$ & 0.0156 \\
\bottomrule
\end{tabular}}
\end{table}

\paragraph{What transfers, and what remains unresolved.}
The one-direction models have collinear batch gradients, because all examples share one parameter direction. Full models have pointwise spectral ranks 24 and 12, but mean stacked ranks approximately 130 and 65; treating these measurements as interchangeable would hide the batch distinction. Low-capacity mean gains are small, but non-significance does not prove equivalence. Full-capacity pointwise bootstrap intervals are $[5.49,16.81]\%$ for shortest path and $[7.30,13.81]\%$ for knapsack. The former's adjusted $p=0.0684$ does not pass 0.05, despite its interval excluding zero, because the interval is unadjusted and the tests differ. All ten knapsack dataset pairs favor full-capacity SPO+; eight of ten do so on shortest path.

Changing capacity changes the attainable predictor family as well as its geometry. This intervention therefore does not identify rank as the sole cause of the gain. Neither task reproduces financial covariance estimation, detached asset selection, or transaction costs. The experiments test the broader predict--then--optimize question; the finance-specific support bound requires its own assumptions. Saved selected models and held-out arrays permit direct recomputation of decisions and gradients.

\subsection{Training Budget and Batch-Order Sensitivity}
\label{app:robustness}
\paragraph{Fresh data and crossed randomness.}
We retain the previous study and add ten new datasets per task, using disjoint generator seeds. Task sizes, noise and the 512/256/512 split remain unchanged. Within each dataset, the ridge initializer and update basis are fixed while three independently seeded minibatch orders vary. We test scalar and full-capacity updates at 20 and 80 epochs. Each 80-epoch optimization trajectory supplies both checkpoints, so training-budget comparisons share their initial 20 epochs. The same three learning rates are compared independently at each budget and for each loss. This entails 720 optimization trajectories, 1,440 trained candidate checkpoints and 480 selected models.

\paragraph{A common no-update option.}
Both MSE and SPO+ may additionally select the unchanged ridge predictor using validation regret. Ties favor no update, then the smaller learning rate. Thus the comparison does not force the prediction-trained baseline away from an already useful initializer. At full capacity, MSE selects no update in 11/30 path runs at each budget and 8/30, 6/30 knapsack runs at 20, 80 epochs. SPO+ selects no update in 1/30, 2/30 path runs and 3/30 knapsack runs at each budget. These are selection frequencies, not independent significance tests.

\paragraph{Independent units and uncertainty.}
For each task, capacity and budget, the three test regrets are first averaged within each dataset. The ten dataset-level pairs are then used for the relative reduction, paired bootstrap interval and two-sided Wilcoxon test. A separate Holm family covers the eight follow-up comparisons. This avoids treating optimization repetitions as additional independently generated tasks; bootstrap intervals remain pointwise. The study was designed after inspecting the first extension and is reported as a sensitivity analysis, not as a preregistered confirmation.

\begin{table}[t]
\centering
\caption{Fresh-data sensitivity study with three minibatch orders per dataset and a shared no-update candidate. Gain is relative test-regret reduction over validation-selected MSE. Intervals resample ten independent dataset pairs after averaging the three training seeds. $p_H$ is Holm-adjusted across the eight follow-up comparisons. Full means 144 path or 72 knapsack update directions.}
\label{tab:robustness}
\fontsize{9}{10.8}\selectfont
\setlength{\tabcolsep}{5pt}
\renewcommand{\arraystretch}{1.08}
\publicationtable{\begin{tabular}{@{}lrlrrr@{}}
\toprule
Task & Epochs & Capacity & Gain (\%) & 95\% interval & $p_H$ \\
\midrule
Path & 20 & 1 & $+0.47$ & $[+0.08, +1.00]$ & 0.1250 \\
Path & 20 & Full & $+12.12$ & $[+5.22, +19.68]$ & 0.0820 \\
Path & 80 & 1 & $+0.52$ & $[+0.02, +1.16]$ & 0.2812 \\
Path & 80 & Full & $+12.54$ & $[+5.22, +20.34]$ & 0.0977 \\
Knapsack & 20 & 1 & $+0.11$ & $[-0.31, +0.54]$ & 0.7422 \\
Knapsack & 20 & Full & $+13.67$ & $[+9.55, +17.83]$ & 0.0156 \\
Knapsack & 80 & 1 & $+0.27$ & $[-0.06, +0.67]$ & 0.5938 \\
Knapsack & 80 & Full & $+13.76$ & $[+9.42, +18.11]$ & 0.0273 \\
\bottomrule
\end{tabular}}
\end{table}

\paragraph{Interpretation.}
The full-capacity mean changes little between the two budgets on these datasets: path $12.12\%$ to $12.54\%$ and knapsack $13.67\%$ to $13.76\%$. This is evidence of limited sensitivity over the tested budgets, not proof of convergence. Knapsack passes the adjusted 0.05 threshold at both budgets, while path does not. Scalar changes remain below $0.6\%$; their lack of adjusted significance is not an equivalence result. This follow-up reduces concern about one minibatch ordering or a forced update baseline, but retains the small synthetic task sizes, linear predictor family and restricted learning-rate grid. Full selected parameters, data arrays and validation records are saved for audit.

\FloatBarrier
\section{Demonstration Controls}
\label{app:demonstration_group}
\subsection{Invertible Coordinates at Fixed Expressivity}
\label{app:reparameterization}
\paragraph{Question and protocol.}
Does the capacity contrast merely reflect lost expressivity? We use fresh datasets (seed offset 270927) from the same degree-four, noise-0.5 generators as Appendix~\ref{app:controlled}: ten datasets per task, 512 training, 256 validation and 512 test examples. A train-only ridge predictor $P_0$ initializes every run. We retain the full affine function class and scale its output coordinates by $D_\varepsilon=\mathrm{diag}(1,\varepsilon,\ldots,\varepsilon)$ for $\varepsilon\in\{1,0.1,0.01\}$. Exact pointwise Jacobian rank stays at 24 for path and 12 for knapsack. Spectral rank is the exponential entropy of normalized squared singular values; the reported values are pointwise, not stacked ranks.

\paragraph{Why compensation is an exact control.}
Write $G=\nabla_P L$. Ordinary SGD in $\Theta$ gives
\begin{equation}
\nabla_\Theta L=D_\varepsilon G,\qquad
\Delta P=-\eta D_\varepsilon^2G.
\end{equation}
The compensated step $\Delta\Theta=-\eta D_\varepsilon^{-2}\nabla_\Theta L$ instead gives $\Delta P=-\eta G$. With the same initializer, batch order and learning rate, induction therefore gives the same predictor trajectory as identity-coordinate SGD, for either loss. This equivalence uses invertibility and plain SGD; it is not a claim about arbitrary adaptive optimizers. The largest final parameter discrepancy in double precision was $6.7\times10^{-16}$.

Each loss and coordinate setting independently selects among learning rates $\{0.003,0.01,0.03\}$ and the unchanged initializer by validation decision regret, with ties favoring no update. We train 40 epochs with batch size 64 and matched minibatch order. There are 720 candidate fits and 240 selected models. The audit recomputes validation and test regret from saved predictors and exact decision oracles. Intervals below bootstrap the ten paired datasets 10,000 times; they are pointwise and no new multiplicity-adjusted significance claim is made.

\begin{table}[t]
\centering
\caption{Fixed-class coordinate control. Gain is mean test-regret reduction relative to matched MSE; positive favors SPO+. Exact rank is unchanged. All compensated rows within a task recover the same predictor outcomes.}
\label{tab:reparameterization}
\fontsize{9}{10.8}\selectfont
\publicationtable{\begin{tabular}{@{}llrrrr@{}}
\toprule
Task & Update & $\varepsilon$ & Spectral rank & Gain (\%) & 95\% interval\\
\midrule
Path & Ordinary & 1 & 24.00 & 11.79 & $[7.51, 15.76]$ \\
Path & Ordinary & 0.1 & 2.91 & 3.12 & $[0.85, 5.27]$ \\
Path & Ordinary & 0.01 & 1.02 & 1.19 & $[-0.50, 2.75]$ \\
Path & Compensated & 1 & 24.00 & 11.79 & $[7.51, 15.76]$ \\
Path & Compensated & 0.1 & 2.91 & 11.79 & $[7.51, 15.76]$ \\
Path & Compensated & 0.01 & 1.02 & 11.79 & $[7.51, 15.76]$ \\
Knapsack & Ordinary & 1 & 12.00 & 10.17 & $[5.73, 14.44]$ \\
Knapsack & Ordinary & 0.1 & 1.75 & 2.90 & $[1.30, 4.34]$ \\
Knapsack & Ordinary & 0.01 & 1.01 & 0.52 & $[-0.28, 1.29]$ \\
Knapsack & Compensated & 1 & 12.00 & 10.17 & $[5.73, 14.44]$ \\
Knapsack & Compensated & 0.1 & 1.75 & 10.17 & $[5.73, 14.44]$ \\
Knapsack & Compensated & 0.01 & 1.01 & 10.17 & $[5.73, 14.44]$ \\
\bottomrule
\end{tabular}}
\end{table}

\paragraph{Scope.}
The control rules out a change of function class as the explanation within this experiment. It does not isolate spectral rank from conditioning: both vary with the same coordinate transformation, and the unscaled first coordinate is fixed rather than randomized across orientations. The finite learning-rate grid and budget do not establish convergence. Exact rank collapse, spectral concentration and expressivity must therefore remain distinct concepts.

\subsection{Chronological Forward-Target Financial Baselines}
\label{app:forecast_control}
\paragraph{Question and data.}
We test whether fitting future covariance changes the conclusion drawn from a reconstruction baseline. This is a separate protocol using a frozen 100-stock universe and observed index returns from the cached 2005--2025 series (5,281 return observations). It is survivorship-biased and does not reconstruct point-in-time membership. Ten annual test folds cover 2016--2025, each preceded by six validation months and 36 training months. Shared histories make the annual folds unsuitable for an independence-based significance claim here.

At each month's first trading day, all estimators use strictly preceding returns. A common $K=20$ support is fixed for each fold using squared stock--index correlations from the 252 returns preceding validation. The input $S_t$ is the trailing 63-day joint covariance of stocks and the observed index. A long-only, unit-sum QP uses its selected stock block and stock--index cross-covariance. The weights are held numerically constant across daily return evaluations until the next monthly estimate; this corresponds to daily rebalancing to target weights. Transaction costs are omitted in this control.

\paragraph{Predictive and task baselines.}
The common scalar family is $M_\alpha=(1-\alpha)S+\alpha\,\mathrm{tr}(S)I/101$. Reconstruction uses $\alpha=0$. Future-target variants fit a single $\alpha\in[0,1]$ by closed-form Frobenius MSE against future 21- or 63-day covariance at monthly training origins. Every forward label ends before validation begins. Future-selected MSE chooses its horizon by validation tracking error. Task validation chooses among 21 uniformly spaced $\alpha$ values by the same metric; this is grid selection, not end-to-end DFL. Additional comparators are Ledoit--Wolf and EWMA with decay selected from $\{0.94,0.97,0.99\}$. Support, QP constraints and evaluation dates are matched across methods.

\begin{table}[t]
\centering
\caption{Future-target control over ten annual test folds. TE is mean annualized tracking error (lower is better). Covariance error is mean relative squared Frobenius error against future 21-day covariance, using only labels contained in the test year. Wins count years with lower TE than reconstruction. Changes are descriptive, not significance tests.}
\label{tab:forecast_targets}
\fontsize{9}{10.8}\selectfont
\publicationtable{\begin{tabular}{@{}lrrrr@{}}
\toprule
Estimator & TE (\%) & TE change (\%) & Cov. error & Wins / 10\\
\midrule
Reconstruction & 5.700 & +0.00 & 1.636 & --- \\
Future MSE: 21 days & 5.945 & +4.31 & 1.026 & 6 \\
Future MSE: 63 days & 6.083 & +6.73 & 0.983 & 6 \\
Future MSE: selected & 5.935 & +4.13 & 1.027 & 6 \\
Task validation & 5.616 & -1.46 & 1.453 & 7 \\
Ledoit--Wolf & 5.660 & -0.69 & 1.261 & 6 \\
EWMA: selected & 5.708 & +0.14 & 1.527 & 2 \\
\bottomrule
\end{tabular}}
\end{table}

\begin{figure}[t]
\centering
\includegraphics[width=\figurewidth]{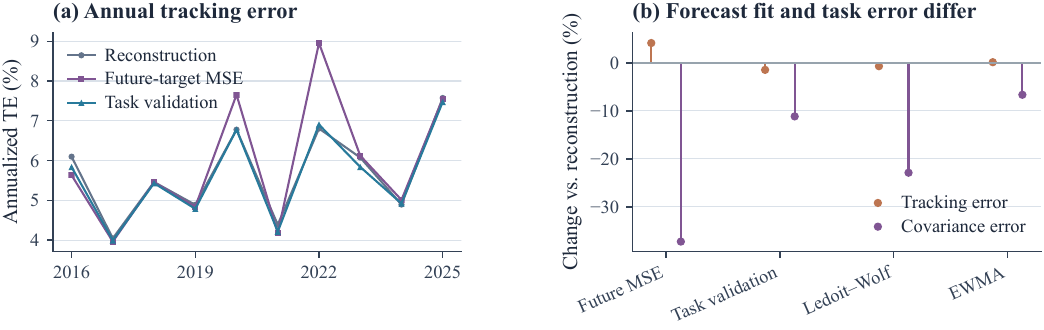}
\caption{\textbf{Improved covariance fit need not improve tracking.} Left: annual held-out TE reveals the uneven cost of future-target fitting across market years. Right: changes in two distinct metrics relative to reconstruction; negative is better for either metric. Their percentages have different denominators and are not a common utility scale.}
\label{fig:forecast_targets}
\end{figure}

\paragraph{Interpretation and remaining gap.}
The selected future-target estimator improves mean relative covariance error by $37.25\%$, yet worsens mean TE by $4.13\%$; it beats reconstruction in six years, illustrating why win counts alone miss the magnitude of failures. Task validation improves mean TE by $1.46\%$, and Ledoit--Wolf by $0.69\%$. These results do not show that reconstruction is the strongest predictive baseline or that DFL dominates forecasting. They demonstrate target sensitivity within one scalar family. The neural comparison in Appendix~\ref{app:neural_forward} extends this control beyond a scalar family. Point-in-time stock membership, transaction costs and broader forecast architectures remain needed. The release saves all 70 portfolios, daily returns, support selections and chronology; its independent audit recomputes TE, forward-label boundaries, fitted shrinkage and validation selections.

\subsection{Matched Neural Forward-Target Comparison}
\label{app:neural_forward}
\paragraph{Design.}
To address the scalar family's limited expressivity, we compare MSE and DFL using the same residual neural covariance predictor. This is a new matched experiment, not a rerun of the historical neural architecture. We reuse the chronology, frozen universe, training-only $K=20$ support and daily target-weight evaluation of Appendix~\ref{app:forecast_control}. Both objectives use the same monthly training origins with 21-day forward labels ending before validation. Ten annual test folds cover 2016--2025. Each fold uses three paired initialization seeds, a 36-month training interval and six validation months.

\paragraph{Architecture and objectives.}
For each of the 20 selected stocks and the index, six trailing features describe 21-/63-day means divided by 63-day standard deviation, log 21-/63-day volatility, correlation with the index, and the five-day mean divided by 63-day standard deviation. Feature means and scales are fitted only on training origins. A shared $6\to16$ tanh layer and $16\to5$ linear head contain 197 parameters. If its outputs are $(z_i,f_i)$, the normalized covariance is
\begin{equation}
\widehat M = D\widetilde S D + FF^\top + 10^{-4}I,
\quad D_{ii}=\exp\!\left(\tfrac12\tanh z_i\right),
\quad F_i=0.1f_i.
\end{equation}
Here $\widetilde S$ is trailing 63-day joint covariance divided by the mean training covariance trace per asset. The rank-four residual does not define the predictor's Jacobian rank. Small head weights initialize the network close to trailing covariance. The diagonal scaling is bounded and the residual is positive semidefinite, so this is one restricted neural family rather than an unrestricted covariance forecaster.

MSE minimizes mean squared entries against centered future 21-day sample covariance. DFL minimizes mean squared tracking returns on those same 21 days, using the exact long-only, unit-sum QP against the predicted joint stock--index covariance. DFL's training objective is a second moment, while reported TE uses centered tracking-return dispersion. This control omits turnover penalties and transaction costs.

\paragraph{Matched optimization and selection.}
Both losses use full-batch Adam with gradient norm clipped at one, learning rates $\{0.001,0.003\}$, and 40 epochs. Validation TE selects among checkpoints at 20 and 40 epochs and the unchanged initializer; both methods receive the same candidate budget. The 120 candidate trajectories yield 60 selected models. We determine the QP active set numerically and differentiate its equality-constrained solution; this derivative is local to a stable active set. Twelve double-precision directional finite-difference checks had maximum absolute discrepancy $2.1\times10^{-11}$. Independent saved-model checks reconstruct neural outputs in NumPy, verify KKT conditions and chronological boundaries, and recompute tracking returns and selected validation scores.

\begin{table}[t]
\centering
\caption{\textbf{Matched neural future-target comparison.} TE is annualized and averaged over three initialization seeds within each year. Positive gain denotes a reduction from MSE to DFL. The last row compares the two means over years, rather than averaging yearly percentage gains. Annual folds share history; results are descriptive.}
\label{tab:neural_forward}
\fontsize{9}{10.8}\selectfont
\publicationtable{\begin{tabular}{@{}rrrr@{}}
\toprule
Test year & Future MSE TE (\%) & DFL TE (\%) & Gain (\%)\\
\midrule
2016 & 6.090 & 6.127 & -0.60 \\
2017 & 4.106 & 4.039 & +1.64 \\
2018 & 5.463 & 5.454 & +0.17 \\
2019 & 4.885 & 4.899 & -0.28 \\
2020 & 6.789 & 6.836 & -0.69 \\
2021 & 4.366 & 4.463 & -2.21 \\
2022 & 6.828 & 6.795 & +0.48 \\
2023 & 6.085 & 6.081 & +0.08 \\
2024 & 4.929 & 4.897 & +0.64 \\
2025 & 7.569 & 7.560 & +0.12 \\
\midrule
Mean & 5.711 & 5.715 & -0.07 \\
\bottomrule
\end{tabular}}
\end{table}

\begin{figure}[t]
\centering
\includegraphics[width=\figurewidth]{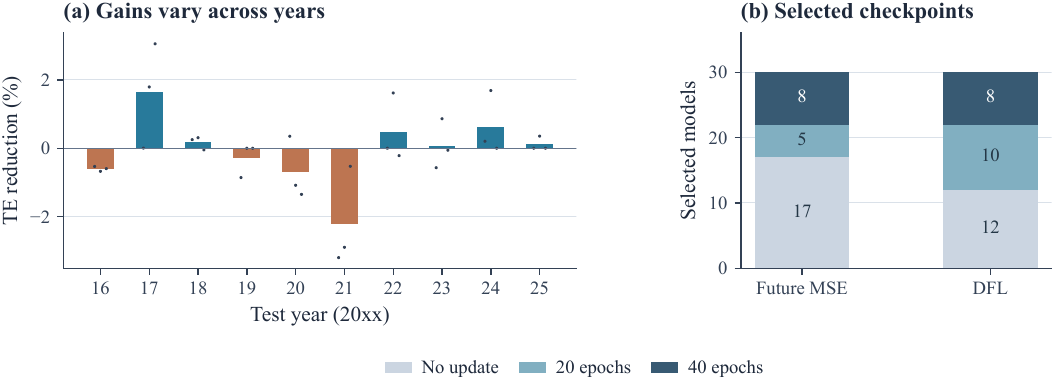}
\caption{\textbf{A matched neural control without an aggregate DFL gain.} (a) Bars compare seed-averaged TE within each year; dots show the three paired initialization results and are not additional independent datasets. Positive values favor DFL. (b) Validation-selected checkpoints across 30 models per loss. Frequent no-update selections limit conclusions about trained-model superiority.}
\label{fig:neural_forward}
\end{figure}

\paragraph{Results and limits.}
Mean TE is $5.711\%$ for future-target MSE and $5.715\%$ for DFL; the latter is $0.07\%$ higher in relative terms. DFL improves six of ten years, but its mean advantage is absent. Validation selects no update for 17/30 MSE models and 12/30 DFL models. This finding is a descriptive negative result, not evidence of statistical equivalence or proof that neural DFL is ineffective. It addresses a matched forward-target comparison for one 197-parameter family. The narrow learning-rate grid, 40-epoch budget, small number of monthly training origins and initialization near historical covariance warrant further sensitivity analysis. The result also cannot attribute differences from historical financial gains to the target alone: architecture, asset support and evaluation protocol differ. Broader architectures, more training and point-in-time membership remain unresolved.

\FloatBarrier
\section{Open Questions and Limitations}
\label{app:limits_group}

\subsection{Unresolved Empirical Findings}
\label{app:unexplained}

The following observations delimit the empirical interpretation. They distinguish failures of a proposed ordering from experiments that remain too limited to identify a mechanism. In particular, the new neural control cannot explain why its aggregate result differs from historical financial comparisons.

\paragraph{1. Two markets exceed $1\%$ at $d{=}1$, and no measured quantity
predicts which.} Three of the 38 one-parameter configurations clear a $1\%$ gain: Hang~Seng at $K=10$ ($1.76\%$) and $K=5$ ($1.28\%$), and ASX~200 at $K=5$ ($1.21\%$). No simple ordering by heterogeneity, sparsity or universe size explains these cases. Euro~Stoxx has the second-largest measured heterogeneity ($h=0.636$, versus $0.713$ for Hang~Seng and $0.505$ for ASX), yet its three gains are $0.06\%$, $0.29\%$ and $0.08\%$. ASX at $K=5$ has a small support ratio ($K/N=0.031$), while Euro~Stoxx has the smallest universe ($N=47$). These comparisons do not identify which market characteristics cause the differences.

\paragraph{2. The upper threshold does not transfer across domains.}
On PyEPO shortest path, SPO+ at sampled stacked rank $r_{\mathrm{eff}}=7.02$ gains $0.72\%$, below the equity rule's $1\%$ benefit threshold (Appendix~\ref{app:pyepo}). The rank--gain ordering is strong ($\rho=0.93$), but it does not validate a shared cut point. The equity proxy and sampled stacked rank also measure different quantities. A five-seed result suggesting threshold transfer did not persist with ten seeds.

\paragraph{3. Architecture orders differently across equity universes.}
At $N=100$, the largest displayed gain occurs for the structured model; at $N=478$, the conditional model has the largest relative reduction among these three classes. Its pointwise rank-one Jacobian does not resolve this difference, because input-dependent Jacobian directions can span a larger batch subspace. Changes in architecture, data, initialization and optimization also remain potential explanations.

\paragraph{4. \texttt{perturbedOpt} is significantly worse than two-stage
through the middle of the capacity range.} With hyper-parameters frozen from
the full-capacity stage, it loses to MSE at every intermediate
$r_{\mathrm{eff}}$---$-9.2\%$ at $1.94$, $-7.9\%$ at $3.73$, $-20.8\%$ at $7.02$,
$-30.3\%$ at $12.70$, $-13.2\%$ at $20.39$, all $p \leq 0.02$---then recovers to
$+4.98\%$ at full rank. At the lowest rank, its estimated change is $-2.32\%$ ($p=0.19$); this does not establish equivalence. Capacity-specific tuning is needed to separate optimization effects from geometry.

\paragraph{5. Negative-identity training is unstable at reduced capacity.}
The full-capacity configuration gives a $2.24\%$ gain, while the reduced-capacity runs exceed $1.5\times$ baseline regret. The instability could reflect the loss approximation, learning-rate transfer, or the restricted parameterization. It is retained as an experimental failure rather than excluded by the geometric argument.

\paragraph{6. The knapsack sweep does not support the proposed ordering.}
In the archived protocol, the full-capacity gain is imprecise ($1.85\%$, $p=0.41$, five seeds), and gain is not detectably monotone in measured capacity. More repetitions and separate tuning are needed to distinguish uncertainty from a task-specific failure of the hypothesis. The new validation-tuned extension in Appendix~\ref{app:controlled} gives a different result on a smaller task; changes in problem size, initialization, tuning and data prevent attributing that difference to one cause.

\paragraph{7. Spatial replication is limited.}
NOAA and EPA provide one evaluation window per configuration, except EPA at $N=124$, which provides two. Table~\ref{tab:cross_domain} therefore supports only a descriptive contrast. The synthetic sweep has five folds at each of 18 settings, but cannot establish that heterogeneity explains the real-domain difference. Measured equity heterogeneity is reported separately in Appendix~\ref{app:diagnostic}.

\paragraph{8. The neural forward-target control does not reproduce the historical gain.}
The matched residual-network experiment gives mean TE of $5.711\%$ for MSE and $5.715\%$ for DFL, whereas the historical nine-fold neural comparison favors DFL. Architecture, support, targets and evaluation protocol differ, so this contrast cannot identify which change matters. Moreover, validation frequently retains the initializer. Longer training, a wider learning-rate search and alternative covariance architectures are needed before attributing the result to objective choice or predictor geometry. Shared annual histories and the absence of an equivalence test further limit inference (Appendix~\ref{app:neural_forward}).

\paragraph{Implication for use.}
For a new task, first establish a predictive baseline with labels aligned to the intended forecasting horizon. Compare objectives within the same architecture, initialization distribution and validation budget, allowing both to retain an unchanged baseline. Interpret geometry alongside effect sizes and held-out outcomes; rank one alone does not justify skipping task training. The small financial gains and negative neural control reported here motivate these comparisons, but do not prescribe a universal choice of objective.

\end{document}